\documentclass[11pt]{article}
\usepackage[T1]{fontenc}
\usepackage[utf8]{inputenc}
\usepackage{lmodern}
\usepackage[margin=1in]{geometry}
\usepackage{amsmath,amssymb,amsthm,mathtools}
\usepackage{needspace}
\usepackage{microtype}
\usepackage{booktabs,tabularx,array}
\usepackage{algorithm,algpseudocode}
\usepackage{xcolor}
\usepackage{graphicx}
\usepackage{tikz}
\usepackage[colorlinks=true,linkcolor=blue!55!black,citecolor=blue,urlcolor=blue!55!black]{hyperref}
\newtheorem{theorem}{Theorem}[section]
\newtheorem*{geometricratetheorem}{Theorem~\ref{thm:general-geometric-rate} (restated)}
\newtheorem{lemma}[theorem]{Lemma}
\newtheorem{proposition}[theorem]{Proposition}
\newtheorem{corollary}[theorem]{Corollary}
\theoremstyle{definition}

\theoremstyle{remark}

\newcommand{\CAL}{\mathsf{Cal}}
\newcommand{\APPLOSS}{\mathsf{AppLoss}}
\newcommand{\diam}{\operatorname{diam}}
\newcommand{\dist}{\operatorname{dist}}

\newcommand{\conv}{\operatorname{conv}}
\newcommand{\R}{\mathbb R}
\newcommand{\cY}{\mathcal Y}
\newcommand{\cL}{\mathcal L}

\newcommand{\E}{\mathbb E}

\allowdisplaybreaks[1]
\renewenvironment{abstract}
  {\normalsize\begin{center}\bfseries\abstractname\end{center}\begin{quotation}}
  {\end{quotation}}
\hypersetup{pdftitle={High-Dimensional Online Calibration from Harmonic Weights},
 pdfauthor={Maxwell Fishelson and Mehryar Mohri}}
\title{High-Dimensional Online Calibration from Harmonic Weights}
\author{Maxwell Fishelson\thanks{Institute for Advanced Study.}\and
Mehryar Mohri\thanks{Google Research and Courant Institute of Mathematical Sciences, New York.}}
\date{October 29, 2026}

\begin{document}
\maketitle
\begin{abstract}
We study the online calibration of multidimensional forecasts over an
arbitrary convex set $\cY\subseteq\R^d$ relative to an arbitrary
error norm $\|\cdot\|_{\cL}$. For forecasting \(d\) binary outcomes simultaneously \((\mathcal Y=[0,1]^d)\),
% For multi-event calibration (forecasting $d$ binary outcomes concurrently AKA forecasting over the hypercube),
we give the first algorithm that achieves
$\varepsilon$-calibration in a number of rounds that is
polynomial in $d$ for every fixed accuracy. It requires
$d^{O(1/\varepsilon)}$ rounds, exponentially improving the dimension
dependence of previous bounds~\cite{perchet2014approachability,fishelson2025full,fishelson2025high}. For multi-class forecasting \((\mathcal Y=\Delta_d)\), we obtain the same $d^{O(1/\varepsilon)}$ rate, improving the
$d^{\widetilde O(1/\varepsilon^2)}$ bounds of Peng~\cite{peng2025high} and
Fishelson, Golowich, Mohri, and Schneider~\cite{fishelson2025high}.

Our algorithm is simple: on each round, it outputs a harmonically weighted
distribution over harmonically smoothed past outcomes. The same algorithm
works for every forecast set and error norm. More generally, it achieves $\varepsilon$-calibration after
$\exp(O(\gamma(\cY,\cL)/\varepsilon))$ rounds, where \(\gamma(\mathcal Y,\mathcal L)\) is a geometric parameter defined by a matrix discrepancy problem.
% \[
% \Gamma_{\cL,T}(H)
%  =\sup_{\substack{\ell_t\in\cL\\r_s\in\cL^\circ}}
%                 \sum_{t,s<T}H_{ts}\langle\ell_t,r_s\rangle.
% \]
% The motivation for the harmonic weights of our algorithm is that, over a constrained class of matrices, the discrete Hilbert transform matrix $H$ minimizes $\Gamma_{\cL,T}(H)$ up to a universal constant factor, for every possible $\cL$. This optimality result may be of independent interest.
The harmonic weights are motivated by the fact that the discrete Hilbert transform matrix achieves the optimal discrepancy up to a universal constant, simultaneously for every \(\mathcal L\). This optimality result may be of independent interest.

The same harmonic responses also yield an algorithm for high-dimensional
Blackwell approachability. Its guarantees are governed by a sharpened
geometric parameter that captures the joint geometry of the outcome and
action sets.

A video presentation of this work is available at
\url{https://youtu.be/2aL0viNUveE}.
\end{abstract}
\clearpage

\section{Introduction}\label{sec:introduction}
Consider a forecaster who reports a probability of rain each day. What
should it mean for these probabilities to be reliable? A natural requirement
is that, among the days on which the forecaster predicts a $40\%$ chance of
rain, it should rain on approximately $40\%$ of them. The same requirement
should hold for every probability the forecaster reports. This agreement
between forecasts and observed frequencies is called \emph{calibration}.
It gives a way to assess probabilistic predictions over time, since a single
day's outcome says little about the reliability of a $40\%$ forecast.

We study calibration in an online setting: on each round, the
forecaster makes a prediction, observes the outcome, and uses this feedback
to form subsequent predictions. To measure its error, group together the
rounds on which it made the same prediction, compare that prediction with
the empirical frequency in the group, and average the discrepancies with
weights proportional to the group sizes. A forecaster is approximately
calibrated when this average error is small. The online question is whether
one can guarantee such agreement regardless of how the outcomes are
generated. Foundational work by Foster and Vohra
\cite{foster1998asymptotic} established that calibration is possible even
against adversarial sequences.

Much of the early literature on calibration focused on forecasting a
single binary outcome. Many forecasting problems, however,
require higher-dimensional predictions. In
\emph{multi-event forecasting}, a forecaster predicts several binary
events concurrently: for example, whether it will rain in each of $d$
cities. The forecast is a vector in $[0,1]^d$, and the outcome records
which cities receive rain.
To evaluate calibration, we first group together the rounds on which the
same forecast vector was selected. Within each group, we compute the
empirical frequency of rain in each city: the fraction of rounds in that
group on which it rained there. We then compare the resulting vector of
empirical frequencies with the forecast vector, measuring error in the
infinity norm $\|\cdot\|_\infty$. Finally, we average these discrepancies
across groups, weighting each group by its fraction of all rounds.

In \emph{multi-class forecasting}, the forecaster predicts a single event
with $d$ possible outcomes: for example, assigning a probability to each
possible next token in a vocabulary of size $d$. Here, forecasts are
elements of the probability simplex $\Delta_d$. Within each group of
rounds sharing the same forecast, the empirical outcome frequencies also
form a vector in $\Delta_d$. We measure the discrepancy between these
two distributions using the norm $\|\cdot\|_1$.

The central difficulty in higher-dimensional forecasting is the dependence
on dimension. Standard algorithms discretize the forecast set using an
$\varepsilon$-net and predict over this finite set
\cite{foster1998asymptotic,perchet2014approachability,fishelson2025full}. Such a grid requires
$(1/\varepsilon)^{O(d)}$ points and yields algorithms with calibration rates that are exponential
in $d$ even for fixed accuracy. This dependence is prohibitive for modern
forecasting tasks with many events or classes.

Recent breakthroughs improve the dimension dependence in important
settings, at the cost of a worse dependence on accuracy.
Peng~\cite{peng2025high} (FOCS 2025) gives a multi-class calibration algorithm that
achieves $\varepsilon$-calibration in
$d^{\widetilde O(1/\varepsilon^2)}$ rounds.
The TreeCal algorithm of Fishelson, Golowich, Mohri, and Schneider~\cite{fishelson2025high}
(NeurIPS 2025, oral) generalizes this approach to arbitrary convex forecast sets $\cY$ and
error norms $\|\cdot\|_{\cL}$, obtaining a sufficient horizon of
$\exp(\widetilde O(\rho/\varepsilon^2))$, where $\rho$ is a geometric
parameter related to online linear optimization rates. This recovers the
multi-class bound and gives a dimension-independent horizon of
$\exp(\widetilde O(1/\varepsilon^2))$ for Euclidean forecasting ($\cY = B_2^d$).
Notably, for multi-event forecasting, the resulting bound is
$\exp(\widetilde O(d/\varepsilon^2))$, which remains exponential in
dimension. Discretization-based methods give the stronger bound
$(1/\varepsilon)^{O(d)}$ in this setting
\cite{perchet2014approachability}, but still incur exponential dimension
dependence. We improve the dimension and accuracy dependence of these
guarantees using a single elementary algorithm.

\paragraph{Our results.}
Our algorithm works for every convex forecast set $\cY$ and error norm
$\|\cdot\|_{\cL}$. We first present its guarantees in three standard
forecasting settings, then describe the algorithm and state its general
geometric guarantee.
For multi-event forecasting, we obtain $\varepsilon$-calibration in
$d^{O(1/\varepsilon)}$ rounds. This is the first algorithm for multi-event
calibration whose horizon is polynomial in $d$ for each fixed accuracy.
For multi-class forecasting we obtain the same horizon, sharpening the
accuracy dependence in the bounds of Peng~\cite{peng2025high} and
Fishelson, Golowich, Mohri, and Schneider~\cite{fishelson2025high} from
$d^{O(\varepsilon^{-2}\log(1/\varepsilon))}$ to
$d^{O(1/\varepsilon)}$. On the Euclidean unit ball, we improve the TreeCal
horizon~\cite{fishelson2025high} from
$\exp(O(\varepsilon^{-2}\log(1/\varepsilon)))$ to
$\exp(O(1/\varepsilon))$, independently of $d$.
Table~\ref{tab:rates} summarizes the improved rates.

\Needspace{12\baselineskip}
\begin{samepage}
\begin{theorem}[Principal consequences, informal]\label{thm:headlines}
A single explicit anytime algorithm, stated formally as
Algorithm~\ref{alg:harmonic}, produces forecasts over the set $\cY$,
achieving average calibration error at most $\varepsilon$ relative to the
error norm $\|\cdot\|_{\cL}$ whenever
\[
 T\ge
 \begin{cases}
 d^{O(1/\varepsilon)}
   & \text{for multi-event forecasting }(\cY=[0,1]^d,\ \cL=B_\infty^d),\\[3pt]
 d^{O(1/\varepsilon)}
   & \text{for multi-class forecasting }(\cY=\Delta_d,\ \cL=B_1^d),\\[3pt]
 \exp(O(1/\varepsilon))
   & \text{for Euclidean forecasting }(\cY=\cL=B_2^d),
 \end{cases}
\]
where $B_q^d=\{x\in\R^d:\|x\|_q\le1\}$ is the unit ball of the
$\ell_q$ norm in $d$ dimensions.
\end{theorem}
\end{samepage}

\begin{table}[htbp]
\centering
\small
\renewcommand{\arraystretch}{1.25}
\begin{tabular}{|c|c|c|c|}
\hline
\textbf{Setting} &
$\boldsymbol{\varepsilon}$\textbf{-net approaches} \cite{perchet2014approachability,fishelson2025full} &
\textbf{High-dimensional regime} & \textbf{This work}\\
\hline
Multi-event & $(1/\varepsilon)^{O(d)}$ & --- & $d^{O(1/\varepsilon)}$\\
\hline
Multi-class & $(1/\varepsilon)^{O(d)}$ &
$d^{O(\varepsilon^{-2}\log(1/\varepsilon))}$ \cite{peng2025high,fishelson2025high} &
$d^{O(1/\varepsilon)}$\\
\hline
Euclidean & $(1/\varepsilon)^{O(d)}$ &
$\exp\!\bigl(O(\varepsilon^{-2}\log(1/\varepsilon))\bigr)$ \cite{fishelson2025high} &
$\exp(O(1/\varepsilon))$\\
\hline
\end{tabular}
\caption{Rounds sufficient to achieve $\varepsilon$-calibration in the
three forecasting settings of Theorem~\ref{thm:headlines}.}
\label{tab:rates}
\end{table}

\paragraph{Setting.}
Our principal results concern the following distributional protocol.
As in many adversarial online learning problems, randomization is necessary
to attain meaningful rates (much as a deterministic strategy can be
exploited in rock--paper--scissors).
On round $t$, the forecaster announces a finite distribution $\mu_t$ over
forecasts in $\cY$, and the adversary chooses $y_t\in\cY$ after seeing it.
This interaction repeats for $T$ rounds. At its conclusion, we evaluate
the calibration error by comparing each forecast $p\in\cY$ with its
associated average outcome. On round $t$, the forecast $p$ receives
weight $\mu_t(p)$ and the outcome is $y_t$, so its weighted average
outcome over all rounds is
\[
 \nu_p=\frac{\sum_t\mu_t(p)y_t}{\sum_t\mu_t(p)},
\]
whenever its total weight is positive. For an error norm
$\|\cdot\|_{\cL}$ with unit ball $\cL$, the
\emph{$\cL$-calibration error} is
\[
 \CAL_{\cL,T}
 =\sum_p\left(\sum_t\mu_t(p)\right)\|\nu_p-p\|_{\cL}
 =\sum_p\left\|\sum_t\mu_t(p)(y_t-p)\right\|_{\cL},
\]
where the sum is over distinct forecast points with positive total
weight \cite{fishelson2025high}. Thus $\CAL_{\cL,T}/T$ is the average
calibration error. We say that an algorithm achieves
$\varepsilon$-calibration in $T$ rounds if $\CAL_{\cL,T}/T\le\varepsilon$
for every realized transcript $(\mu_0,y_0),\ldots,(\mu_{T-1},y_{T-1})$.
Our goal is to achieve $\varepsilon$-calibration in as few rounds $T$
as possible.

In some of the literature, this setting is called \emph{pseudo-calibration}.
We simply refer to it as \emph{calibration} and call the setting in which
the forecaster must sample a single forecast each round
\emph{sampled calibration}. The distributional notion is the more natural
formulation, as it suffices for downstream regret guarantees through the
calibrated-response approachability framework developed in
Section~\ref{sec:approachability}. For sampled calibration,
Appendix~\ref{app:sampling} obtains the same headline rates in expectation.

\paragraph{An elementary algorithm.}
The algorithm has two operations. As it observes outcomes $y_0,\ldots,y_t$,
it stores weighted averages $p_0,\ldots,p_t$ of the outcomes.
At each round, it outputs a weighted distribution over
previously stored forecasts. Both operations use harmonic weights: weights that decay
inversely with age, as in $1,\frac13,\frac15,\ldots,\frac1{2t+1}$,
normalized to sum to $1$.\footnote{Alternatively, one could use weights
$1,\frac12,\frac13,\ldots,\frac1t$. The shift to odd denominators is
convenient for the discrete Hilbert transform analysis; both choices
give the same rates.}

On the first round, we choose $p_{\mathrm{init}}\in\cY$ and announce
$\mu_0=\delta_{p_{\mathrm{init}}}$. For subsequent rounds $t\ge1$, we first compute and store a harmonically
weighted average of the past outcomes:
\[
 p_{t-1}\propto\sum_{s\le t-1}\frac{y_s}{2(t-s)-1}.
\]
Here $\propto$ means that the sum is normalized by its total weight,
$1+\frac13+\cdots+\frac1{2t-1}$. Thus $p_{t-1}\in\cY$.
We then construct our output $\mu_t$ for the round as a distribution
supported on the stored forecasts $p_0,\ldots,p_{t-1}$.
Let $\delta_{p_s}$ denote the Dirac distribution that puts all its mass
on $p_s$. Using the same harmonically decaying weights, we set
\[
 \mu_t\propto\sum_{s\le t-1}\frac{\delta_{p_s}}{2(t-s)-1}.
\]
Algorithm~\ref{alg:harmonic} gives the formal statement. The procedure
is independent of the norm, the horizon, and the target accuracy.

\paragraph{A general geometric guarantee.}
In fact, the same algorithm admits a general singly exponential horizon
bound governed by a matrix discrepancy objective.
For nonempty compact convex sets $\mathcal A,\mathcal B\subseteq\R^d$ and a
real $T\times T$ matrix $H$, define
\begin{equation}\label{eq:matrix-gamma}
 \Gamma_{\mathcal A,\mathcal B}^{(T)}(H)=
 \sup_{\substack{a_0,\ldots,a_{T-1}\in\mathcal A\\b_0,\ldots,b_{T-1}\in\mathcal B}}
                \sum_{t,s=0}^{T-1}H_{ts}\langle a_t,b_s\rangle.
\end{equation}

This objective relates to our calibration rates as follows.
Let $c_{\cY}$ be the center of mass of $\cY$, and define
\[
 \widetilde\cY=\cY-c_{\cY}.
\]
Let $\cL^\circ$ be the polar of the norm ball $\cL$:
\[
 \cL^\circ=\{r\in\R^d:\langle\ell,r\rangle\le1
                          \text{ for every }\ell\in\cL\}.
\]
Let $K_T$ be the shifted discrete Hilbert transform matrix:
\begin{equation}\label{eq:harmonic-kernel}
 (K_T)_{ts}=\frac1{2(t-s)-1}\qquad(0\le t,s<T).
\end{equation}
The quantity $\Gamma_{\widetilde\cY,\cL^\circ}^{(T)}(K_T)$ governs our
calibration rates. As we will see in Section~\ref{sec:harmonic},
$\Gamma_{\widetilde\cY,\cL^\circ}^{(T)}(K_T)=O_{\cY,\cL}(T)$ for every
forecast set $\cY$ and norm ball $\cL$.
Thus we define the geometric coefficient
\begin{equation}\label{eq:matrix-geometric-coefficient}
 \boxed{\qquad
 \gamma(\cY,\cL)=\sup_{T\ge1}\frac{\Gamma_{\widetilde\cY,\cL^\circ}^{(T)}(K_T)}T.
 \qquad}
\end{equation}
Our main algorithm obtains the following guarantee.

\begin{theorem}[General geometric guarantee]\label{thm:general-geometric-rate}
Let $\cL$ be any norm unit ball, let $\cY$ be a nonempty compact convex
forecast set.
For an absolute constant $C$, Algorithm~\ref{alg:harmonic}
achieves $\varepsilon$-calibration for every
\[
 T\ge\exp\!\left(\frac{C\,\gamma(\cY,\cL)}{\varepsilon}\right).
\]
\end{theorem}

\paragraph{Why harmonic weights?}
In Section~\ref{sec:matrices}, we consider a generalized version of our
algorithm with arbitrary admissible weights encoded by a matrix $H$.
Its calibration bound is governed by $\Gamma_{\widetilde\cY,\cL^\circ}^{(T)}(H)$, so we seek
weights that minimize this objective.

Section~\ref{sec:harmonic-optimality} proves that, for all sets
$\mathcal A,\mathcal B$ with $0\in\mathcal A$ and $\mathcal B=-\mathcal B$,
$K_T$ minimizes $\Gamma_{\mathcal A,\mathcal B}^{(T)}$ up to an absolute
constant factor (independent of $\mathcal A,\mathcal B,T$) over a class
of matrices $\mathcal H_T$ with non-negative entries below the diagonal,
non-positive entries above the diagonal, and fixed total absolute mass:
\[
 \Gamma_{\mathcal A,\mathcal B}^{(T)}(K_T)
 \le C\min_{H\in\mathcal H_T}\Gamma_{\mathcal A,\mathcal B}^{(T)}(H).
\]

Section~\ref{sec:harmonic} explains why the discrete Hilbert transform
matrix gives sublinear calibration. Our core tool is the following bound, proved in
Lemma~\ref{lem:harmonic-finite-moments}.

\begin{lemma}[Finite Hilbert matrix moment bounds]\label{lem:intro-hilbert-moments}
There is an absolute constant $C$ such that, for every $T\ge1$,
every $q\in\{2,4,8,\ldots\}$, and every $x\in\R^T$,
\[
 \|K_Tx\|_q\le Cq\,\|x\|_q.
\]
\end{lemma}

\paragraph{Downstream decisions and Blackwell approachability.}
The central motivation for calibration is to provide information that
downstream decision makers can use with guarantees. Consider an event
planner who earns \$100 from an outdoor picnic if it stays dry and loses
\$900 if it rains, while canceling yields zero. The planner prefers to
hold the picnic when the probability of rain is below $10\%$
and cancel when it is above $10\%$. Such a planner would benefit from
access to a calibrated forecast. If the planner faces this decision
repeatedly and chooses the action recommended by the forecasts,
calibration guarantees vanishing average regret relative to the best
fixed action in hindsight~\cite{foster1997calibrated,perchet2014approachability}.
The same guarantee holds for any decision maker whose payoffs are bounded
functions of the binary outcome.

Many sequential decision problems, however, involve vector-valued
outcomes. In the classic
experts problem, for example, a learner chooses among $d$ options on each
round and then observes a vector in $[0,1]^d$ specifying the loss of each
option. The relevant forecast for this decision maker would therefore
predict the upcoming loss vector.

The benefits of calibrated vector forecasts extend well beyond the
experts problem. Downstream decision makers can use these forecasts to
solve problems in the general framework of Blackwell
approachability~\cite{blackwell1956analog}, which encompasses regret
minimization and even the stronger notion of swap regret.
In approachability, a learner repeatedly chooses an action before observing an outcome;
together, the action and outcome determine a vector-valued payoff.
The learner seeks to steer the average payoff towards a convex target
set. Taking a best-response action to a calibrated forecast of the upcoming
outcome guarantees approachability~\cite[Section~4.4]{perchet2014approachability}.
Calibration with favorable dimension dependence at fixed accuracy therefore
provides a general route to high-dimensional approachability.

The first breakthrough in this line of work came for swap regret
minimization. The TreeSwap algorithm of Dagan, Daskalakis, Fishelson, and
Golowich~\cite{dagan2024external} and the concurrent work of Peng and
Rubinstein~\cite{peng2024fast} (both in the STOC 2024 Special Issue)
achieve average swap regret at most $\varepsilon$ with $d$ actions in
$(\log d/\varepsilon)^{O(1/\varepsilon)}$ rounds.
At fixed accuracy, this improves the polynomial dependence on the number
of experts in classical algorithms~\cite{blum2007external} to a
polylogarithmic dependence. These results raise a broader question:
which other approachability problems admit
favorable guarantees in high dimensions at fixed accuracy?

Our calibration results yield new algorithms for a broad family of
high-dimensional approachability problems. Each follows the same strategy,
which we call \emph{harmonic approachability}: best respond to the harmonic
forecasts and play the corresponding mixture of actions, as specified
in~\eqref{eq:harmonic-response-algorithm}. Beyond providing simultaneous
guarantees for all downstream decision makers, this approach gives the
best known dimension dependence at fixed accuracy for certain individual
approachability problems
(Appendices~\ref{app:skew-response} and~\ref{sec:app-orthants}).

Our guarantees for harmonic approachability are stronger than those
obtained by treating calibration as a black box. In
Section~\ref{sec:approachability}, we adapt the harmonic calibration proof
to approachability, generalizing the geometric parameter $\gamma$ to
capture the joint geometry of the outcome and action sets.

\paragraph{Related work.}
For binary sequential calibration, Dagan, Daskalakis, Fishelson, Golowich, Kleinberg, and Okoroafor~\cite{dagan2024breaking}
improve the classical $O(T^{2/3})$ cumulative error bound.
Dai and Fishelson~\cite{dai2026explicit} obtain an explicit
$O(T^{0.663})$ bound on expected cumulative $\ell_1$-calibration
error. Our results target dimension dependence rather than the best
binary accuracy rate.
Fishelson, Kleinberg, Okoroafor, Paes Leme, Schneider, and Teng~\cite{fishelson2025full} connect full swap regret with
several calibration objectives and develop discretization-based
algorithms. The high-dimensional calibration results of
Peng~\cite{peng2025high} and Fishelson, Golowich, Mohri, and
Schneider~\cite{fishelson2025high} are the principal comparisons discussed
above. Appendix~\ref{app:comparisons} compares the OLO parameter governing
TreeCal~\cite{fishelson2025high} with our harmonic coefficient. Even after taking the
better of the two OLO orientations, the parameters are incomparable:
we give families separating them in both directions.

For some downstream decision objectives, weaker forecast criteria
suffice. For binary outcomes, Kleinberg,
Paes Leme, Schneider, and Teng~\cite{kleinberg2023u} introduce
$U$-calibration to control external regret for agents with unknown
utilities, and Hu and Wu~\cite{hu2024predict} study calibration decision
loss to control swap regret across payoff-bounded decision tasks.
These criteria admit faster online rates than calibration.
Bairaktari, Hu, Nguyen, and Ullman~\cite{bairaktari2026testable} develop
a testable calibration measure with guarantees for suitably chosen
response maps, also for binary forecasts. Roth and
Shi~\cite{roth2024forecasting} and Rivkin, Deo-Campo Vuong, Kleinberg,
Onyeze, Sinanaj, and Tardos~\cite{rivkin2026improved} obtain improved
downstream swap-regret guarantees in low dimensions.

In general, however, controlling forecast errors uniformly over
unrestricted bounded response maps requires full calibration:
calibration error is precisely the largest discrepancy detected by a
bounded, forecast-dependent linear test. The weaker guarantees above
exploit low-dimensional outcomes or structure in the downstream decision
problem. In higher dimensions,
Roth and Shi~\cite{roth2024forecasting} assume smooth responses, while
Rivkin, Deo-Campo Vuong, Kleinberg, Onyeze, Sinanaj, and
Tardos~\cite[Theorem~5]{rivkin2026improved} require a known bound
on the number of actions available to each agent, with rates depending
on that bound. Such assumptions need not hold in a general
approachability instance. Our full calibration guarantees apply without
them and, for some individual high-dimensional approachability tasks,
yield the best known dimension dependence at fixed accuracy.

Our use of outcome-response maps is closely related to the response-based
approachability algorithms of Bernstein and
Shimkin~\cite{bernstein2015response}, subsequently applied to swap-regret
minimization by Anagnostides, Farina, Fishelson, Luo, and
Schneider~\cite{anagnostides2026swap}. These methods are part of a broader
study of reductions between approachability, online learning, and
forecasting. Abernethy, Bartlett, and Hazan~\cite{abernethy2011blackwell}
establish an equivalence between approachability and online linear
optimization, while Farina and Perdomo~\cite{farina2026efficient} use
multicalibration as an intermediate step from online learning to
transformation regret. The rates inherited through such reductions are
central to our setting: Dann, Mansour, Mohri, Schneider, and
Sivan~\cite{dann2025rate} show that preserving sublinear convergence need
not preserve optimal rates. Our harmonic approachability analysis tracks
the joint geometry of the outcome and action sets to obtain bounds
specific to the underlying approachability problem.

\paragraph{Organization.}
In Section~\ref{sec:matrices}, we generalize our algorithm to arbitrary
admissible weights and bound its calibration error by a matrix objective.
Section~\ref{sec:harmonic} analyzes this objective for harmonic weights
using moment bounds for the discrete Hilbert transform, yielding our
calibration rates. Section~\ref{sec:harmonic-optimality} shows that harmonic
weights minimize the same objective up to a universal constant factor
for every forecast set and error norm. Section~\ref{sec:approachability}
extends the algorithm to Blackwell approachability and derives rates
governed by the joint geometry of the outcome and action sets.

Appendix~\ref{app:sampling} extends the headline rates to sampled
calibration in expectation, and Appendix~\ref{app:implementation} gives
efficient implementations of the forecast updates and sampling.
Appendix~\ref{app:approachability-foundations}
develops the classical, response, and max--min--max formulations of
approachability and their connections to regret minimization.
Appendix~\ref{app:harmonic-optimality} supplies supporting estimates for
the optimality proof. Appendix~\ref{app:harmonic-refinements} gives further
bounds on the geometric coefficient, and Appendix~\ref{app:comparisons}
compares it with the OLO parameter governing TreeCal.
Appendix~\ref{app:ai-acknowledgements} describes how the paper developed
and the role of AI assistance.

\paragraph{Human acknowledgements.}
The authors thank Jon Schneider and Yuval Dagan for many helpful discussions
throughout this work.

\paragraph{AI acknowledgements.}
The authors used Gemini 3 Deep Think, ChatGPT, and Claude to help develop
proofs and prepare this manuscript. However, despite its simplicity, the
final algorithm was beyond the one-shot capabilities of the strongest
commercially available models under our \$200-per-month subscription budgets.
Here are two experiments in which GPT-6 Astra Ultra and Claude Fable 5.1
Max failed to resolve even a simplified version of the problem (orthant
calibration; see Appendix~\ref{sec:app-orthants}), despite an incredibly detailed prompt describing all the
relevant prior work and algorithmic ideas known to us:

\begin{itemize}
\item Prompt: \url{https://maxkfish.com/wp-content/uploads/2026/10/prompt.txt}
\item GPT: \url{https://chatgpt.com/share/6ac4672a-5968-83e8-8b96-25e4ad40b3ab}
\item Claude: \url{https://claude.ai/share/2ec24e46-ee20-4cbd-b682-1fee29989d2c}
\end{itemize}

We suspect that the problem was already within reach of internal models
unconstrained by commercial subscription budgets. Within our subscription
limits, however, we were able to succeed by prompting Gemini 3 Deep Think with a specific
human-developed idea: take a tree-weighted
average both before and after applying the response map. Gemini 3 Deep Think then helped us
substantially simplify this idea using harmonic weights and develop the
Hilbert-transform analysis. We suspect that steering the model with this
idea reduced the number of tokens needed by one or two orders of magnitude.
Appendix~\ref{app:ai-acknowledgements} tells the full story of how the paper
came together, including the ideas, failed approaches, and AI contributions
along the way.

\section{A General Matrix Formulation}\label{sec:matrices}

We now present a family of forecasting algorithms that includes
our main algorithm (Algorithm~\ref{alg:harmonic}) as a special case. Each algorithm is
specified by matrices $P,Q\in\R^{T\times T}$, chosen before observing
any outcomes. The matrix $P$ is lower triangular and row stochastic.
The matrix $Q$ is strictly lower triangular and row stochastic
(except for the first row, which is all zero). That is,
\begin{equation}\label{eq:matrix-decomposition}
 \begin{gathered}
 P_{ts},Q_{ts}\ge0,\qquad
 P_{ts}=0\quad(s>t),\qquad Q_{ts}=0\quad(s\ge t),\\
 \sum_{s=0}^{t}P_{ts}=1\quad(0\le t<T),\qquad
 \sum_{s=0}^{t-1}Q_{ts}=1\quad(1\le t<T).
 \end{gathered}
\end{equation}
Our generalized algorithm is simply
\begin{equation}\label{eq:general-matrix-algorithm}
 \boxed{\qquad
 p_t=\sum_{s\le t}P_{ts}y_s\quad(t\ge0),\qquad
 \mu_t=\sum_{s<t}Q_{ts}\delta_{p_s}\quad(t\ge1).
 \qquad}
\end{equation}
We choose $\mu_0=\delta_{p_{\mathrm{init}}}$ for an arbitrary
$p_{\mathrm{init}}\in\cY$.
Each stored forecast $p_t$ is a convex combination of
$y_0,\ldots,y_t$, so it belongs to $\cY$.
For $t\ge1$, $\mu_t$ is a probability distribution supported on
$p_0,\ldots,p_{t-1}$ and therefore depends only on
$y_0,\ldots,y_{t-1}$. Thus the algorithm is feasible and causal.

We now show that this algorithm has a calibration guarantee in terms
of the matrix objective from the introduction. To state the guarantee,
define
\[
 H=Q-P^{\mathsf T}.
\]
The constraints on $P,Q$ in \eqref{eq:matrix-decomposition} can equivalently
be written as constraints on $H$. Define
\begin{equation}\label{eq:admissible-matrices}
 \widehat{\mathcal H}_T=\left\{H\in\R^{T\times T}:\;
 \begin{array}{ll}
 H_{ts}\ge0 &(s<t),\\[3pt]
 H_{ts}\le0 &(t\le s),\\[3pt]
 \displaystyle \sum_{s=0}^{t-1}H_{ts}=1 &(1\le t<T),\\[6pt]
 \displaystyle -\sum_{t=0}^{s}H_{ts}=1 &(0\le s<T)
 \end{array}\right\}.
\end{equation}
Indeed, given $H\in\widehat{\mathcal H}_T$, we can recover $P$ and $Q$, defining a valid instance
of \eqref{eq:general-matrix-algorithm}. We have the following guarantee.

\begin{theorem}\label{thm:matrix-calibration}
For every $H\in\widehat{\mathcal H}_T$, the algorithm
\eqref{eq:general-matrix-algorithm} parameterized by $H$ satisfies
\begin{equation}\label{eq:matrix-calibration-bound}
 \CAL_{\cL,T}
 =O\!\left(\Gamma_{\widetilde\cY,\cL^\circ}^{(T)}(H)\right)
\end{equation}
for every outcome sequence $y_0,\ldots,y_{T-1}\in\cY$.
\end{theorem}

\begin{proof}
\textbf{Keep a separate bucket for each stored forecast.}
On each round $t\ge1$, Algorithm~\eqref{eq:general-matrix-algorithm}
announces a forecast distribution $\mu_t$ supported on the stored forecasts
$p_0,\ldots,p_{t-1}$, assigning weight $Q_{ts}$ to $p_s$.
Thus, for each forecast value $p$ and round $t\ge1$, the weight placed on
$p$ is
\[
 \mu_t(p)=\sum_{s:p_s=p}Q_{ts}.
\]

Two stored forecasts $p_s$ and $p_{s'}$ may happen to be the same
vector. Mass placed on these forecasts would have one shared contribution to the total calibration error,
which would only reduce the error by the triangle inequality.
We upper bound by treating all stored forecasts and $p_{\mathrm{init}}$ as
distinct.

\[
 \begin{aligned}
 \CAL_{\cL,T}
 &=\sum_{p\in\cY}\left\|\sum_{t=0}^{T-1}\mu_t(p)(y_t-p)\right\|_{\cL}\\
 &\le\|y_0-p_{\mathrm{init}}\|_{\cL}
       +\sum_{p\in\cY}\left\|\sum_{t=1}^{T-1}
                    \sum_{s:p_s=p}Q_{ts}(y_t-p_s)\right\|_{\cL}\\
 &\le\|y_0-p_{\mathrm{init}}\|_{\cL}
       +\sum_s\left\|\sum_{t=1}^{T-1}Q_{ts}(y_t-p_s)\right\|_{\cL}.
 \end{aligned}
\]
Note that $p_s=\sum_{t}P_{st}y_t$ and $\sum_{t}P_{st}=1$.
Therefore, $\sum_{t}P_{st}(y_t-p_s)=0$, and
\begin{equation}\label{eq:matrix-residual-identity}
 \left\|\sum_t Q_{ts}(y_t-p_s)\right\|_{\cL}
 =\left\|\sum_t(Q_{ts}-P_{st})(y_t-p_s)\right\|_{\cL}
 =\left\|\sum_t H_{ts}(y_t-p_s)\right\|_{\cL}.
\end{equation}
Center at $c_{\cY}$ and write
$\widetilde y_t=y_t-c_{\cY}$ and
$\widetilde p_s=p_s-c_{\cY}$.
Both vectors belong to $\widetilde\cY$.
Then
\[
 \begin{aligned}
 \left\|\sum_t H_{ts}(y_t-p_s)\right\|_{\cL}
 &=\left\|\sum_t H_{ts}(\widetilde y_t-\widetilde p_s)\right\|_{\cL}\\
 &\le\left\|\sum_t H_{ts}\widetilde y_t\right\|_{\cL}
       +\left\|\sum_t H_{ts}\widetilde p_s\right\|_{\cL}.
 \end{aligned}
\]
We bound the sum over $s$ of each term by
$\Gamma_{\widetilde\cY,\cL^\circ}^{(T)}(H)$. Recall the definition of
$\Gamma_{\widetilde\cY,\cL^\circ}^{(T)}(H)$ from \eqref{eq:matrix-gamma}:
\[
 \Gamma_{\widetilde\cY,\cL^\circ}^{(T)}(H)
 =\sup_{\substack{a_t\in\widetilde\cY\\b_s\in\cL^\circ}}
       \sum_{s,t}H_{ts}\langle a_t,b_s\rangle
 =\sup_{a_t\in\widetilde\cY}\sum_s
       \left\|\sum_{t}H_{ts}a_t\right\|_{\cL}.
\]
The first term is therefore bounded directly
\[
 \sum_s\left\|\sum_{t}H_{ts}\widetilde y_t\right\|_{\cL}
 \le\Gamma_{\widetilde\cY,\cL^\circ}^{(T)}(H).
\]
For the second term, take every $a_t$ to be the same vector
$a\in\widetilde\cY$ maximizing $\|a\|_{\cL}$. The same definition gives
\[
 \Gamma_{\widetilde\cY,\cL^\circ}^{(T)}(H)
 \ge\sum_s\left\|\left(\sum_{t}H_{ts}\right)a\right\|_{\cL}
 =\max_{y\in\widetilde\cY}\|y\|_{\cL}
       \sum_s\left|\sum_{t}H_{ts}\right|.
\]
Consequently,
\[
 \begin{aligned}
 \sum_s\left\|\sum_{t}H_{ts}\widetilde p_s\right\|_{\cL}
 &=\sum_s\left|\sum_{t}H_{ts}\right|
                  \|\widetilde p_s\|_{\cL}\\
 &\le\Gamma_{\widetilde\cY,\cL^\circ}^{(T)}(H).
 \end{aligned}
\]
Combining the two bounds with the initial-round contribution yields
\[
 \CAL_{\cL,T}
 \le\|y_0-p_{\mathrm{init}}\|_{\cL}
       +2\Gamma_{\widetilde\cY,\cL^\circ}^{(T)}(H).
\]
For $T\ge2$, we have $H_{00}=-1$ and $H_{10}=1$.
Taking $a_0=y_0-c_{\cY}$, $a_1=p_{\mathrm{init}}-c_{\cY}$,
and $a_t=0\in\widetilde\cY$ for $t\ge2$ therefore gives
$\Gamma_{\widetilde\cY,\cL^\circ}^{(T)}(H)
\ge\|y_0-p_{\mathrm{init}}\|_{\cL}$.
Thus $\CAL_{\cL,T}\le3\Gamma_{\widetilde\cY,\cL^\circ}^{(T)}(H)$.
For $T=1$, the bound holds with constant $2$ by the triangle inequality.
\end{proof}

In the next section, we show that setting $H=\widehat K_T$, a row-column
normalized shifted discrete Hilbert transform matrix, attains the desired
sublinear rates for $\Gamma_{\widetilde\cY,\cL^\circ}^{(T)}(H)$. In the section after that, we
demonstrate that $\widehat K_T$ minimizes $\Gamma_{\widetilde\cY,\cL^\circ}^{(T)}(H)$ over
$H\in\widehat{\mathcal H}_T$ up to an absolute constant factor, for every
forecast set $\cY$, norm ball $\cL$, and horizon $T\ge2$.
Thus, fixing the harmonic weights in our main algorithm
(Algorithm~\ref{alg:harmonic}) gives, up to an absolute constant factor,
as good a calibration bound as optimizing the weights in the general
algorithm \eqref{eq:general-matrix-algorithm}.

\section{Harmonic Growth and Geometric Calibration Rates}
\label{sec:harmonic}

We now state our main algorithm formally.
Let $h_t=\sum_{j=0}^{t-1}1/(2j+1)$, and define
\begin{equation}\label{eq:harmonic-algorithm-kernel}
 (\widehat K_T)_{ts}=
 \begin{cases}
  1/[h_t(2(t-s)-1)],&s<t,\\
  -1/[h_{s+1}(2(s-t)+1)],&t\le s,
 \end{cases}
 \qquad 0\le t,s<T.
\end{equation}
Each positive lower row sums to $1$, and each negative upper column
sums to $-1$, so $\widehat K_T\in\widehat{\mathcal H}_T$.
Our main algorithm is exactly algorithm~\eqref{eq:general-matrix-algorithm}
with $H=\widehat K_T$.

\begin{algorithm}[H]
\caption{Harmonic averaging and replay}\label{alg:harmonic}
\begin{algorithmic}[1]
\State Choose $p_{\mathrm{init}}\in\cY$; announce $\mu_0=\delta_{p_{\mathrm{init}}}$; observe $y_0$.
\State Store $p_0=y_0$ and set $h=1$.
\For{$t=1,2,\ldots$}
  \State Announce $\mu_t\gets\sum_{s=0}^{t-1}\delta_{p_s}/[h(2(t-s)-1)]$.
  \Statex \hspace{\algorithmicindent}
     (Aggregate masses of equal point values if needed.)
  \State Observe $y_t\in\cY$ and append it to the outcome list.
  \State $h\gets h+1/(2t+1)$.
  \State Store $p_t\gets\sum_{s=0}^{t}y_s/[h(2(t-s)+1)]$.
\EndFor
\end{algorithmic}
\end{algorithm}

The algorithm requires neither $\cL$ nor $T$ as input.
Appendix~\ref{app:sampling} treats sampled forecasts and their implementation.

We compare the algorithm matrix $\widehat K_T$ with the unnormalized shifted discrete
Hilbert transform matrix $K_T$ introduced in
\eqref{eq:harmonic-kernel}. Recall that
\[
 (K_T)_{ts}=\frac1{2(t-s)-1}\qquad(0\le t,s<T).
\]
In Appendix~\ref{app:harmonic-normalization}, we prove the following lemma
by a simple accounting for the normalization terms.

\begin{lemma}[Comparison of harmonic normalizations]
\label{lem:harmonic-kernel-comparison}
For nonempty compact convex sets $\mathcal A,\mathcal B\subseteq\R^d$ with
$\mathcal B=-\mathcal B$, and every $T\ge1$,
\begin{equation}\label{eq:harmonic-kernel-comparison}
 \frac13\,\Gamma_{\mathcal A,\mathcal B}^{(T)}(K_T)
 \le h_T\Gamma_{\mathcal A,\mathcal B}^{(T)}(\widehat K_T)
 \le3\,\Gamma_{\mathcal A,\mathcal B}^{(T)}(K_T).
\end{equation}
\end{lemma}

Combining this lemma with
Theorem~\ref{thm:matrix-calibration} gives our main theorem.
Recall the definition of $\gamma$:
\[
 \gamma(\cY,\cL)=\sup_{T\ge1}
 \frac{\Gamma_{\widetilde\cY,\cL^\circ}^{(T)}(K_T)}T.
\]
In Theorem~\ref{thm:john}, we show that
$\Gamma_{\widetilde\cY,\cL^\circ}^{(T)}(K_T)=O_{\cY,\cL}(T)$ for all $\cY,\cL$.
This shows that $\gamma(\cY,\cL)$ is finite and well-defined.

\begin{geometricratetheorem}
Let $\cL$ be any norm unit ball, let $\cY$ be a nonempty compact convex
forecast set.
For an absolute constant $C$, Algorithm~\ref{alg:harmonic}
achieves $\varepsilon$-calibration for every
\[
 T\ge\exp\!\left(\frac{C\,\gamma(\cY,\cL)}{\varepsilon}\right).
\]
\end{geometricratetheorem}
\begin{proof}[Proof of Theorem~\ref{thm:general-geometric-rate}]
Theorem~\ref{thm:matrix-calibration},
Lemma~\ref{lem:harmonic-kernel-comparison}, and
$h_T\ge\frac12\log T$ give, for every $T\ge2$,
\begin{equation}\label{eq:harmonic-comparison-calibration}
 \CAL_{\cL,T}
 \le3\Gamma_{\widetilde\cY,\cL^\circ}^{(T)}(\widehat K_T)
 \le\frac{18}{\log T}\Gamma_{\widetilde\cY,\cL^\circ}^{(T)}(K_T)
\le\frac{18\,\gamma(\cY,\cL)T}{\log T}.
\end{equation}
Consequently, for every $T\ge\exp(18\,\gamma(\cY,\cL)/\varepsilon)$,
\[
 \varepsilon\ge\frac{18\,\gamma(\cY,\cL)}{\log T}
 \ge\frac{\CAL_{\cL,T}}T
\]
as desired.
\end{proof}

We now establish the following bounds, with additional examples proved in
Appendix~\ref{app:harmonic-refinements}:

\begin{table}[htbp]
\centering
\small
\renewcommand{\arraystretch}{1.2}
\begin{tabular}{lccc}
\toprule
Geometry & $\cY$ & $\cL$ & Bound on $\gamma(\cY,\cL)$ \\
\midrule
Euclidean ball & $B_2^d$ & $B_2^d$ & $O(1)$ \\
Multi-event & $[0,1]^d$ & $B_\infty^d$ & $O(\log(2d))$ \\
Multi-class & $\Delta_d$ & $B_1^d$ & $O(\log(2d))$ \\
$\ell_p$ balls ($2\le p\le\infty$) & $B_p^d$ & $B_p^d$ & $O(\min\{p,\log(2d)\})$ \\
$\ell_p$ balls ($1\le p\le2$) & $B_p^d$ & $B_p^d$ & $O(\min\{1/(p-1),\log(2d)\})$ \\
Symmetric polytopes ($m$ constraints) & $\cL$ & $\cL$ & $O(\log(2m))$ \\
Product of Euclidean balls & $\prod_{j=1}^k B_2^{d_j}$ & $\prod_{j=1}^k B_2^{d_j}$ & $O(\log(2k))$ \\
Arbitrary norm & $\cY$ & $\cL$ & $O(\diam_{\cL}(\cY)\sqrt d)$ \\
\bottomrule
\end{tabular}
\caption{Bounds on $\gamma(\cY,\cL)$ for different geometries.}
\label{tab:geometric-gamma-bounds}
\end{table}

These bounds give dimension-independent calibration horizons for the
Euclidean ball and horizons polynomial in $d$ for multi-event and
multi-class forecasting, at every fixed accuracy.

\subsection{Higher Moments of the Finite Hilbert Matrix}
\label{sec:harmonic-upper-bound}

The following lemma is our main tool for upper bounding $\gamma$.
For $x\in\R^T$ and $1\le q<\infty$, define
\[
 \|x\|_q=\left(\sum_{t=0}^{T-1}|x_t|^q\right)^{1/q}.
\]
\begin{lemma}[Finite Hilbert matrix moment bounds]
\label{lem:harmonic-finite-moments}
There is an absolute constant $C$ such that, for every $T\ge1$,
every $q=2^k$ with integer $k\ge1$, and every $x\in\R^T$,\footnote{\label{fn:hilbert-interpolation}Gemini
points out that Riesz--Thorin interpolation between successive powers
of two yields $\|K_Tx\|_q\le3q\,\|x\|_q$ for every real $q\ge2$.
We keep the elementary dyadic result in the main proof and indicate
later when the full extension is used.}
\begin{equation}\label{eq:harmonic-finite-moments}
 \|K_Tx\|_q\le Cq\,\|x\|_q.
\end{equation}
\end{lemma}
\begin{proof}
Fix $T$ and write $K=K_T$. We first prove
$\|Kx\|_2\le(\pi/2)\|x\|_2$ using the Fourier maps $U,V$:
\[
 (Uy)(\theta)=\sum_{t=0}^{T-1}y_t e^{i(2t)\theta},\qquad
 (Vx)(\theta)=\sum_{s=0}^{T-1}x_s e^{i(2s+1)\theta},
 \qquad 0\le\theta\le\pi.
\]
For every integer $k$,
\[
 \frac1\pi\int_0^\pi e^{ik\theta}\,d\theta
 =\begin{cases}
  1,&k=0,\\
  \dfrac{2i}{\pi k},&k\text{ odd},\\
  0,&k\ne0\text{ even}.
 \end{cases}
\]
Thus these maps are \emph{isometries}: they preserve the $2$-norm
when we average over $[0,\pi]$. Indeed,
\[
\begin{aligned}
 \frac1\pi\int_0^\pi|(Uy)(\theta)|^2\,d\theta
 &=\frac1\pi\int_0^\pi
   \left(\sum_{t=0}^{T-1}y_t e^{i(2t)\theta}\right)
   \left(\sum_{u=0}^{T-1}\overline{y_u}e^{-i(2u)\theta}\right)\,d\theta\\
 &=\sum_{t,u<T}y_t\overline{y_u}\,
   \frac1\pi\int_0^\pi e^{i2(t-u)\theta}\,d\theta\\
 &=\sum_{t<T}|y_t|^2=\|y\|_2^2.
\end{aligned}
\]
The same calculation gives
$\frac1\pi\int_0^\pi|(Vx)(\theta)|^2\,d\theta=\|x\|_2^2$.
These maps also give a decomposition of $K$, since
\[
 K_{ts}=\frac1{2(t-s)-1}
 =\frac i2\int_0^\pi
 e^{i(2s+1)\theta}\overline{e^{i(2t)\theta}}\,d\theta.
\]
Consequently, for $x,y\in\mathbb C^T$,
\[
 \langle y,Kx\rangle
 =\frac i2\int_0^\pi(Vx)(\theta)\overline{(Uy)(\theta)}\,d\theta.
\]
Cauchy--Schwarz and the norm identities above give
\[
 |\langle y,Kx\rangle|
 \le\frac\pi2
 \left(\frac1\pi\int_0^\pi|(Vx)(\theta)|^2\,d\theta\right)^{1/2}
 \left(\frac1\pi\int_0^\pi|(Uy)(\theta)|^2\,d\theta\right)^{1/2}
 =\frac\pi2\|x\|_2\|y\|_2.
\]
Taking the supremum over $\|y\|_2=1$ proves
$\|Kx\|_2\le(\pi/2)\|x\|_2$.

The higher moments follow from the finite cancellation identity
\begin{equation}\label{eq:finite-cotlar}
 (Kx)^2=K\bigl(x\,(K-K^{\mathsf T})x\bigr)+R(x),
\end{equation}
where products and squares are coordinatewise and
\[
 R_t(x)=\sum_{s,u<T}
       \frac{K_{ts}K_{tu}x_sx_u}{1-4(s-u)^2}.
\]
Indeed, for all $t,s,u$, including repeated indices,
\[
 K_{ts}K_{tu}
 =\frac12(K_{ts}-K_{tu})(K_{su}-K_{us})
   +\frac{K_{ts}K_{tu}}{1-4(s-u)^2}.
\]
Multiplying by $x_sx_u$, summing over $s,u$, and interchanging $s,u$
in one of the two terms proves \eqref{eq:finite-cotlar}.

To bound the remainder, observe that
\[
 \sum_{j\in\mathbb Z}\frac1{|1-4j^2|}
 =1+2\sum_{j\ge1}\frac1{4j^2-1}=2.
\]
Thus $2|ab|\le|a|^2+|b|^2$ gives
\[
 |R_t(x)|\le2\sum_{s<T}K_{ts}^2|x_s|^2.
\]
Every row and column sum of the matrix $J_{ts}=K_{ts}^2$ is at most
$\sum_{j\in\mathbb Z}(2j-1)^{-2}=\pi^2/4$. Hence, for every $q\ge1$,
\[
 \|R(x)\|_q\le2\|J(|x|^2)\|_q
 \le\frac{\pi^2}{2}\|x\|_{2q}^2.
\]
Write $c_q=\|K\|_{q\to q}$ over complex vectors. Reversing the order
of the coordinates conjugates $K$ to $K^{\mathsf T}$, so their
$q$-operator norms agree. Taking the $q$-norm in
\eqref{eq:finite-cotlar} and using H\"older therefore gives
\[
 \|Kx\|_{2q}^2
 \le c_q\|x\|_{2q}
       \bigl(\|Kx\|_{2q}+\|K^{\mathsf T}x\|_{2q}\bigr)
       +\frac{\pi^2}{2}\|x\|_{2q}^2.
\]
Taking the supremum over $\|x\|_{2q}=1$ yields
$c_{2q}^2\le2c_qc_{2q}+\pi^2/2$, and hence
\begin{equation}\label{eq:finite-hilbert-recurrence}
 c_{2q}\le c_q+\sqrt{c_q^2+\pi^2/2}.
\end{equation}
Starting from $c_2\le\pi/2\le\pi/\sqrt2$, the identity
$\cot(\theta/2)=\cot\theta+\sqrt{1+\cot^2\theta}$ proves by induction that
\begin{equation}\label{eq:finite-dyadic-hilbert}
 c_q\le\frac\pi{\sqrt2}\cot\frac{\pi}{2q}\le\sqrt2\,q,
 \qquad q=2^k,\quad k\in\mathbb Z_{\ge1}.
\end{equation}
This proves the lemma, and the same bounds hold for
$K^{\mathsf T}$.
\end{proof}

\subsection{Geometric Bounds}\label{sec:geometry}

We are now ready to prove our bounds on
$\gamma(\cY,\cL)=\sup_{T\ge1}
\frac{\Gamma_{\widetilde\cY,\cL^\circ}^{(T)}(K_T)}T$.
For a matrix $Z$, let $Z_{:,s}$ denote column $s$ and $Z_{i,:}$ denote row $i$.
We want to upper bound quantities of the form
\[
 \Gamma_{\widetilde\cY,\cL^\circ}^{(T)}(K_T)
 =\sup_{\substack{a_t\in\widetilde\cY\\b_s\in\cL^\circ}}
   \sum_{t,s<T}(K_T)_{ts}\langle a_t,b_s\rangle
 =\sup_{A\in\widetilde\cY^T}
   \sum_{s<T}\|(AK_T)_{:,s}\|_{\cL}.
\]
Here $\widetilde\cY^T$ is the set of $d\times T$ matrices whose columns
belong to $\widetilde\cY$.

\begin{lemma}[The three principal examples]\label{lem:three-bodies}
For every $d\ge2$,
\begin{equation}\label{eq:principal-geometric-bounds}
 \gamma(B_2^d,B_2^d)\le\frac\pi2,\qquad
 \gamma(B_\infty^d,B_\infty^d)=\gamma(B_1^d,B_1^d)=O(\log d).
\end{equation}
Consequently, Algorithm~\ref{alg:harmonic} attains $\varepsilon$-calibration
for every
\[
 \begin{aligned}
 T&\ge\exp(C/\varepsilon)
 &&\text{if }\cY=\cL=B_2^d,\\[3pt]
 T&\ge d^{C/\varepsilon}
 &&\text{if }\cY\in\{[0,1]^d,B_\infty^d\},\ \cL=B_\infty^d,\\[3pt]
 T&\ge d^{C/\varepsilon}
 &&\text{if }\cY\in\{\Delta_d,B_1^d\},\ \cL=B_1^d,
 \end{aligned}
\]
for an absolute constant $C$.
\end{lemma}
\begin{proof}
Fix $T$ and write $K=K_T$.
Note the following norm inequality and power mean inequality.
For $x\in\R^n$ and
$1\le p\le q\le\infty$,
\[
 \|x\|_q\le\|x\|_p,
 \qquad
 n^{-1/p}\|x\|_p\le n^{-1/q}\|x\|_q.
\]
These are standard consequences of H\"older's inequality.

\textbf{Euclidean ball ($\gamma(B_2^d,B_2^d)$).}
Set $Z=AK$. We want to bound $\frac1T\sum_{s<T}\|Z_{:,s}\|_2$ uniformly over
$A\in(B_2^d)^T$.
First apply the power-mean inequality with exponents $1$ and $2$
to the column norms:
\[
 \frac1T\sum_{s<T}\|Z_{:,s}\|_2
 \le\left(\frac1T\sum_{s<T}\|Z_{:,s}\|_2^2\right)^{1/2}.
\]
The $q=2$ bound in Lemma~\ref{lem:harmonic-finite-moments}
also holds for $K^{\mathsf T}$ and gives
\[
 \begin{aligned}
 \sum_{s<T}\|Z_{:,s}\|_2^2
 &=\sum_{i=1}^d\|Z_{i,:}\|_2^2
 =\sum_{i=1}^d\|A_{i,:}K\|_2^2\\
 &\le\left(\frac\pi2\right)^2
       \sum_{i=1}^d\|A_{i,:}\|_2^2\\
 &=\left(\frac\pi2\right)^2\sum_{t<T}\|A_{:,t}\|_2^2
 \le\left(\frac\pi2\right)^2T.
 \end{aligned}
\]
Combining the two estimates gives $\frac1T\sum_{s<T}\|Z_{:,s}\|_2\le\pi/2$.
Taking the supremum over the inputs and the horizon gives
$\gamma(B_2^d,B_2^d)\le\pi/2$.

\textbf{Multi-event forecasting ($\gamma(B_\infty^d,B_\infty^d)$).}
Set $Z=AK$. We want to bound
$\frac1T\sum_{s<T}\|Z_{:,s}\|_\infty$ uniformly over
$A\in(B_\infty^d)^T$.
Fix $q\ge2$ to be a power of $2$. First apply the power-mean inequality
with exponents $1$ and $q$ to the column norms, then the norm inequality
with exponents $q$ and $\infty$ within each column:
\[
 \begin{aligned}
 \frac1T\sum_{s<T}\|Z_{:,s}\|_\infty
 &\le\left(\frac1T\sum_{s<T}\|Z_{:,s}\|_\infty^q\right)^{1/q}\\
 &\le\left(\frac1T\sum_{s<T}\|Z_{:,s}\|_q^q\right)^{1/q}.
 \end{aligned}
\]
The bound in Lemma~\ref{lem:harmonic-finite-moments}
also holds for $K^{\mathsf T}$ and gives
\[
 \begin{aligned}
 \sum_{s<T}\|Z_{:,s}\|_q^q
 &=\sum_{i=1}^d\|Z_{i,:}\|_q^q
 =\sum_{i=1}^d\|A_{i,:}K\|_q^q\\
 &\le(Cq)^q\sum_{i=1}^d\|A_{i,:}\|_q^q\\
 &=(Cq)^q\sum_{t<T}\|A_{:,t}\|_q^q
 \le(Cq)^qTd,
 \end{aligned}
\]
where the final inequality follows from the power-mean inequality
with exponents $q$ and $\infty$, since $\|A_{:,t}\|_\infty\le1$.
Combining these estimates gives
\[
 \frac1T\sum_{s<T}\|Z_{:,s}\|_\infty\le Cq\,d^{1/q}.
\]
Choose the least power of two $q\ge\max\{2,\log d\}$.
Then $d^{1/q}\le e$ and $q=O(\log d)$.
Taking the supremum over the inputs and the horizon proves the cube bound.

\textbf{Multi-class forecasting ($\gamma(B_1^d,B_1^d)$).}
Swapping the row and column tests in the matrix objective gives
\[
 \frac1T\Gamma_{B_1^d,B_\infty^d}^{(T)}(K)
 =\frac1T\Gamma_{B_\infty^d,B_1^d}^{(T)}(K^{\mathsf T})
 =\frac1T\Gamma_{B_\infty^d,B_1^d}^{(T)}(K)
 =O(\log d).
\]
The second equality holds because reversing both index orders turns
$K$ into $K^{\mathsf T}$, and permuting indices does not change
the supremum. Thus $\gamma(B_1^d,B_1^d)=\gamma(B_\infty^d,B_\infty^d)$.
Theorem~\ref{thm:general-geometric-rate} now gives the calibration bounds.
\end{proof}

\begin{theorem}[General test sets]\label{thm:john}
Let $\mathcal A,\mathcal B\subseteq\R^d$ be nonempty compact convex sets
with $0\in\mathcal A$ and $\mathcal B=-\mathcal B$, and write
$r=\max_{a\in\mathcal A,\,b\in\mathcal B}\langle a,b\rangle$.
For every $T\ge1$,
\[
 \Gamma_{\mathcal A,\mathcal B}^{(T)}(K_T)
 \le\frac\pi2 r\sqrt d\,T.
\]
In particular, for every nonempty compact convex forecast set $\cY\subseteq\R^d$
and norm ball $\cL\subseteq\R^d$,
\begin{equation}\label{eq:john}
 \gamma(\cY,\cL)\le\frac\pi2\diam_{\cL}(\cY)\sqrt d<\infty.
\end{equation}
\end{theorem}
\begin{proof}
Fix $T$ and write $K=K_T$.
If $\mathcal B=\{0\}$, the general bound is immediate. Otherwise,
project $\mathcal A$ onto $\operatorname{span}(\mathcal B)$ and work
in that subspace. This preserves all dot products and does not increase
the dimension, so we may assume that $\mathcal B$ has nonempty interior.
The norm with unit ball $\mathcal B^\circ$ is
$\|v\|_{\mathcal B^\circ}=\max_{b\in\mathcal B}\langle v,b\rangle$.
John's ellipsoid theorem~\cite[Theorem J and its symmetric consequence]{ball1992ellipsoids}
gives an invertible linear map $S:\R^d\to\R^d$ such that
\[
 \|v\|_{\mathcal B^\circ}\le\|Sv\|_2\le\sqrt d\,\|v\|_{\mathcal B^\circ}
 \qquad(v\in\R^d).
\]
Every $a\in\mathcal A$ satisfies $\|a\|_{\mathcal B^\circ}\le r$,
and hence $\|Sa\|_2\le r\sqrt d$.
Fix $A\in\mathcal A^T$ and set $Z=AK$, so $SZ=(SA)K$.
The first three inequalities below use the ellipsoid comparison,
the power-mean inequality across columns, and the $q=2$ moment bound
on coordinate rows, respectively:
\[
 \begin{aligned}
 \sum_{s<T}\|Z_{:,s}\|_{\mathcal B^\circ}
 &\le\sum_{s<T}\|(SZ)_{:,s}\|_2
 \le\left(T\sum_{s<T}\|(SZ)_{:,s}\|_2^2\right)^{1/2}\\
 &=\left(T\sum_{i=1}^d\|(SZ)_{i,:}\|_2^2\right)^{1/2}
 \le\frac\pi2\left(T\sum_{i=1}^d\|(SA)_{i,:}\|_2^2\right)^{1/2}\\
 &=\frac\pi2\left(T\sum_{t<T}\|(SA)_{:,t}\|_2^2\right)^{1/2}
 \le\frac\pi2 r\sqrt d\,T.
 \end{aligned}
\]
Taking the supremum over $A\in\mathcal A^T$ proves the
general bound. For $\mathcal A=\widetilde\cY$ and $\mathcal B=\cL^\circ$,
we have $r=\max_{a\in\widetilde\cY}\|a\|_{\cL}\le\diam_{\cL}(\cY)$,
because $c_{\cY}\in\cY$. Dividing by $T$ and taking the supremum over
the horizon proves \eqref{eq:john}.
\end{proof}

The $\sqrt d$ bound in Theorem~\ref{thm:john} can be loose: for
$\cY=\cL=B_\infty^d$ it gives $O(\sqrt d)$, whereas the more refined
argument above gives $O(\log d)$. One might therefore conjecture that
$\gamma(\cL,\cL)=O(\log d)$ for every norm ball $\cL\subseteq\R^d$.
However, Theorem~\ref{thm:bourgain-lower-bound} in
Appendix~\ref{app:harmonic-refinements} uses a construction of
Bourgain~\cite{bourgain1983sinus} to give norm balls with
$\gamma(\cL,\cL)=\Omega(d^{1/3})$. Thus the general guarantee in
Theorem~\ref{thm:general-geometric-rate} cannot give a
polynomial-in-$d$ calibration horizon for every norm ball.

We defer the remaining cases, including $\ell_p$ balls, symmetric polytopes,
and products of Euclidean balls, to Appendix~\ref{app:harmonic-refinements}.

\section{Universal Optimality of the Discrete Hilbert Transform Matrix}
\label{sec:harmonic-optimality}

The idea of using harmonic weights for high-dimensional calibration
stems from the TreeCal algorithm~\cite{fishelson2025high,peng2025high}.
The algorithm
organizes historical outcomes into nested blocks of time and mixes
forecasts associated with different levels of the tree.\footnote{This
idea originates in the earlier TreeSwap algorithm for high-dimensional
swap regret~\cite{dagan2024external,peng2024fast}.}
This corresponds to weighting historical data with harmonically
decaying importance. A deeper reason for these weights is revealed by
our proof that harmonic weights in algorithm~\eqref{eq:general-matrix-algorithm}
are optimal up to a constant factor for the matrix objective, a result
that rests on a fundamental property of the Hilbert transform.

Recall the shifted discrete Hilbert transform matrix $K_T$ and its
row-column normalized version $\widehat K_T$ from
\eqref{eq:harmonic-kernel} and \eqref{eq:harmonic-algorithm-kernel}:
\[
 (K_T)_{ts}=\frac1{2(t-s)-1}\qquad(0\le t,s<T),
\]
\[
 (\widehat K_T)_{ts}=
 \begin{cases}
  1/[h_t(2(t-s)-1)],&s<t,\\
  -1/[h_{s+1}(2(s-t)+1)],&t\le s,
 \end{cases}
 \qquad 0\le t,s<T,
\]
where $h_t=\sum_{j=0}^{t-1}1/(2j+1)$.
Also recall that algorithm~\eqref{eq:general-matrix-algorithm} can be
instantiated with any matrix in the class
\[
 \widehat{\mathcal H}_T=\left\{H\in\R^{T\times T}:\;
 \begin{array}{ll}
 H_{ts}\ge0 &(s<t),\\[3pt]
 H_{ts}\le0 &(t\le s),\\[3pt]
 \displaystyle\sum_{s=0}^{t-1}H_{ts}=1 &(1\le t<T),\\[6pt]
 \displaystyle-\sum_{t=0}^{s}H_{ts}=1 &(0\le s<T)
 \end{array}\right\}.
\]
Choosing $H=\widehat K_T$ gives our harmonic algorithm. The following
corollary shows that this choice minimizes the matrix objective up to
an absolute constant factor, simultaneously for every pair of test sets.

\begin{corollary}[Optimality of harmonic weights]
\label{cor:optimality-harmonic-weights}
There is an absolute constant $C$ such that, for every pair of nonempty
compact convex sets $\mathcal A,\mathcal B\subseteq\R^d$ with
$0\in\mathcal A$ and $\mathcal B=-\mathcal B$, every $T\ge2$, and
every $H\in\widehat{\mathcal H}_T$,
\begin{equation}\label{eq:optimality-harmonic-weights}
 \Gamma_{\mathcal A,\mathcal B}^{(T)}(\widehat K_T)
 \le C\,\Gamma_{\mathcal A,\mathcal B}^{(T)}(H).
\end{equation}
\end{corollary}

We prove a stronger property: $K_T$ is optimal over the more general class
\begin{equation}\label{eq:optimality-mass-class}
 \mathcal H_T=\left\{H\in\R^{T\times T}:\begin{array}{l}
 H_{ts}\ge0\text{ for }s<t,\\
 H_{ts}\le0\text{ for }t\le s,\\
 \|H\|_1=\|K_T\|_1=T(2h_T-1)
 \end{array}\right\}.
\end{equation}
Here $\|H\|_1=\sum_{t,s}|H_{ts}|$ is the entrywise absolute mass.
The class $\mathcal H_T$ has the same sign constraints as
$\widehat{\mathcal H}_T$, but constrains only the total absolute mass,
not the masses of individual lower rows and upper columns. We choose
this total mass so that $K_T\in\mathcal H_T$: for each $0\le j<T$,
there are $2T-2j-1$ entries of $K_T$ with magnitude $1/(2j+1)$, and hence
\[
 \|K_T\|_1
 =\sum_{j=0}^{T-1}\frac{2T-2j-1}{2j+1}
 =2T\sum_{j=0}^{T-1}\frac1{2j+1}-T
 =T(2h_T-1).
\]

Our main theorem in this section is the following.

\begin{theorem}[Universal optimality of the discrete Hilbert transform]
\label{thm:universal-harmonic-comparison}
There is an absolute constant $C$ such that, for every pair of nonempty
compact convex sets $\mathcal A,\mathcal B\subseteq\R^d$ with
$0\in\mathcal A$ and $\mathcal B=-\mathcal B$, every $T\ge2$, and
every $H\in\mathcal H_T$,
\begin{equation}\label{eq:optimality-hilbert-comparison}
 \Gamma_{\mathcal A,\mathcal B}^{(T)}(K_T)
 \le C\,\Gamma_{\mathcal A,\mathcal B}^{(T)}(H).
\end{equation}
The constant is independent of the sets, their dimension, and the horizon.
\end{theorem}

Corollary~\ref{cor:optimality-harmonic-weights} follows from
Lemma~\ref{lem:harmonic-kernel-comparison} and the fact that every matrix in
$\widehat{\mathcal H}_T$ has total absolute mass $2T-1$, so
\begin{equation}\label{eq:optimality-feasible-inclusion}
 \frac{T(2h_T-1)}{2T-1}\widehat{\mathcal H}_T\subseteq\mathcal H_T.
\end{equation}

\subsection{Scalar Warm-Up}\label{sec:scalar-warmup}

As a warm-up to proving Theorem~\ref{thm:universal-harmonic-comparison},
we show that $\Gamma_{\mathcal A,\mathcal B}^{(T)}(H)=\Omega(T)$ for
every $H\in\mathcal H_T$. This matches the $T$
dependence of $\Gamma_{\mathcal A,\mathcal B}^{(T)}(K_T)=O(T)$ proved
in Theorem~\ref{thm:john}. In Section~\ref{sec:optimality-full-proof},
we use a more involved construction to obtain a lower bound on
$\Gamma_{\mathcal A,\mathcal B}^{(T)}(H)$ that matches not only the
$\Theta(T)$ growth of $\Gamma_{\mathcal A,\mathcal B}^{(T)}(K_T)$,
but also its dependence on $d$ through the geometry of $\mathcal A$ and
$\mathcal B$ (see Table~\ref{tab:geometric-gamma-bounds}).

We prove the $\Gamma_{\mathcal A,\mathcal B}^{(T)}(H)=\Omega(T)$ lower bound
in the scalar setting $\mathcal A=\mathcal B=[-1,1]$.
We obtain the lower bound for general sets by embedding this scalar construction.
Note that $\|H\|_1=\Theta(T\log T)$ for all $H\in\mathcal H_T$.
Below we prove the general lower bound
$\Gamma_{[-1,1],[-1,1]}^{(T)}(H)=\Omega(\|H\|_1/\log T)$.

\begin{lemma}[Dyadic scalar tests]\label{lem:optimality-scalar-baseline}
Let $T\ge2$ and let $H\in\R^{T\times T}$ satisfy $H_{ts}\ge0$ for
$t>s$ and $H_{ts}\le0$ for $t\le s$. With $M=\|H\|_1$,
\begin{equation}\label{eq:optimality-scalar-baseline}
 \Gamma_{[-1,1],[-1,1]}^{(T)}(H)
 =\Omega\left(\frac{M}{\log T}\right).
\end{equation}
\end{lemma}
\begin{proof}
Put $q=\lceil\log_2 T\rceil$ and pad $H$ with zero rows and columns
to size $2^q\times2^q$. View the row and column indices as the leaves,
from left to right, of a complete binary tree of depth $q$,
with the root at depth $0$ and the leaves at depth $q$.
Partition the $2^{2q}$ entries of $H$ into $q+1$ parts: entry $H_{ts}$
belongs to part $\ell$ if the lowest common ancestor of leaves $t$
and $s$ has depth $\ell$. Thus the diagonal entries form part $q$.

\begin{center}
 % Matrix parts indexed by the depth of the lowest common ancestor.
\definecolor{dyadiczero}{HTML}{ED8A86}
\definecolor{dyadicone}{HTML}{DDA451}
\definecolor{dyadictwo}{HTML}{8BCB91}
\definecolor{dyadicthree}{HTML}{83B7E5}
\begin{tikzpicture}[x=0.5cm,y=-0.5cm,font=\small]
 \fill[dyadiczero] (0,0) rectangle (8,8);
 \foreach \k in {0,4}
   \fill[dyadicone] (\k,\k) rectangle ({\k+4},{\k+4});
 \foreach \k in {0,2,4,6}
   \fill[dyadictwo] (\k,\k) rectangle ({\k+2},{\k+2});
 \foreach \k in {0,...,7}
   \fill[dyadicthree] (\k,\k) rectangle ({\k+1},{\k+1});
 \foreach \k in {1,...,7} {
   \draw[white,line width=0.35pt] (\k,0) -- (\k,8);
   \draw[white,line width=0.35pt] (0,\k) -- (8,\k);
 }
 \draw[black!60,line width=0.4pt] (0,0) rectangle (8,8);
 \foreach \k in {0,...,7} {
   \node[font=\scriptsize] at ({\k+0.5},-0.45) {$\k$};
   \node[anchor=east,font=\scriptsize] at (-0.15,{\k+0.5}) {$\k$};
 }
 \node at (4,-1.15) {$s$};
 \node at (-1.05,4) {$t$};
 \node at (13.5,-0.8) {Example: $q=3$};
 \draw[black!55] (11.5,1) -- (13.5,0) -- (15.5,1);
 \foreach \x in {11.5,15.5}
   \draw[black!55] ({\x-1},2) -- (\x,1) -- ({\x+1},2);
 \foreach \x in {10.5,12.5,14.5,16.5}
   \draw[black!55] ({\x-0.5},3) -- (\x,2) -- ({\x+0.5},3);
 \fill[dyadiczero] (13.5,0) circle[radius=2pt];
 \foreach \x in {11.5,15.5}
   \fill[dyadicone] (\x,1) circle[radius=2pt];
 \foreach \x in {10.5,12.5,14.5,16.5}
   \fill[dyadictwo] (\x,2) circle[radius=2pt];
 \foreach \k in {0,...,7} {
   \fill[dyadicthree] ({10+\k},3) circle[radius=2pt];
   \node[font=\scriptsize] at ({10+\k},3.5) {$\k$};
 }
 \foreach \partIndex/\partColor/\partText in {
   0/dyadiczero/{ (root)},
   1/dyadicone/{},
   2/dyadictwo/{},
   3/dyadicthree/{ (diagonal)}} {
   \fill[\partColor] (10,{4.45+0.95*\partIndex})
      rectangle (10.6,{5.05+0.95*\partIndex});
   \node[anchor=west] at (10.8,{4.75+0.95*\partIndex})
      {Depth $\ell=\partIndex$\partText};
 }
\end{tikzpicture}

\end{center}

Let $M_\ell$ be the absolute mass of part $\ell$.
Since $\sum_{\ell=0}^q M_\ell=M$, some part has mass at least
$M/(q+1)=\Omega(M/\log T)$. It therefore suffices to construct,
for each $\ell$, random signs $a_t,b_s\in\{-1,1\}$ such that
\[
 \E\sum_{t,s}H_{ts}a_t b_s=M_\ell.
\]
This guarantees a realization whose objective
is at least the mass of the largest part.

\Needspace{8\baselineskip}
Fix $\ell<q$ and choose independent uniform signs
$\epsilon_v,\eta_v\in\{-1,1\}$ for all nodes $v$ at depth $\ell$.
For each leaf $t$ below $v$, set
\[
 (a_t,b_t)=
 \begin{cases}
  (\epsilon_v,\eta_v),&t\text{ lies below the left child of }v,\\
  (\eta_v,-\epsilon_v),&t\text{ lies below the right child of }v.
 \end{cases}
\]
If $H_{ts}$ belongs to part $\ell$, then $t$ and $s$ lie in opposite
child subtrees of the same node $v$. The product $a_t b_s$ is $+1$
when $t>s$ and $-1$ when $t<s$, so $H_{ts}a_t b_s=|H_{ts}|$.
All other products have expectation $0$: leaves in the same child
subtree use the independent signs $\epsilon_v,\eta_v$, and leaves
below different nodes use independent signs. Thus the expected
objective is $M_\ell$.
For the diagonal part, take independent Rademacher signs $\epsilon_t$
and set $a_t=\epsilon_t$, $b_t=-\epsilon_t$; the expected objective
is $-\sum_t H_{tt}=M_q$, completing the construction.
\end{proof}

Note that the discrete Hilbert transform matrix $K_T$ spreads its
$\Theta(T\log T)$ mass such that $M_\ell=O(T)$ for every $\ell$,
attaining the lower bound up to constants.

\paragraph{A lower bound for general test sets.}
Embedding the scalar construction gives a baseline lower bound that is
independent of dimension. In Section~\ref{sec:optimality-full-proof},
we use this bound to control an error term while proving the sharper
comparison with $\Gamma_{\mathcal A,\mathcal B}^{(T)}(K_T)$, which also
captures the dependence on the geometry of $\mathcal A$ and $\mathcal B$.
Let $\mathcal A,\mathcal B\subseteq\R^d$ be nonempty compact convex
sets with $0\in\mathcal A$ and $\mathcal B=-\mathcal B$, and put
\[
 r=\max_{a\in\mathcal A,\,b\in\mathcal B}\langle a,b\rangle.
\]
The scalar construction also gives, for every sign-compatible
$T\times T$ matrix $H$ of mass $M$ and $T\ge2$,
\begin{equation}\label{eq:optimality-vector-baseline}
 \Gamma_{\mathcal A,\mathcal B}^{(T)}(H)
 \ge\frac{r M}{2(1+\lceil\log_2 T\rceil)}
 \ge\frac{r M}{18\log T}.
\end{equation}
Indeed, choose $u\in\mathcal A$, $v\in\mathcal B$ with
$\langle u,v\rangle=r$. Replace the scalar signs $a_t,b_s$ by
$((1+a_t)/2)u$ and $b_sv$. These vectors are admissible, and
$\E b_s=0$ in every distribution above, so each expected pairing
is $r/2$ times its scalar value.

\subsection{Full Proof of Theorem~\ref{thm:universal-harmonic-comparison}}
\label{sec:optimality-full-proof}

Throughout this subsection, write $K=K_T$.

\paragraph{Proof idea.}

To prove Theorem~\ref{thm:universal-harmonic-comparison}, we want to
show that, for every $H\in\mathcal H_T$,
\[
 \max_{\substack{a_{0:T-1}\in\mathcal A^T\\b_{0:T-1}\in\mathcal B^T}}
   \sum_{t,s<T}K_{ts}\langle a_t,b_s\rangle
 \le C\max_{\substack{a'_{0:T-1}\in\mathcal A^T\\b'_{0:T-1}\in\mathcal B^T}}
   \sum_{t,s<T}H_{ts}\langle a'_t,b'_s\rangle.
\]
The general proof relies on a stronger structural property of harmonic weights.
Fix vectors $a_0,\ldots,a_{T-1}\in\mathcal A$ and
$b_0,\ldots,b_{T-1}\in\mathcal B$ that attain the maximum for $K$.
We want to construct random vectors $a'_0,\ldots,a'_{T-1}\in\mathcal A$ and
$b'_0,\ldots,b'_{T-1}\in\mathcal B$ such that
\[
 \E\!\left[\sum_{t,s}H_{ts}\langle a'_t,b'_s\rangle\right]
 \ge\frac1C\sum_{t,s}K_{ts}\langle a_t,b_s\rangle
 \qquad\text{for every }H\in\mathcal H_T.
\]
For each fixed $H$, some choice of these vectors attains at least
this expectation, proving the desired comparison.

Since this bound must hold for every $H\in\mathcal H_T$, we consider
which choice of $H$ makes the expectation smallest for a given
distribution over $a'_{0:T-1}$ and $b'_{0:T-1}$.
Recall that $\mathcal H_T$ requires
$H_{ts}\ge0$ for $t>s$ and $H_{ts}\le0$ for $t\le s$, and fixes the
total absolute mass. Subject to these constraints, the minimizing $H$
puts all its mass on a single entry $H_{ts}$ minimizing
$\pm\E\langle a'_t,b'_s\rangle$, with $+$ for $t>s$ and $-$ for
$t\le s$. To prevent any one pair of
indices from limiting the bound, we aim to make the expected dot
products approximately equal up to sign:
\[
 \E\langle a'_t,b'_s\rangle\approx
 \begin{cases}
  \alpha,&t>s,\\
  -\alpha,&t\le s,
 \end{cases}
 \qquad\alpha>0.
\]

We will construct $a'_{0:T-1}$ and $b'_{0:T-1}$ as random weighted sums
of $a_{0:T-1}$ and $b_{0:T-1}$, respectively.
How can their expected dot product depend only on the sign of $t-s-\tfrac12$,
and not its magnitude? The key is harmonic weights, whose effect is
easiest to see with functions of a continuous variable.

For example, let $\phi:\R\to\R$ be an odd function. Suppose we are
tasked with constructing a new function $\phi'$ as a weighted integral
of values of $\phi$, such that $\phi'(k)$ depends only on the sign of
$k$. If the improper integral $I=\int_0^\infty\phi(v)\,dv/v$ converges,
then we can naturally construct $\phi'$ using harmonic weights:
\begin{equation}\label{eq:optimality-scale-invariance}
 \phi'(k):=\int_0^\infty\frac{\phi(ku)}u\,du
 =\operatorname{sgn}(k)\int_0^\infty\frac{\phi(v)}v\,dv
 =\operatorname{sgn}(k)I\qquad(k\ne0).
\end{equation}
Here we use oddness to factor out $\operatorname{sgn}(k)$, then
substitute $v=|k|u$. The harmonic weight $1/u$ cancels the change
of scale, leaving only the sign of $k$.

To make use of this identity, we have to turn our discrete lists of
vectors $a_0,\ldots,a_{T-1}$ and $b_0,\ldots,b_{T-1}$ into functions
$f$ and $g$ of a continuous variable, respectively. A natural way to
do this is to use pulse functions. The function $f$ will take the value
$a_t$ on a short interval of width $w$ for each $t$ and be zero
outside these intervals, and likewise $g$ takes the value $b_s$ on a
short interval for each $s$.
An antisymmetrized convolution of $f$ and $g$
produces the odd function $\phi$, which encodes the relevant dot
products $\langle a_t,b_s\rangle$.
The pulse locations align with the Hilbert denominators, so the
harmonic weights roughly give
\[
 I=\int_0^\infty\frac{\phi(v)}v\,dv
 \approx\frac1{32T}\sum_{t,s<T}K_{ts}\langle a_t,b_s\rangle
 =\frac{\Gamma_{\mathcal A,\mathcal B}^{(T)}(K)}{32T}.
\]

We then give a random construction of the vectors $a'_t,b'_s$ that
mirrors the harmonic averaging used to define $\phi'$, in such a way
that, for all $t,s$,
\[
 \E\langle a'_t,b'_s\rangle
 \approx\phi'(t-s-\tfrac12)
 =\operatorname{sgn}(t-s-\tfrac12)I.
\]
This gives the desired relationship between the expected dot products
and $\Gamma_{\mathcal A,\mathcal B}^{(T)}(K)$, with signs matching
those of the entries of $H$.

The above is a rough picture motivating our construction of $a'_t,b'_s$
and the proof of Theorem~\ref{thm:universal-harmonic-comparison}.
For the actual proof, we must also deal with error terms originating
from truncation and approximation. The weights $dv/v$ have infinite
total mass, whereas the distribution over $a'_t,b'_s$ must have total
mass $1$. Our actual construction therefore aligns the expected dot
products with a normalized, truncated integral. The resulting expectation is
\[
 \E\langle a'_t,b'_s\rangle
 =\frac1{\log(256T^2)}\int_{1/(2T)}^{128T}
       \frac{\phi((t-s-\tfrac12)v)}v\,dv,
\]
introducing an error in recovering
$\operatorname{sgn}(t-s-\tfrac12)I$. Additionally, the finite pulse
width and periodic copies introduce errors in recovering the Hilbert
objective. The proof below demonstrates how we can control all of
these errors.\footnote{Related uses of periodic Hilbert-transform bounds
and discrete Hilbert matrices appear in
Bourgain~\cite[Section~3]{bourgain1983remarks}.}

\begin{proof}[Proof of Theorem~\ref{thm:universal-harmonic-comparison}]
Let $r=\max_{a\in\mathcal A,\,b\in\mathcal B}\langle a,b\rangle$.
Fix vectors
$a_0,\ldots,a_{T-1}\in\mathcal A$ and
$b_0,\ldots,b_{T-1}\in\mathcal B$ that attain the maximum for $K$:
\[
 \sum_{t,s<T}\frac{\langle a_t,b_s\rangle}{2(t-s)-1}
 =\Gamma_{\mathcal A,\mathcal B}^{(T)}(K).
\]
\paragraph{Define pulse functions using $a,b$.}
For a visual aid, Appendix~\ref{app:pulse-visualization} illustrates
the construction of $f$, $g$, $\chi$, and $\phi$ in a scalar example.
Set $w=1/2$, $\Lambda=4T$, and define $\Lambda$-periodic pulse functions
$f:\R\to\mathcal A$, $g:\R\to\mathcal B$ as follows.
For each $t,s\in\{0,\ldots,T-1\}$, set
\[
 \begin{aligned}
  f(x)&=a_t &&\text{when }|x-2t|<w/2,\\
  g(x)&=b_s &&\text{when }|x-(2s+1)|<w/2.
 \end{aligned}
\]
Set both functions to $0$ elsewhere in $[-w/2,\Lambda-w/2]$, then repeat
this pattern with period $\Lambda$. That is, $f$ takes the value $a_t$ on
a pulse centered at $2t$ with width $w$, and $g$ takes the value $b_s$
on a pulse centered at $2s+1$ with width $w$. Note that the second
half of each period is empty, leaving a gap between repetitions of
the pulse configuration.

Thus $f(x)\in\mathcal A$ and
$g(x)\in\mathcal B$ for all $x\in\R$, and
$|\langle f(x),g(y)\rangle|\le r$ for all $x,y\in\R$.
The signed distance between the centers of the $a_t$ pulse in $f$
and the $b_s$ pulse in $g$ is exactly $2(t-s)-1$, the denominator of $K_{ts}$. The small width
keeps row and column supports separated.

Define
\begin{equation}\label{eq:optimality-correlation}
 \chi(v)=\frac1{\Lambda}\int_0^\Lambda\langle f(x+v),g(x)\rangle\,dx,
 \qquad \phi(v)=\frac{\chi(v)-\chi(-v)}2.
\end{equation}
The convolution is a finite sum of pulse-overlap functions, so it
is Lipschitz even though the pulses themselves are piecewise
constant. Moreover, $\phi$ is odd and periodic, $|\phi|\le r$, and
$\phi(v)=0$ for $|v|\le1/2$. Its mean over a period is zero, so
$\Phi(v)=\int_0^v\phi(z)\,dz$ is periodic with $\|\Phi\|_\infty\le r \Lambda$.
These facts give convergence of
\begin{equation}\label{eq:optimality-periodic-integral}
 I=\int_0^\infty\frac{\phi(v)}v\,dv.
\end{equation}
To see why this integral recovers the original objective
$\Gamma_{\mathcal A,\mathcal B}^{(T)}(K)$, consider
the contribution of an $a_t$ pulse in $f$ and a $b_s$ pulse in $g$
to $\chi(v)$. Before adding its periodic copies, this is
\[
 \frac{\langle a_t,b_s\rangle}{\Lambda}
       \bigl(w-|v-(2t-(2s+1))|\bigr)_+,
\]
where $(a)_+=\max\{a,0\}$. This is a signed triangular bump centered
at $2t-(2s+1)$, the denominator of the shifted discrete Hilbert
transform matrix entry $K_{ts}$. Its area is
$w^2\langle a_t,b_s\rangle/\Lambda$. Across this narrow bump,
$1/v$ is close to $K_{ts}$. Taking the odd part, integrating,
and summing over pulse pairs therefore gives
\[
 \int_0^\infty\frac{\phi(v)}v\,dv
 \approx\frac{w^2}{2\Lambda}\sum_{t,s<T}
            \frac{\langle a_t,b_s\rangle}{2(t-s)-1}
 =\frac{\Gamma_{\mathcal A,\mathcal B}^{(T)}(K)}{32T}.
\]
The next lemma controls the two
approximations: the finite pulse width and the periodic copies.

\begin{lemma}[The Hilbert objective survives the pulse embedding]
\label{lem:optimality-pulses}
The pulses just constructed satisfy
\begin{equation}\label{eq:optimality-pulse-integral-bound}
 I\ge\frac{\Gamma_{\mathcal A,\mathcal B}^{(T)}(K)}{32T}-\frac{r}{16}.
\end{equation}
\end{lemma}

Appendix~\ref{app:harmonic-pulses} proves the estimate by periodizing
the reciprocal kernel. The coefficient $1/(32T)$ is $w^2/(2\Lambda)$.
The two errors, from periodization and from replacing pulses by
their centers, are at most $r/64$ and $3r/64$, respectively.

\paragraph{Constructing $a'$ and $b'$.}
Our construction begins by sampling a scale variable $u>0$ with density
proportional to $1/u$ over a finite interval specified below. We first describe how to
construct $a'$ and $b'$ conditional on a fixed value of $u$.
Draw $x$ uniformly modulo $\Lambda$ and choose each of
the following two families with probability $1/2$:
\begin{equation}\label{eq:optimality-randomized-families}
 \begin{array}{lll}
  a'_t=f(x+tu),&b'_s=g(x+(s+\tfrac12)u),&\text{or}\\
  a'_t=f(x-tu),&b'_s=-g(x-(s+\tfrac12)u),&0\le t,s<T.
 \end{array}
\end{equation}
Every realization remains in $\mathcal A$ and $\mathcal B$: the central
symmetry of $\mathcal B$ allows us to negate the $b'_s$ values in the
second family.
For $k=t-s-1/2$,
\[
 \begin{aligned}
 \E\bigl[\langle a'_t,b'_s\rangle\mid u,\text{first family}\bigr]
 &=\frac1{\Lambda}\int_0^\Lambda
    \langle f(x+tu),g(x+(s+\tfrac12)u)\rangle\,dx\\
 &=\frac1{\Lambda}\int_0^\Lambda\langle f(y+ku),g(y)\rangle\,dy
 =\chi(ku),
 \end{aligned}
\]
using the substitution $y=x+(s+\tfrac12)u$.
\[
 \begin{aligned}
 \E\bigl[\langle a'_t,b'_s\rangle\mid u,\text{second family}\bigr]
 &=-\frac1{\Lambda}\int_0^\Lambda
    \langle f(x-tu),g(x-(s+\tfrac12)u)\rangle\,dx\\
 &=-\frac1{\Lambda}\int_0^\Lambda\langle f(y-ku),g(y)\rangle\,dy
 =-\chi(-ku),
 \end{aligned}
\]
using the substitution $y=x-(s+\tfrac12)u$.
Because each family is selected with probability $1/2$, we have
\[
 \E\bigl[\langle a'_t,b'_s\rangle\mid u\bigr]
 =\frac12\bigl(\chi(ku)-\chi(-ku)\bigr)=\phi(ku).
\]
The half-step shift makes $k>0$ for $t>s$ and $k<0$ for $t\le s$,
with $1/2\le|k|<T$.

\paragraph{A finite logarithmic average nearly removes the distance.}
Independently draw $u$ from $[1/(2T),128T]$ with density
$1/[u\log(256T^2)]$. Averaging over $u$ gives
\begin{equation}\label{eq:optimality-finite-scale-margin}
 \begin{aligned}
 \E\langle a'_t,b'_s\rangle
 &=\frac1{\log(256T^2)}\int_{1/(2T)}^{128T}\frac{\phi(ku)}u\,du\\
 &=\frac{\operatorname{sgn}(k)}{\log(256T^2)}
   \int_{|k|/(2T)}^{128T|k|}\frac{\phi(v)}v\,dv\\
 &=\frac{\operatorname{sgn}(k)}{\log(256T^2)}
   \left(I-\int_0^{|k|/(2T)}\frac{\phi(v)}v\,dv
           -\int_{128T|k|}^\infty\frac{\phi(v)}v\,dv\right),
 \end{aligned}
\end{equation}
using oddness and the substitution $v=|k|u$.
The first subtracted integral is $0$, since $|k|/(2T)\le1/2$
and $\phi(v)=0$ on $[0,1/2]$. To bound the upper tail, recall that
$\Phi(v)=\int_0^v\phi(z)\,dz$ satisfies $\|\Phi\|_\infty\le r\Lambda$.
For any $z>0$, integration by parts gives
\[
 \left|\int_z^\infty\frac{\phi(v)}v\,dv\right|
 =\left|-\frac{\Phi(z)}z+\int_z^\infty\frac{\Phi(v)}{v^2}\,dv\right|
 \le\frac{2\|\Phi\|_\infty}z\le\frac{2r\Lambda}z,
\]
where the boundary term at infinity vanishes because $\Phi$ is bounded.
Taking $z=128T|k|\ge64T=16\Lambda$ bounds the upper tail by $r/8$.
Thus every entry satisfies
\[
 \operatorname{sgn}(k)\E\langle a'_t,b'_s\rangle
 \ge\frac{I-r/8}{\log(256T^2)}
 \ge\frac{\Gamma_{\mathcal A,\mathcal B}^{(T)}(K)/(32T)-3r/16}
              {\log(256T^2)}.
\]
Fix any $H\in\mathcal H_T$. Multiplying the entrywise bounds by
$|H_{ts}|$ and summing gives
\[
 \begin{aligned}
 \Gamma_{\mathcal A,\mathcal B}^{(T)}(H)
 &\ge\E\!\left[\sum_{t,s<T}H_{ts}\langle a'_t,b'_s\rangle\right]\\
 &\ge\frac{\|H\|_1}{\log(256T^2)}
       \left(\frac{\Gamma_{\mathcal A,\mathcal B}^{(T)}(K)}{32T}
             -\frac{3r}{16}\right).
 \end{aligned}
\]
Since $\|H\|_1=T(2h_T-1)$, rearranging yields
\begin{equation}\label{eq:optimality-pulse-bound}
 \Gamma_{\mathcal A,\mathcal B}^{(T)}(K)
 \le\frac{32\log(256T^2)}{2h_T-1}
      \Gamma_{\mathcal A,\mathcal B}^{(T)}(H)+6rT.
\end{equation}

For $T\ge2$, we have $\log(256T^2)\le10\log T$ and
$2h_T-1\ge\log T$, while
\eqref{eq:optimality-vector-baseline}
gives
\[
 \Gamma_{\mathcal A,\mathcal B}^{(T)}(H)
 \ge\frac{rT(2h_T-1)}{18\log T}\ge\frac{rT}{18}.
\]
Thus \eqref{eq:optimality-pulse-bound} yields
\begin{equation}\label{eq:fixed-horizon-core-bound}
 \Gamma_{\mathcal A,\mathcal B}^{(T)}(K)
 \le320\,\Gamma_{\mathcal A,\mathcal B}^{(T)}(H)+6rT
 \le512\,\Gamma_{\mathcal A,\mathcal B}^{(T)}(H).
\end{equation}
This proves the theorem with $C=512$.
\end{proof}

\section{High Dimensional Approachability}
\label{sec:approachability}

Blackwell's approachability \cite{blackwell1956analog} describes a general
family of online learning problems,
capturing classic problems such as regret minimization and swap regret
minimization. Approachability problems involve a repeated interaction
between a learner and an adversary. On each of $T$ rounds, the learner
selects an action $x_t$ from a set $\mathcal X\subset\R^m$ and an
adversary responds with an outcome $y_t$ from a set $\cY\subset\R^d$.
These selections are mapped to a vector-valued payoff
$\kappa_t$. The goal of the learner in an approachability
problem is to steer the average payoff
$\bar\kappa_T=(1/T)\sum_{t=0}^{T-1}\kappa_t$ towards a target set
$\mathcal S$ as $T$ increases. Blackwell's approachability theorem is
a foundational result in online learning, demonstrating that the learner
can succeed at this task given that the target set satisfies certain
preconditions.

We start by describing a natural way to construct approachability problems.
This construction captures the worst-case instances of approachability,
after an affine embedding when needed. For a comparison of the classical
formulation, the max--min--max formulation, and this response formulation,
see Appendix~\ref{app:approachability-foundations}.

\subsection{The Response Formulation}
\label{sec:app-response}

An approachability problem is specified by a vector-valued payoff map,
a target set, and a norm for measuring distance to the target.
Let $\cY\subset\R^d$ and $\mathcal X\subset\R^m$ be nonempty compact
convex sets of outcomes and actions. On round $t$, the learner chooses
$x_t\in\mathcal X$, and then the adversary responds with $y_t\in\cY$,
selected after observing $x_t$. Under the response formulation, the payoff
map will simply be the tensor product. This captures all pairwise products
of the outcome and action coordinates:
\begin{equation}
 \kappa(x,y)=y\otimes x=yx^{\mathsf T}.
 \label{eq:app-payoff-map}
\end{equation}
Every bilinear vector-valued payoff map is a linear image of this tensor.
This is why it provides a general starting point for the construction
(Appendix~\ref{app:response-universality}). Let $\kappa_t=y_t\otimes x_t$
denote the payoff on round $t$, and let the time-averaged payoff be
$\bar\kappa_T=(1/T)\sum_{t<T}\kappa_t$.
The norm unit ball is the signed convex hull of the feasible payoffs:
\[
 \mathcal K=\conv\{\pm(y\otimes x):y\in\cY,\ x\in\mathcal X\}.
\]

In the response formulation of approachability, we specify a response
map $b:\cY\to\mathcal X$ that determines the target set.
Think of $b(y)$ as the response action to take if the outcome $y$
were known in advance. The target induced by $b$ is the convex hull
of the tensors $y\otimes b(y)$:
\begin{equation}
 \mathcal S=\conv\{y\otimes b(y):y\in\cY\}.
 \label{eq:app-response-graph}
\end{equation}
Thus $\mathcal S$ is the smallest convex set containing
$y\otimes b(y)$ for every $y\in\cY$. The learner must choose $x_t$ before
seeing $y_t$, so it uses past observations to steer $\bar\kappa_T$
towards $\mathcal S$ as $T$ increases.

To quantify this goal, we measure distance from $\bar\kappa_T$ to
$\mathcal S$ using the norm $\|\cdot\|_{\mathcal K}$ with unit ball
$\mathcal K$. Throughout this section, for a set $\mathcal A$ that is
not centrally symmetric, we define
\[
 \|v\|_{\mathcal A}:=\min\{r\ge0:v\in r\,\conv(\mathcal A\cup(-\mathcal A))\},
 \qquad v\in\operatorname{span}(\mathcal A).
\]
The tensor norm is taken on the span of the feasible payoffs; its unit
ball $\mathcal K$ is symmetric regardless of whether the feasible sets are.
We define the approachability loss as
\begin{equation}
 \APPLOSS(x_{0:T-1},y_{0:T-1})
 :=\dist_{\mathcal K}(\bar\kappa_T,\mathcal S)
 =\min_{s\in\mathcal S}\|\bar\kappa_T-s\|_{\mathcal K}.
 \label{eq:app-loss}
\end{equation}
This problem satisfies the preconditions of Blackwell's approachability
theorem: the action and outcome sets are compact and convex, the payoff
is bilinear, the target is convex, and the following
response and boundedness properties hold. The first property,
\emph{response satisfiability}, means that each outcome admits an
action whose payoff lies in the target:
\[
 \forall y\in\cY\quad\exists x\in\mathcal X
 \quad\text{such that}\quad y\otimes x\in\mathcal S.
\]
Indeed, $x=b(y)$ gives a valid response by the definition of $\mathcal S$.

\emph{Boundedness.} For every $x\in\mathcal X$ and $y\in\cY$,
$\dist_{\mathcal K}(\kappa(x,y),\mathcal S)\le2$.
For any $s\in\mathcal S$, the triangle inequality gives
\[
 \dist_{\mathcal K}(y\otimes x,\mathcal S)
 \le\|y\otimes x-s\|_{\mathcal K}
 \le\|y\otimes x\|_{\mathcal K}+\|s\|_{\mathcal K}
 \le2,
\]
because both $y\otimes x$ and $s$ belong to $\mathcal K$.

\Needspace{10\baselineskip}
Therefore Blackwell's approachability theorem \cite{blackwell1956analog}
gives the following guarantee.
\begin{theorem}[Blackwell approachability in the response formulation]
\label{thm:app-blackwell}
For the action and outcome sets and response map above, there exists
an online algorithm choosing each $x_t\in\mathcal X$ from past
observations such that, for every $\varepsilon>0$, there is a finite
$T_\varepsilon$ for which
\[
 \APPLOSS(x_{0:T-1},y_{0:T-1})\le\varepsilon
 \qquad\text{for every }T\ge T_\varepsilon
\]
against every adaptive sequence of outcomes $y_t\in\cY$, including
outcomes chosen after observing $x_t$.
\end{theorem}

\paragraph{Example: The Experts Problem and Swap Regret.}
As an example of the response formulation of Blackwell approachability,
we show how it captures swap regret minimization in the experts problem.
Here $\mathcal X=\Delta_d$ is the set of distributions over $d$ experts,
and $\cY=[-1,1]^d$ is the set of loss vectors. On each round, the learner
chooses a distribution $x_t\in\mathcal X$, and the adversary responds
with a loss vector $y_t\in\cY$ after observing $x_t$. The learner
incurs loss $\langle x_t,y_t\rangle$.

The tensor payoff body $\mathcal K$ consists of matrices for which the
sum of the infinity norms of the columns is at most one:
\[
 \mathcal K=\left\{\kappa\in\R^{d\times d}:\sum_{i=1}^d\|\kappa_{:,i}\|_\infty\le1\right\},
 \qquad
 \|\kappa\|_{\mathcal K}=\sum_{i=1}^d\|\kappa_{:,i}\|_\infty.
\]
If the loss vector $y$ were known in advance, the learner would put
all its weight on experts with minimal loss. This gives the best-response
set $\operatorname{BR}(y)=\arg\min_{x\in\Delta_d}\langle x,y\rangle$.
Fix a selection $b(y)\in\operatorname{BR}(y)$ for the algorithm.
The target containing all best responses can be larger than the target
of this fixed selection.
The induced target thus consists of matrices in $\mathcal K$ in which
the diagonal entry is minimal within each column:
\[
 \begin{aligned}
 \mathcal S_{\mathrm{swap}}
 &=\conv\{y\otimes x:y\in\cY,\ x\in\operatorname{BR}(y)\}\\
 &=\{s\in\mathcal K:s_{ii}\le s_{ji}\text{ for all }i,j\}.
 \end{aligned}
\]
Distance to the target of all best responses is
\begin{equation}
 \begin{aligned}
 \dist_{\mathcal K}(\bar\kappa_T,\mathcal S_{\mathrm{swap}})
 &=\frac12\sum_{i=1}^d\left((\bar\kappa_T)_{ii}-\min_j(\bar\kappa_T)_{ji}\right)\\
 &=\frac{1}{2T}\max_{\phi:[d]\to[d]}
   \sum_{t<T}\sum_{i=1}^d x_{t,i}\bigl(y_{t,i}-y_{t,\phi(i)}\bigr).
 \end{aligned}
 \label{eq:app-experts-swap-loss}
\end{equation}
Indeed, making a column's diagonal entry minimal costs half its gap
in $\ell_\infty$ distance and can be done without increasing the column
norm. Thus this distance is exactly half the usual average swap
regret and is at most $\APPLOSS$ for any fixed best-response selection.
External regret restricts the maximum to constant swaps
$\phi(i)\equiv j$, so swap regret is the stronger requirement.

For general action and outcome sets, the convex hull of all
best-response tensors is the zero target for \emph{profile swap distance}
\cite{arunachaleswaran2025profile}. We measure its distance in the
$\mathcal K$-norm; the cited work uses Euclidean distance.
A fixed selection $b$ gives a contained target, so approaching it also
controls profile swap distance, linear swap regret and external regret
(Appendix~\ref{app:profile-swap}). General approachability allows response
maps beyond linear optimization.

\subsection{Approachability via Calibration}
\label{sec:app-two-routes}

A classical algorithm for approachability is to produce a calibrated
forecast of the upcoming outcome $y_t$ and respond to
it~\cite[Section~4.4]{perchet2014approachability}. Given a forecast distribution
$\mu_t$, play the convex mixture of its responses:
\begin{equation}
 x_t=\sum_{p\in\cY}\mu_t(p)b(p).
 \label{eq:response-mixture}
\end{equation}
This gives the following guarantee.

\begin{lemma}[Calibration implies approachability]
\label{lem:calibration-approachability}
Let $\mu_{0:T-1}$ be forecast distributions on $\cY$ and
$y_{0:T-1}$ any sequence of outcomes in $\cY$.
If the actions $x_t$ are chosen according to \eqref{eq:response-mixture}, then
\begin{equation}
 \APPLOSS(x_{0:T-1},y_{0:T-1})
 \le\frac1T\CAL_{\cY,T}(\mu_{0:T-1},y_{0:T-1}),
 \label{eq:graph-calibration-transfer}
\end{equation}
where calibration error is measured in the norm $\|\cdot\|_{\cY}$.
\end{lemma}
\begin{proof}
The rank-one identity from Appendix~\ref{app:tensor-duality} gives
$\|y\otimes x\|_{\mathcal K}=\|y\|_{\cY}\|x\|_{\mathcal X}$.
We upper bound the distance from $\bar\kappa_T$ to $\mathcal S$ by
its distance to the specific target point
\[
 s_T=\frac1T\sum_{t<T}\sum_{p\in\cY}
             \mu_t(p)\,p\otimes b(p)\in\mathcal S.
\]
Since $b(p)\in\mathcal X$, we have $\|b(p)\|_{\mathcal X}\le1$.
Regrouping by forecast and applying the triangle inequality and the
tensor-norm bound, we obtain
\[
 \begin{aligned}
 \dist_{\mathcal K}(\bar\kappa_T,\mathcal S)
 &\le\|\bar\kappa_T-s_T\|_{\mathcal K}\\
 &=\frac1T\left\|\sum_{p\in\cY}
       \left(\sum_{t<T}\mu_t(p)(y_t-p)\right)\otimes b(p)
       \right\|_{\mathcal K}\\
 &\le\frac1T\sum_{p\in\cY}
       \left\|\left(\sum_{t<T}\mu_t(p)(y_t-p)\right)
       \otimes b(p)\right\|_{\mathcal K}\\
 &\le\frac1T\sum_{p\in\cY}
       \left\|\sum_{t<T}\mu_t(p)(y_t-p)\right\|_{\cY}
       \|b(p)\|_{\mathcal X}\\
 &\le\frac1T\sum_{p\in\cY}
       \left\|\sum_{t<T}\mu_t(p)(y_t-p)\right\|_{\cY}
       \\
 &=\frac1T\CAL_{\cY,T}(\mu_{0:T-1},y_{0:T-1}).
 \end{aligned}
\]
\end{proof}

\subsection{Harmonic Approachability}
\label{sec:app-harmonic}
\label{sec:app-proofs}
\label{sec:app-direct}
\label{sec:app-joint-geometry}
\label{sec:app-matrix-discrepancy}

Combining our harmonic calibration algorithm with the response
mixture in \eqref{eq:response-mixture} gives \emph{harmonic approachability},
our algorithm for high-dimensional approachability.

Choose $p_{\mathrm{init}}\in\cY$ and play $x_0=b(p_{\mathrm{init}})$.
As in Algorithm~\ref{alg:harmonic}, after observing $y_t$ we store
a forecast $p_t$ and its response $b_t=b(p_t)$:
\begin{equation}
 \begin{gathered}
 h_t=\sum_{k=1}^t\frac1{2k-1},\qquad
 p_t=\sum_{s=0}^t\frac{y_s}{h_{t+1}(2(t-s)+1)},\qquad b_t=b(p_t),\\
 x_t=\sum_{s=0}^{t-1}\frac{b_s}{h_t(2(t-s)-1)}\qquad(t\ge1).
 \end{gathered}
 \label{eq:harmonic-response-algorithm}
\end{equation}
Each action is a convex mixture of stored responses, computed before
the next outcome is selected. Only one new response computation is
needed per round.

Our main calibration theorem (Theorem~\ref{thm:general-geometric-rate}),
together with Lemma~\ref{lem:calibration-approachability}, already gives
a black-box guarantee for this algorithm, with greatly improved dimension
dependence over previously known bounds for certain approachability
instances (see Appendix~\ref{sec:app-orthants}).

However, we obtain even stronger approachability guarantees by adapting
our calibration proof and generalizing the geometric parameter $\gamma$
to include both outcome and action geometry. For example, when
$\mathcal X$ and $\cY$ are cubes, this gives dimension-independent
approachability rates. We develop this generalized parameter and
guarantee next.

Let $Y=[y_0\ \cdots\ y_{T-1}]$ and $X=[x_0\ \cdots\ x_{T-1}]$
be the matrices whose columns are the outcome and action vectors,
respectively. For a matrix $K\in\R^{T\times T}$,
\begin{equation}
 YKX^{\mathsf T}=\sum_{t,s<T}K_{ts}\,y_t\otimes x_s.
 \label{eq:app-matrix-pairing}
\end{equation}
We generalize the discrepancy and harmonic coefficient by taking the
tensor norm $\|\cdot\|_{\mathcal K}$ defined in Section~\ref{sec:app-response}:
\begin{equation}
 \begin{aligned}
 \Psi_{\cY,\mathcal X}^{(T)}(K)
 &=\sup_{\substack{Y_{:,t}\in\cY\\X_{:,s}\in\mathcal X}}
                           \|YKX^{\mathsf T}\|_{\mathcal K},\\
 \psi(\cY,\mathcal X)
 &=\sup_{T\ge1}\frac1T
                \Psi_{\cY,\mathcal X}^{(T)}(K_T).
 \end{aligned}
 \label{eq:app-discrepancy}
\end{equation}
Here $K_T$ is the same harmonic matrix used for calibration.
The columns of $Y$ and $X$ vary independently, so this parameter
depends only on the outcome and action sets, not on the response map.
We prove the following main guarantee for harmonic approachability.

\begin{theorem}[Harmonic approachability]
\label{thm:app-direct}
Let $\cY\subset\R^d$ and $\mathcal X\subset\R^m$ be
nonempty compact convex feasible sets, and fix any map
$b:\cY\to\mathcal X$. Run \eqref{eq:harmonic-response-algorithm}.
For every adaptive outcome transcript and $T\ge2$,
\begin{equation}
 \APPLOSS(x_{0:T-1},y_{0:T-1})
 =O\!\left(\frac{\psi(\cY,\mathcal X)}{\log T}\right).
 \label{eq:app-direct-graph}
\end{equation}
\end{theorem}

The parameter $\psi$ generalizes the calibration parameter $\gamma$:
it agrees with $\gamma$ when $\cY$ and $\mathcal X$ form a polar pair,
while capturing more of their joint geometry for non-polar pairs.
\begin{lemma}[Polar pairs recover calibration]
\label{lem:app-polar-pair}
If $\cY$ is a symmetric norm ball and $\mathcal X=\cY^\circ$, then
for every $T\ge1$ and $K\in\R^{T\times T}$,
\[
 \Psi_{\cY,\cY^\circ}^{(T)}(K)=\Gamma_{\cY,\cY^\circ}^{(T)}(K),
 \qquad
 \psi(\cY,\cY^\circ)=\gamma(\cY,\cY).
\]
\end{lemma}

For other pairs of outcome and action sets, the joint geometry can
give stronger guarantees than the calibration reduction. In particular,
we obtain the following dimension-independent bounds.
\begin{lemma}[Euclidean and cube approachability]
\label{prop:app-tensor-geometry}
\label{cor:cube-cube}
If either $\cY$ or $\mathcal X$ is a Euclidean unit ball, or if
$\cY=B_\infty^d$ and $\mathcal X=B_\infty^m$, then
\[
 \psi(\cY,\mathcal X)=O(1).
\]
Therefore, in these settings, harmonic approachability satisfies
$\APPLOSS(x_{0:T-1},y_{0:T-1})=O(1/\log T)$ for any response map.
In particular, error $\varepsilon$ is achieved within
\begin{equation}
 T=\exp(O(1/\varepsilon))
 \label{eq:app-cube-horizon}
\end{equation}
rounds, independently of $d,m$.
\end{lemma}

For cube outcome and action sets, the black-box analysis treats our
algorithm as responding to calibrated forecasts. Applying
Lemma~\ref{lem:calibration-approachability} and our calibration guarantee gives
\[
 \APPLOSS(x_{0:T-1},y_{0:T-1})
 \le\frac{\CAL_{\cY,T}}{T}
 =O\!\left(\frac{\log d}{\log T}\right).
\]
This gives a sufficient horizon of $T=d^{O(1/\varepsilon)}$ rounds
for error $\varepsilon$, polynomial in the outcome dimension at fixed
accuracy. The direct analysis accounts for both outcome and action
geometry and removes the $\log d$ factor, giving $O(1/\log T)$ loss.
Thus the same algorithm achieves error $\varepsilon$ in
$T=\exp(O(1/\varepsilon))$ rounds, independently of both dimensions.

\begin{proof}[Proof of Theorem~\ref{thm:app-direct}]
We first bound approachability loss for any averaging and replay
matrices $P,Q$ satisfying \eqref{eq:matrix-decomposition}, as in
Theorem~\ref{thm:matrix-calibration}. Write $H=Q-P^{\mathsf T}$.
Let $Y=[y_0\ \cdots\ y_{T-1}]$ and $B=[b_0\ \cdots\ b_{T-1}]$
be the matrices whose columns are the outcome vectors and stored
response vectors, respectively. Thus $p_s=\sum_tP_{st}y_t$ and
$x_t=\sum_sQ_{ts}b_s$ for $t\ge1$, with $b_s=b(p_s)$ and
$x_0=b(p_{\mathrm{init}})$. Define
\[
 s_T=\frac1T\sum_{s<T}p_sb_s^{\mathsf T}
     =\frac1T YP^{\mathsf T}B^{\mathsf T}\in \mathcal S.
\]
We have the identity
\begin{equation}
 \bar\kappa_T-s_T=\frac1T\bigl[y_0x_0^{\mathsf T}
                           +YHB^{\mathsf T}\bigr].
 \label{eq:app-certificate}
\end{equation}
Since $\|y_0x_0^{\mathsf T}\|_{\mathcal K}\le1$, the triangle
inequality and the definition of $\Psi$ give
\[
 \begin{aligned}
 \APPLOSS(x_{0:T-1},y_{0:T-1})
 &\le\|\bar\kappa_T-s_T\|_{\mathcal K}\\
 &\le\frac{1+\Psi_{\cY,\mathcal X}^{(T)}(H)}{T}
 =O\!\left(\frac{\Psi_{\cY,\mathcal X}^{(T)}(H)}{T}\right).
 \end{aligned}
\]
For the last step, $\sum_{t,s}H_{ts}=-1$, so taking constant columns
in $Y,B$ gives $\Psi_{\cY,\mathcal X}^{(T)}(H)\ge1$ whenever the
payoff space is nonzero. The zero payoff case is immediate.

For harmonic weights, $H=\widehat K_T$. The normalization estimate
from Proposition~\ref{prop:harmonic-matrix} applies to $\Psi$ as well:
each tensor $y_t\otimes b_s$ has norm at most one, so the total
normalization correction $O(T/h_T)$ gives
\[
 \Psi_{\cY,\mathcal X}^{(T)}(\widehat K_T)
 \le\frac{\Psi_{\cY,\mathcal X}^{(T)}(K_T)}{h_T}
       +O\!\left(\frac{T}{h_T}\right).
\]
Combining this with the general-weight bound, $h_T=\Theta(\log T)$,
and the definition of $\psi$, we obtain
\[
 \APPLOSS(x_{0:T-1},y_{0:T-1})
 =O\!\left(\frac{\Psi_{\cY,\mathcal X}^{(T)}(\widehat K_T)}{T}\right)
 =O\!\left(\frac{\psi(\cY,\mathcal X)}{\log T}\right).
\]
Here the additive normalization term is absorbed because the $T=1$
term in the definition of $\psi$ gives $\psi\ge1$.
\end{proof}

The same bound holds for any convex target containing
$y\otimes b(y)$ for every $y\in\cY$, since it contains $\mathcal S$.

\begin{proof}[Proof of Lemma~\ref{lem:app-polar-pair}]
Fix $Y,K$ and let $v_s=\sum_{t<T}K_{ts}y_t$ be the columns of $YK$.
The triangle inequality and the tensor-norm bound give
\[
 \|YKX^{\mathsf T}\|_{\mathcal K}
 \le\sum_{s<T}\|v_s\|_{\cY}\|x_s\|_{\cY^\circ}
 \le\sum_{s<T}\|v_s\|_{\cY}.
\]
Conversely, choose each $x_s\in\cY^\circ$ so that
$\langle v_s,x_s\rangle=\|v_s\|_{\cY}$.
Since $|\operatorname{tr}(yx^{\mathsf T})|=|\langle y,x\rangle|\le1$
for every generating tensor of $\mathcal K$, we have
$|\operatorname{tr}(C)|\le\|C\|_{\mathcal K}$. Thus
\[
 \|YKX^{\mathsf T}\|_{\mathcal K}
 \ge\operatorname{tr}(YKX^{\mathsf T})
 =\sum_{s<T}\langle v_s,x_s\rangle
 =\sum_{s<T}\|v_s\|_{\cY}.
\]
Maximizing over $X$ and then $Y$ proves the discrepancy identity.
Taking $K=K_T$, dividing by $T$, and taking the supremum over $T$
gives the identity for the geometric parameters.
\end{proof}

\begin{proof}[Proof of Lemma~\ref{prop:app-tensor-geometry}]
By tensor-norm duality \eqref{eq:app-projective-norm}, it suffices to
bound $\langle W,YK_TX^{\mathsf T}\rangle$ for every matrix $W$
satisfying $\sup_{y\in\cY,x\in\mathcal X}|y^{\mathsf T}Wx|\le1$.
We use the operator bound $\|K_T\|_{2\to2}=O(1)$ from
Lemma~\ref{lem:harmonic-finite-moments}.

If $\cY=B_2^d$, normalization gives $\|Wx_s\|_2\le1$ for every $s$.
Applying the operator bound in each coordinate and Cauchy--Schwarz gives
\[
 \left|\sum_{t,s<T}(K_T)_{ts}\langle y_t,Wx_s\rangle\right|
 \le\|K_T\|_{2\to2}
      \left(\sum_{t<T}\|y_t\|_2^2\right)^{1/2}
      \left(\sum_{s<T}\|Wx_s\|_2^2\right)^{1/2}
 =O(T).
\]
The case $\mathcal X=B_2^m$ follows by transposing the pairing.

For cubes, we use Grothendieck's inequality: for any real matrix $W$
and vectors $u_i,v_j$ of Euclidean norm at most one,
\[
 \left|\sum_{i,j}W_{ij}\langle u_i,v_j\rangle\right|
 \le O(1)\sup_{y\in B_\infty^d,x\in B_\infty^m}|y^{\mathsf T}Wx|.
\]
Thus replacing bounded scalar coordinates with unit vectors increases
the maximum by at most a universal constant. Here the normalization of
$W$ bounds the right-hand side by $O(1)$. Define vectors from the rows
of $Y,X$ by
\[
 u_i=Y_{i,:}^{\mathsf T}/\sqrt T,\qquad
 v_j=K_TX_{j,:}^{\mathsf T}/(\|K_T\|_{2\to2}\sqrt T).
\]
Both have Euclidean norm at most one, so the inequality gives
\[
 \begin{aligned}
 |\langle W,YK_TX^{\mathsf T}\rangle|
 &=T\|K_T\|_{2\to2}
       \left|\sum_{i,j}W_{ij}\langle u_i,v_j\rangle\right|\\
 &\le O(1)\,T\|K_T\|_{2\to2}=O(T).
 \end{aligned}
\]
Taking the supremum over $W,Y,X$, dividing by $T$, and then taking
the supremum over $T$ proves both bounds on $\psi$.
The approachability guarantee follows from Theorem~\ref{thm:app-direct};
solving $O(1/\log T)\le\varepsilon$ gives the stated horizon.
\end{proof}

\bibliographystyle{alpha}
\bibliography{main}
\appendix
\section{Sampled Forecasts}
\label{app:sampling}

Throughout the main text, we have considered the problem of calibrating
distributions over forecasts: on each round $t$, the forecaster announces a
distribution $\mu_t$ over forecast points, and each point $p$ receives weight
$\mu_t(p)$ in the calibration error,
\[
 \CAL_{\cL,T}
 =\sum_p\left\|\sum_{t<T}\mu_t(p)(y_t-p)\right\|_{\cL}.
\]
In some applications, however, the forecaster must report a single forecast
$\widehat p_t\sim\mu_t$ on each round. We call this setting \emph{sampled
calibration} and measure its error using the forecasts actually reported:
\begin{equation}
 \CAL_{\cL,T}^{\mathrm{sample}}
 =\sum_p\left\|\sum_{t<T}
       \mathbf 1\{\widehat p_t=p\}(y_t-p)\right\|_{\cL}.
 \label{eq:sampled-objective}
\end{equation}
Here the sum runs over distinct realized forecast values, so repeated
occurrences of the same point contribute to the same group.

In this setting, the adversary observes the current sampled forecast only
after choosing the current outcome $y_t$. It may use all previous forecasts
and outcomes when making this choice. The sampling distribution $\mu_t$
is internal to the forecaster, although the adversary may deduce it when
it is determined by the algorithm and the observed history. More precisely,
conditional on the history, the sampling distribution, and the current
outcome, $\widehat p_t$ is a fresh draw from $\mu_t$.

We show that a modification of the harmonic algorithm achieves the same
headline calibration rates in expectation in this setting. The
modification holds each forecast distribution fixed for a block of rounds,
draws a fresh forecast on each round, and updates the algorithm using the
block's average outcome. Repetition introduces an additional cost, but
preserves the $d^{O(1/\varepsilon)}$ horizons for multi-event and multi-class
forecasting and the $\exp(O(1/\varepsilon))$ horizon for Euclidean
forecasting. Appendix~\ref{app:implementation} gives the computational
implementation.

We retain the main text's time convention: actual rounds are indexed by
$t=0,\ldots,T-1$, and an abbreviated time sum such as $\sum_{t<T}$ means
$\sum_{t=0}^{T-1}$. The notation $\widehat p_t$ distinguishes the sampled
forecast from the stored points $p_s$ in Algorithm~\ref{alg:harmonic}.

\paragraph{Why an additional argument is necessary.}
There is no support-free transfer from distributional to sampled
calibration.  For example, fix $T\ge4$ distinct points within $1/T$ of
$1/2$ in $[0,1]$, and use their uniform distribution on every round.
Against independent fair binary outcomes, the expected distributional
error is at most $\sqrt T/2+1$: its pathwise value
is at most $|\sum_{t=0}^{T-1}y_t-T/2|+1$.  Independent sampling from these distributions
creates an expected $T(1-1/T)^{T-1}$ singleton buckets.  Each singleton
contributes at least $1/2-1/T$ to sampled calibration, so the expected
sampled error is linear in $T$.  Thus the total number of forecast labels,
and not merely the support size on an individual round, matters in a
sampling proof.

\subsection{A Vector Martingale Bound for Predictable Labels}

As in the main text, $\|\cdot\|_2$ denotes the standard Euclidean norm
on $\R^d$.  Fix a norm-comparison constant $c_{\mathrm E}>0$ and define
\begin{equation}
 \|v\|_{\cL}\le c_{\mathrm E}\|v\|_2\quad(v\in\R^d),\qquad
 \delta _2=\sup_{p,y\in\cY}\|y-p\|_2,\qquad C_{\mathrm{samp}}=c_{\mathrm E}\delta_2.
 \label{eq:sampling-geometry}
\end{equation}
The constant $c_{\mathrm E}$ is finite by equivalence of norms, and
$\delta_2$ is finite because the forecast set $\cY$ is compact.
The constant $C_{\mathrm{samp}}$ controls only the sampling error.

\begin{lemma}[Predictably created forecast labels]
\label{lem:predictable-sampling}
Suppose at most $m$ distinct forecast points appear in the supports of the
sampling distributions during $T$ rounds, where $m$ is a deterministic bound. Each point is chosen
before any sample that can select it.  Under the timing condition above,
\begin{equation}
 \mathbb E\CAL_{\cL,T}^{\mathrm{sample}}
 \le \mathbb E\CAL_{\cL,T}(\mu,y)+C_{\mathrm{samp}}\sqrt{mT},
 \label{eq:global-sampling-bound}
\end{equation}
where
$\CAL_{\cL,T}(\mu,y)
=\sum_p\|\sum_{t<T}\mu_t(p)(y_t-p)\|_{\cL}$.
The labels may be random and may depend on past sampled forecasts.
\end{lemma}

\begin{proof}
Assign each new distinct point a permanent slot in order of
first appearance in a sampling distribution's support, using a fixed
measurable order within a round.
The slots are $k=1,\ldots,m$.  A point already seen retains its slot;
thus equal point values are merged as they appear.  Fix an anchor
$p_{\mathrm{init}}\in\cY$.  Let $p_{t,k}$ be the value of slot $k$ at round $t$
if that slot has been born, and let $p_{t,k}=p_{\mathrm{init}}$ before its birth.
Let $w_{t,k}$ be its sampling probability, with $w_{t,k}=0$ before birth.
Once born, a slot's point value never changes.

For precision about conditioning, let $\mathcal F_t$ contain the entire
history before the current outcome is selected, including the current
sampling distribution and all labels in its support. This filtration is
used for the analysis and need not represent the adversary's observations. Set
$\mathcal G_t=\mathcal F_t\vee\sigma(y_t)$, and take
$\mathcal F_{t+1}$ to contain $\mathcal G_t$ and the current sampled
forecast.  Writing $I_t$ for the sampled slot, the timing condition is
\[
 \mathbb P(I_t=k\mid\mathcal G_t)=w_{t,k}.
\]
The outcome may depend on all earlier samples.  Define the $\R^d$-valued
increments
\[
 \xi_{t,k}=(\mathbf 1\{I_t=k\}-w_{t,k})(y_t-p_{t,k}),
 \qquad \Xi_k=\sum_{t<T}\xi_{t,k}.
\]
Before birth these increments are zero, and otherwise their point value
is already known.  Consequently
\begin{align*}
 \mathbb E[\xi_{t,k}\mid\mathcal G_t]&=0,\\
 \mathbb E[\|\xi_{t,k}\|_2^2\mid\mathcal G_t]
 &=w_{t,k}(1-w_{t,k})\|y_t-p_{t,k}\|_2^2.
\end{align*}
For $i<t$, the increment $\xi_{i,k}$ is $\mathcal G_t$-measurable, so
$\mathbb E\langle \xi_{i,k},\xi_{t,k}\rangle=0$.  Euclidean martingale
orthogonality and $\sum_kw_{t,k}=1$ therefore give
\begin{align}
 \sum_{k=1}^m\mathbb E\|\Xi_k\|_2^2
 &=\sum_{t<T}\sum_{k=1}^m
   \mathbb E\left[w_{t,k}(1-w_{t,k})
                   \|y_t-p_{t,k}\|_2^2\right]\notag\\
 &\le \delta _2^2T.
 \label{eq:slot-quadratic-variation}
\end{align}
In particular, no conditioning on the terminal collection of labels
has been used.  Cauchy--Schwarz over the fixed $m$ slots, followed by
Jensen's inequality, yields
\[
 \mathbb E\sum_{k=1}^m\|\Xi_k\|_{\cL}
 \le c_{\mathrm E}\,\mathbb E\sqrt{m\sum_{k=1}^m\|\Xi_k\|_2^2}
 \le c_{\mathrm E}\sqrt{m\sum_{k=1}^m\mathbb E\|\Xi_k\|_2^2}
 \le C_{\mathrm{samp}}\sqrt{mT}.
\]
Pathwise, for each born slot the sampled residual equals its
distributional residual plus $\Xi_k$.  Slots not born contribute zero.
The triangle inequality proves \eqref{eq:global-sampling-bound}.
Alternatively, an algorithm may maintain separate slots for repeated
values.  Apply the variance argument to these slots and then sum their
discrepancies within each point bucket; this merging contracts the sum
of discrepancy norms.  Such a proof requires a deterministic bound
on the total number of slots, including the ones never sampled.
\end{proof}

\subsection{The Repeated Harmonic Algorithm}
\label{sec:sampled-repetition}

Run Algorithm~\ref{alg:harmonic} for $n\ge2$ virtual rounds, indexed by
$u=0,\ldots,n-1$, and choose an integer repetition count $S\ge1$.
Write $\overline\mu_u$ and $\overline y_u$ for its distribution and outcome.
Virtual round $u$ occupies actual rounds $t=uS,\ldots,(u+1)S-1$,
so the actual horizon is $T=nS$. On every round in this block, use
$\mu_t=\overline\mu_u$ to draw a fresh forecast.
After the block, return the average outcome
\begin{equation}
 \overline y_u=\frac1S\sum_{t=uS}^{(u+1)S-1}y_t
 \label{eq:virtual-block-outcome}
\end{equation}
to the virtual algorithm.  Convexity ensures that
$\overline y_u\in\cY$.  In particular, $\overline\mu_u$ uses only
averages from blocks completed before virtual round $u$.
Within a block, the distribution remains fixed even if the adversary
adapts to past sampled forecasts.  The virtual initialization is
$\overline\mu_0=\delta_{p_{\mathrm{init}}}$; the stored virtual points are
$\overline p_u$, obtained by applying Algorithm~\ref{alg:harmonic} to the
averages $\overline y_0,\ldots,\overline y_u$. The virtual distributions are supported
only on $p_{\mathrm{init}},\overline p_0,\ldots,\overline p_{n-2}$, and thus
use at most $n$ distinct forecast points.

\begin{theorem}[Expected sampled calibration]
\label{thm:sampled-harmonic}
Suppose there is a constant $C_{\mathrm{cal}}\ge0$ such that the harmonic
algorithm satisfies, for every virtual outcome sequence and every $n\ge2$,
\begin{equation}
 \frac1n\CAL_{\cL,n}(\overline\mu,\overline y)
 \le\frac{C_{\mathrm{cal}}}{h_n}.
 \label{eq:virtual-error-assumption}
\end{equation}
Then the repeated algorithm satisfies, at the real horizon $T=nS$,
\begin{equation}
 \frac1T\mathbb E\CAL_{\cL,T}^{\mathrm{sample}}
 \le\frac{C_{\mathrm{cal}}}{h_n}+\frac{C_{\mathrm{samp}}}{\sqrt S}.
 \label{eq:sampled-harmonic-rate}
\end{equation}
In particular, for every $\varepsilon>0$, the choices
\begin{equation}
 n=\max\{2,\lceil \exp(4C_{\mathrm{cal}}/\varepsilon)\rceil\},\qquad
 S=\max\{1,\lceil4C_{\mathrm{samp}}^2/\varepsilon^2\rceil\},\qquad T=nS
 \label{eq:sampled-horizon-choice}
\end{equation}
give
\begin{equation}
 \mathbb E\CAL_{\cL,T}^{\mathrm{sample}}
 \le\varepsilon T,
 \qquad T\le n(1+4C_{\mathrm{samp}}^2/\varepsilon^2).
 \label{eq:sampled-final-horizon}
\end{equation}
One may take $C_{\mathrm{cal}}=9\gamma(\cY,\cL)$: the proof of
Theorem~\ref{thm:general-geometric-rate}, before replacing $h_n$ by a
logarithm, gives
\[
 \frac1n\CAL_{\cL,n}
 \le\frac3n\Gamma_{\widetilde\cY,\cL^\circ}^{(n)}(\widehat K_n)
 \le\frac{9\gamma(\cY,\cL)}{h_n},
\]
using Lemma~\ref{lem:harmonic-kernel-comparison}.
\end{theorem}

\begin{proof}
Write $\mu_t=\overline\mu_u$ on block $u$.  For every point value $p$, exactly
\[
 \sum_{t<T}\mu_t(p)(y_t-p)
 =S\sum_{u<n}\overline\mu_u(p)(\overline y_u-p).
\]
Thus the real distributional calibration error equals $S$ times the
virtual distributional error.  Although the block averages may be
random and depend on within-block samples, the virtual guarantee is
pathwise for every outcome sequence, so its use here is valid.
The virtual distributions use at most $n$ distinct support points, including
$p_{\mathrm{init}}$, each known before it can be sampled.  Lemma~\ref{lem:predictable-sampling},
applied once over all $T$ real rounds, adds at most $C_{\mathrm{samp}}\sqrt{nT}$.
Since $T=nS$, division by $T$ proves
\eqref{eq:sampled-harmonic-rate}.

The elementary inequality $2h_n\ge\log n$ and
\eqref{eq:sampled-horizon-choice} bound the virtual term by
$\varepsilon/2$.  The same choices bound the sampling term by
$\varepsilon/2$.  Finally,
$\max\{1,\lceil x\rceil\}\le1+x$ for $x\ge0$, proving the stated
integer horizon bound, including the degenerate case $C_{\mathrm{samp}}=0$.
\end{proof}

The global application of Lemma~\ref{lem:predictable-sampling} is
essential for the stated cost.  Applying a separate support bound within
each block and adding the results would instead give the valid but
larger normalized cost $C_{\mathrm{samp}}\sqrt{n/S}$.

\begin{corollary}[The three principal sampled horizons]
\label{cor:sampled-examples}
Let $0<\varepsilon\le1$. Choose an absolute constant $C_0>0$ such that
$4C_{\mathrm{cal}}\le C_0$ for the Euclidean unit ball and
$2C_{\mathrm{cal}}\le C_0\log(2d)$ both for the cube with the infinity
norm and for the simplex with the one norm, using the geometric bounds
in Section~\ref{sec:geometry}. Set
\[
 n_{\mathrm{ball}}=\left\lceil\exp\left(\frac{C_0}{\varepsilon}\right)
            \right\rceil,
 \qquad
 n_{\mathrm{poly}}=\left\lceil(2d)^{2C_0/\varepsilon}\right\rceil,
\]
increasing $C_0$ if necessary so that $n_{\mathrm{poly}}\ge2$.  Expected sampled
calibration at most $\varepsilon$ is attained at the following real
horizons:
\begin{align}
 T_{B_2^d,2}^{\mathrm{sample}}
 &=n_{\mathrm{ball}}\left\lceil\frac{16}{\varepsilon^2}\right\rceil
 \le n_{\mathrm{ball}}\left(1+\frac{16}{\varepsilon^2}\right),
 \label{eq:sampled-ball-horizon}\\
 T_{[0,1]^d,\infty}^{\mathrm{sample}}
 &=n_{\mathrm{poly}}\left\lceil\frac{16d}{\varepsilon^2}\right\rceil
 \le n_{\mathrm{poly}}\left(1+\frac{16d}{\varepsilon^2}\right),
 \label{eq:sampled-cube-horizon}\\
 T_{\Delta_d,1}^{\mathrm{sample}}
 &=n_{\mathrm{poly}}\left\lceil\frac{8d}{\varepsilon^2}\right\rceil
 \le n_{\mathrm{poly}}\left(1+\frac{8d}{\varepsilon^2}\right).
 \label{eq:sampled-simplex-horizon}
\end{align}
Here these symbols denote particular sufficient horizons of the repeated
algorithm.  In particular their asymptotic upper bounds are
$\exp(C/\varepsilon)$, $(2d)^{C/\varepsilon}$, and
$(2d)^{C/\varepsilon}$, respectively, for an absolute constant $C$.
\end{corollary}

\begin{proof}
For the Euclidean unit ball, take $c_{\mathrm E}=1$, $\delta _2=2$, and $C_{\mathrm{samp}}=2$, and use
$\gamma(B_2^d,B_2^d)\le\pi/2$. For the cube $[0,1]^d$, take
$c_{\mathrm E}=1$ and $\delta_2=\sqrt d$. The displayed repetition count
also covers the centered cube $[-1,1]^d$, using the common upper bound
$C_{\mathrm{samp}}\le2\sqrt d$.  For the simplex, the inequalities
$\|v\|_1\le\sqrt d\|v\|_2$ and
$\operatorname{diam}_2(\Delta_d)\le\sqrt2$ give
$C_{\mathrm{samp}}\le\sqrt{2d}$.  The displayed bounds follow from
Theorem~\ref{thm:sampled-harmonic}; the simplex at $d=1$ is a singleton
and in fact has zero error.  For $0<\varepsilon\le1$, the polynomial
factors in $d$ and $1/\varepsilon$ are absorbed by increasing the
absolute exponent constant, using $\log(1/\varepsilon)\le1/\varepsilon$
and $\log(2d)\ge\log2$.
\end{proof}

These are expectation statements at the specified product horizons.
No high-probability or anytime sampled guarantee is implicit.  The
Euclidean bound is dimension independent; the displayed Euclidean
embedding argument does not give dimension independence for every
norm on $\R^d$.

\section{Efficient Implementation}
\label{app:implementation}

Implementing harmonic calibration involves computing the stored forecast
points and, in the sampled setting, selecting a point according to the
harmonic replay weights. Section~\ref{app:fft} computes the stored points
by an online fast Fourier transform (FFT), using $O(d\log^2T)$ amortized
arithmetic operations per round and $O(dT)$ storage over $T$ updates.
Section~\ref{sec:sampling-implementation} shows how to draw efficiently
from the replay distribution and combines these two operations for the
repeated algorithm of Appendix~\ref{app:sampling}.

We count arithmetic operations and comparisons. Sampling additionally
uses unit-cost draws from a continuous uniform distribution.

\subsection{Computing Stored Forecasts via Online FFT}
\label{app:fft}

Direct evaluation of the new stored forecast $p_t$ sums $t+1$ outcome
vectors and costs $O(d(t+1))$ arithmetic operations, or $O(dT^2)$ over
$T$ rounds. The harmonic weights depend only on elapsed time, so these
sums are normalized causal convolutions. We specialize the online
multiplication methods of Fischer and Stockmeyer~\cite{fischer1974online}
and van der Hoeven~\cite[Section~4.3, Theorem~4.1]{hoeven2002relax}.
Define the causal kernel
\begin{equation}\label{eq:fft-harmonic-kernel}
 w_r=1/(2r+1),\qquad r\ge0.
\end{equation}
Then
\begin{equation}\label{eq:fft-causal-convolution}
 c_t=(w*y)_t=\sum_{s=0}^tw_{t-s}y_s,\qquad
 p_t=\frac{c_t}{h_{t+1}}.
\end{equation}
The challenge is to compute $c_t$ using only outcomes already observed.

\paragraph{Dyadic blocks.}
For level $k\ge0$ and block index $q\ge1$, define
\begin{equation}\label{eq:fft-dyadic-blocks}
 Y^{(k,q)}=(y_{(q-1)2^k},\ldots,y_{q2^k-1}),\qquad
 W^{(k)}=(w_{2^k},\ldots,w_{2^{k+1}-1}).
\end{equation}
The outcomes are columns of the $d\times2^k$ matrix $Y^{(k,q)}$.
Every product $w_ry_s$ with $r\ge1$ belongs to exactly one pair of
blocks: choose $k=\lfloor\log_2r\rfloor$ and
$q=\lfloor s/2^k\rfloor+1$.
The outcome block completes on round $t=q2^k-1$, whereas its earliest
contribution is needed at index
\begin{equation}\label{eq:fft-causality}
 (q-1)2^k+2^k=t+1.
\end{equation}
Its convolution with $W^{(k)}$ can therefore be added to future
accumulators when the block completes. The diagonal term $w_0y_t=y_t$
is added when $y_t$ arrives.

\begin{algorithm}[t]
\caption{Online FFT computation of harmonic stored forecasts}
\label{alg:online-fft}
\begin{algorithmic}[1]
\State Initialize $A_0,\ldots,A_{T-1}\gets0\in\R^d$ and $h_1\gets1$.
\For{$t=0,1,\ldots,T-1$}
 \State Observe $y_t$ and store $p_t\gets(A_t+y_t)/h_{t+1}$.
 \State Set $h_{t+2}\gets h_{t+1}+w_{t+1}$.
 \For{each $k\ge0$ with $2^k\mid(t+1)$}
  \State Set $Y^{(k)}\gets(y_{t-2^k+1},\ldots,y_t)$.
  \State Compute $B^{(k)}=Y^{(k)}*W^{(k)}$ coordinatewise by FFT,
         using transform length $2^{k+1}$.
  \For{$r=0,\ldots,2^{k+1}-2$ with $t+1+r<T$}
   \State $A_{t+1+r}\gets A_{t+1+r}+B^{(k)}_r$.
  \EndFor
 \EndFor
\EndFor
\end{algorithmic}
\end{algorithm}

The linear convolution has length $2^{k+1}-1$, so zero-padding to
$2^{k+1}$ avoids wraparound. Its entry $r$ contains the products
$w_{2^k+j}y_{t-2^k+1+i}$ with $i+j=r$; their output index is exactly
$t+1+r$. The scalar transform of $W^{(k)}$ is precomputed once and
reused for every coordinate and block at level $k$.

\begin{proposition}[Online FFT complexity]
\label{prop:online-fft-complexity}
For $T\ge2$, Algorithm~\ref{alg:online-fft} computes the exact stored
forecasts in $O(dT\log^2T)$ arithmetic operations and $O(dT)$ storage.
Its amortized per-round cost is $O(d\log^2T)$.
\end{proposition}
\begin{proof}
The block partition and \eqref{eq:fft-causality} show that $A_t$
contains every off-diagonal contribution exactly once before round $t$.
A level-$k$ block costs $O(d(k+1)2^k)$ operations and completes
$\lfloor T/2^k\rfloor$ times. Summing gives
\[
 \sum_{k\le\lfloor\log_2T\rfloor}
       \left\lfloor\frac T{2^k}\right\rfloor O(d(k+1)2^k)
 =O(dT\log^2T).
\]
The normalizations add $O(dT)$ work. The outcomes, forecasts,
accumulators and transform workspaces occupy $O(dT)$ storage;
output indices at least $T$ can be discarded.
\end{proof}

\paragraph{Distributional forecasts and response actions.}
\label{sec:fft-consequences}
An implicit distribution stores its forecast points and the weight
rule $\mu_t=\sum_{s<t}w_{t-1-s}\delta_{p_s}/h_t$.
This representation costs $O(dT\log^2T)$ over $T$ rounds. Explicitly
printing all replay weights adds $O(T^2)$ output operations.
For the response algorithm, use the main text's $b_s=b(p_s)\in\R^m$:
\begin{equation}\label{eq:fft-response-convolution}
 x_t=\frac1{h_t}\sum_{s<t}w_{t-1-s}b_s
     =\frac{(w*b)_{t-1}}{h_t},\qquad t\ge1.
\end{equation}
A second causal convolution gives total cost
$O((d+m)T\log^2T)$, plus at most $T$ calls to the response oracle.
For orthant calibration, keep the weighted list of realized sign
labels instead of a dense $2^d$-dimensional action. Computing each
label takes $O(d)$ work, so only the first convolution is needed.
Section~\ref{sec:sampling-implementation} applies the same method to
the $n$ virtual outcomes of the sampled algorithm.

\subsection{Implementing the Sampled Algorithm}
\label{sec:sampling-implementation}

The sampled algorithm of Appendix~\ref{app:sampling} must output a fresh
forecast on every actual round. Given the stored points computed by
Section~\ref{app:fft}, it remains to select one with the prescribed
replay probability. The harmonic sums let us do this by binary search.

\paragraph{Sampling from harmonic replay.}
Put $h_0=0$ and use the kernel $w_r=1/(2r+1)$ defined in
\eqref{eq:fft-harmonic-kernel}.
On virtual round $u\ge1$,
the distribution $\overline\mu_u$ assigns probability $w_r/h_u$ to
$\overline p_{u-1-r}$, for $0\le r<u$. The harmonic sums are already
the cumulative weights:
\[
 \sum_{r=0}^j w_r=h_{j+1},\qquad 0\le j<u.
\]
Draw $U$ uniformly from $(0,h_u)$, use binary search to find the least
$r$ with $h_{r+1}\ge U$, and output $\overline p_{u-1-r}$. The interval
for lag $r$ has length $h_{r+1}-h_r=w_r$, so the sampled index has
exactly the required law. Equal stored points simply receive the sum
of their index probabilities. At $u=0$, output $p_{\mathrm{init}}$.
Each draw uses fresh randomness under the timing condition stated in
Appendix~\ref{app:sampling}.

The array $h_0,\ldots,h_n$ takes $O(n)$ operations to construct and
$O(n)$ space. A draw needs $O(\log n)$ comparisons and $O(d)$ work
to output a dense forecast vector. Thus sampling never requires
constructing a new cumulative-probability table for each virtual round.
The distribution can remain fixed throughout a block without reusing
the sampled forecast.

\paragraph{Total cost with repetition.}
With repetition, forecast updates occur on $n$ virtual rounds, while
sampling and outcome collection occur on all $T=nS$ actual rounds.

\begin{samepage}
\begin{proposition}[Runtime of sampled harmonic calibration]
\label{prop:sampled-runtime}
For $n\ge2$ virtual rounds and $S\ge1$ repetitions per round, the
algorithm of Theorem~\ref{thm:sampled-harmonic} can be implemented in
\[
 O\bigl(dn\log^2 n+dT+T\log n\bigr)
\]
total arithmetic operations and comparisons, with $O(dn)$ storage,
where $T=nS$. Its amortized cost per actual round is
\[
 O\!\left(d+\log n+\frac{d\log^2 n}{S}\right).
\]
\end{proposition}
\end{samepage}
\begin{proof}
Each block average is available before the next virtual forecast is
required. Proposition~\ref{prop:online-fft-complexity} therefore computes
the stored points in $O(dn\log^2 n)$ operations and $O(dn)$ storage.
Maintaining one running outcome sum per block and outputting the forecasts
costs $O(dT)$; the binary searches cost $O(T\log n)$. The running sum
uses $O(d)$ additional storage. Dividing by $T=nS$ gives the amortized bound.
\end{proof}

When $S\ge\log^2 n$, the amortized bound is $O(d+\log n)$ per actual
round.

\section{Classical, Response, and Max--Min--Max Approachability}
\label{app:approachability-foundations}

This appendix develops the max--min--max perspective on approachability.
We first pass from classical payoffs to the minimal response target,
then express its distance as a maximum of affine tests. The resulting
minimax identity exchanges responses to outcomes with responses to
tests. For cone targets, this duality relates response maps to
families of maps with feasible zeros, and hence to fixed points of
transformations. It does not let the learner move first in a single
round. We then develop two principal routes through repeated play:
regret minimization over tests, including transformation regret and
Blum--Mansour, and calibrated outcome responses. Skew-symmetric
transformations and orthant calibration illustrate these routes,
followed by connections to profile swap regret and tensor norms.

\subsection{From Classical Payoffs to a Minimal Response Target}
\label{app:response-universality}

Retain the nonempty compact convex action and outcome sets
$\mathcal X\subset\R^m$, $\cY\subset\R^d$ and the order of play
from Section~\ref{sec:app-response}. Blackwell's classical formulation
allows an arbitrary bilinear payoff map
$\kappa:\mathcal X\times\cY\to\mathcal Z\subset\R^n$
and a nonempty convex target $\mathcal S\subset\R^n$.
Fix a norm $\|\cdot\|$ on $\R^n$.
The average payoff and approachability loss are
\begin{equation}
 \bar\kappa_T=\frac1T\sum_{t<T}\kappa(x_t,y_t),\qquad
 \dist(\bar\kappa_T,\mathcal S)=\min_{s\in \mathcal S}\|\bar\kappa_T-s\|.
 \label{eq:app-classical-loss}
\end{equation}
We seek to make this distance tend to zero uniformly over adaptive
outcome sequences.
For quantitative comparisons, fix a bound $B>0$ with
$\|\kappa(x,y)\|\le B$ on the feasible sets; such a bound exists by
compactness. Continuous separately affine payoffs are included by
appending constant coordinates, as described below.

The response condition from Section~\ref{sec:app-response} applies
to this general payoff map as well: Blackwell's theorem says that $\mathcal S$
is approachable exactly when every $y\in\cY$ admits an action $x$
with $\kappa(x,y)\in \mathcal S$ \cite{blackwell1956analog,perchet2014approachability}.
For an approachable target, we may therefore select a response map
\begin{equation}
 b:\cY\to\mathcal X,\qquad \kappa(b(y),y)\in \mathcal S\quad(y\in\cY).
 \label{eq:forecast-response}
\end{equation}
Each response supplies a point of $\mathcal S$ with norm at most $B$, so the
triangle inequality gives $\dist(\bar\kappa_T,\mathcal S)\le2B$ on every
transcript. Rescaling the norm by $1/B$ puts the loss on the same
bounded scale as in Section~\ref{sec:app-response}: payoffs have norm
at most one and distance to the target is at most two.

We can now compare this classical instance with the response
construction. We first encode $\kappa$ through tensors, then take the
smallest convex target compatible with the chosen responses.

\paragraph{Why tensors contain all the payoff information.}
To relate the general payoff map $\kappa$ to the one used in
Section~\ref{sec:app-response}, write its bilinear expansion using
fixed coefficient vectors $v_{ij}\in\R^n$:
\[
 \kappa(x,y)=\sum_{i=1}^d\sum_{j=1}^m y_i x_j v_{ij}.
\]
The tensor $y\otimes x=yx^{\mathsf T}\in\R^{d\times m}$ records
exactly these products: its $(i,j)$ entry is $y_i x_j$.
Define the linear map $L:\R^{d\times m}\to\R^n$ by
\[
 L(C)=\sum_{i=1}^d\sum_{j=1}^m C_{ij}v_{ij}.
\]
Then $\kappa(x,y)=L(y\otimes x)$: applying $L$ to the tensor payoff
recovers the original vector payoff. The tensor map is the same for
every bilinear $\kappa$ on these action and outcome sets; the choice
of $\kappa$ determines the coefficients in $L$.

For separately affine payoffs, append constant coordinates when needed:
\[
 \widehat y=(y,1),\qquad \widehat x=(x,1),\qquad
 \widehat y\otimes\widehat x=\widehat y\widehat x^{\mathsf T}
 =\begin{pmatrix}yx^{\mathsf T}&y\\x^{\mathsf T}&1\end{pmatrix}.
\]
The four blocks encode the bilinear term, the two linear terms,
and the constant. Every separately affine payoff therefore has the
form $\kappa(x,y)=L(\widehat y\otimes\widehat x)$ in at most
$(d+1)(m+1)$ coordinates. In the bilinear case, no lift is needed.
Existing constant coordinates, as on a simplex, can also avoid a lift.
Appendix~\ref{app:affine-response-geometry} records the effect on norms.

Below, $x,y$ and the response $b(y)$ denote the augmented coordinates
when a lift is needed. In either case, linearity gives the same relation
for time averages:
\[
 \bar\kappa_T
 =L\left(\frac1T\sum_{t<T}y_t\otimes x_t\right).
\]
Thus the average tensor payoff determines the average payoff of the
original problem. We next choose a target in tensor space whose image
under $L$ lies in $\mathcal S$.

\paragraph{Why the target can be taken minimal.}
Once $b$ is fixed, response satisfiability requires the target to
contain every $\kappa(b(y),y)$. Convexity then requires it to contain
their convex hull. In the universal tensor space, this gives
\[
 \mathcal S_b
 =\conv\{y\otimes b(y):y\in\cY\}.
\]
This target itself satisfies the response condition. Every classical
target compatible with $b$ contains its image under $L$, so approaching
$\mathcal S_b$ is enough to approach the original target.

To compare distances, equip the span of the tensor payoffs with the
norm whose unit ball is
\[
 \mathcal K=\conv\{\pm(y\otimes x):y\in\cY,\ x\in\mathcal X\}.
\]
Every feasible tensor has norm at most one, and
$\mathcal S_b\subseteq\mathcal K$. Hence the tensor
approachability loss is at most two, matching the normalized classical
scale above. Write $\|L\|$ for the operator norm from
$\|\cdot\|_{\mathcal K}$ to $\|\cdot\|$, namely
$\|L\|=\max_{\|C\|_{\mathcal K}\le1}\|L(C)\|$.
The following proposition gives the quantitative comparison.

\begin{proposition}[Universality of the response target]
\label{prop:app-universality}
For any norm $\|\cdot\|_{\mathcal K}$ on the span of the tensor payoffs,
\begin{equation}
 \dist\left((1/T)\sum_{t<T}\kappa(x_t,y_t),\mathcal S\right)
 \le \|L\|\,
       \dist_{\mathcal K}\left(
          (1/T)\sum_{t<T}y_t\otimes x_t,\mathcal S_b\right).
 \label{eq:app-universality}
\end{equation}
If $\|\kappa(x,y)\|\le B$ on the feasible sets and
$\mathcal K=\conv\{\pm(y\otimes x):y\in\cY,x\in\mathcal X\}$,
then $\|L\|\le B$.
\end{proposition}
\begin{proof}
Each generating graph point maps to $\kappa(b(y),y)\in \mathcal S$.
Linearity of $L$ and convexity of $\mathcal S$ give
$L(\mathcal S_b)\subseteq \mathcal S$. For
$C_T=(1/T)\sum_t y_t\otimes x_t$ and $D\in\mathcal S_b$,
\[
 \dist(LC_T,\mathcal S)\le\|L(C_T-D)\|
 \le\|L\|\|C_T-D\|_{\mathcal K}.
\]
Minimizing over $D\in\mathcal S_b$ proves the claim. Every atom of the stated body
maps to a vector of norm at most $B$, which proves the operator bound.
\end{proof}

In particular, normalize the classical payoff bound to $B=1$.
The proposition shows that, on every transcript, the distance to any
classical target compatible with $b$ is at most the tensor distance to
$\mathcal S_b$. Conversely, the tensor problem is itself a
normalized classical instance: take $\kappa(x,y)=y\otimes x$, use the
$\mathcal K$-norm on its span, and take $\mathcal S=\mathcal S_b$.
For this choice the comparison is an equality. Thus, for fixed feasible
sets and a fixed response map, the response formulation realizes the
largest loss among compatible normalized classical instances. This is
the precise sense in which it captures the worst case.
Minimality is relative to $b$: different response selections can give
different targets. When affine lifts are needed, they preserve convex
mixtures and the order of play.

\subsection{The Max--Min--Max Formulation}
\label{app:maxminmax}
\label{app:classical-support}

The approachability loss $\dist_{\mathcal K}(\kappa,\mathcal S)$ is
convex in $\kappa$, so it can be written as a maximum of affine
functions, which we call \emph{tests}. We represent each test by a
pair $W=(U,a)$, consisting of a matrix $U\in\R^{d\times m}$ and a
scalar offset $a\in\R$, and write
\[
 W(\kappa)=\langle U,\kappa\rangle-a.
\]

For distance, the dual formula specifies the tests explicitly.
Writing $\sigma_{\mathcal S}(U)=\sup_{s\in\mathcal S}\langle U,s\rangle$
for the support function of the target, we have
\begin{equation}
 \dist_{\mathcal K}(\kappa,\mathcal S)
 =\max_{U\in\mathcal K^\circ}
       \{\langle U,\kappa\rangle-\sigma_{\mathcal S}(U)\}.
 \label{eq:app-polar-tests}
\end{equation}
These coefficient pairs form a compact set, since $\mathcal K^\circ$
is compact and $\sigma_{\mathcal S}$ is continuous. Taking their
convex hull preserves the maximum, so we obtain a compact convex
test family:
\begin{equation}
 \mathcal W=\conv\{(U,\sigma_{\mathcal S}(U)):U\in\mathcal K^\circ\},
 \qquad
 \dist_{\mathcal K}(\kappa,\mathcal S)=\max_{W\in\mathcal W}W(\kappa).
 \label{eq:app-support-distance}
\end{equation}
On an action--outcome pair $(x,y)$, the value of a test $W=(U,a)$ is
\[
 f(W,x,y):=W(yx^{\mathsf T})=y^{\mathsf T}Ux-a.
\]
With $W$ identified with its coefficient pair, $f$ is continuous
and affine in each of its three arguments.
The family contains the zero test, and $\mathcal S\subseteq\mathcal K$
gives the boundedness condition
\begin{equation}
 |f(W,x,y)|\le2
 \qquad\text{for all }(W,x,y)\in\mathcal W\times\mathcal X\times\cY.
 \label{eq:app-test-boundedness}
\end{equation}
Response satisfiability is consequently the statement
\begin{equation}
 \max_{y\in\cY}\min_{x\in\mathcal X}
       \max_{W\in\mathcal W}f(W,x,y)=0.
 \label{eq:app-response-satisfiability}
\end{equation}
For each $y$, the response $b(y)$ puts the payoff in the target,
giving a value at most zero; the zero test gives the reverse inequality.
Together, boundedness~\eqref{eq:app-test-boundedness} and response
satisfiability~\eqref{eq:app-response-satisfiability} give an abstract
\emph{max--min--max formulation}: three nonempty compact convex sets
$\mathcal W,\mathcal X,\cY$ and a continuous payoff
$f:\mathcal W\times\mathcal X\times\cY\to\R$, affine in each of its
three arguments, satisfying these two conditions.

\paragraph{Exchanging the two maxima.}
The max--min--max form has a useful symmetry. The following identity
applies both to the affine tests above and to the matrix
tests introduced below.
\begin{proposition}[Max--min--max identity]
\label{prop:app-maxminmax}
Let $\mathcal W,\mathcal X,\cY$ be nonempty compact convex sets, and
let $f(W,x,y)$ be continuous and affine in each of its three arguments.
Then
\begin{equation}
 \max_y\min_x\max_W f(W,x,y)
 =\max_W\min_x\max_y f(W,x,y).
 \label{eq:bilinear-response-condition}
\end{equation}
In particular, every outcome admits an action satisfying all tests
if and only if every test admits an action satisfying all outcomes.
\end{proposition}
\begin{proof}
Apply the minimax theorem at fixed $y$, commute the two maxima, and
apply minimax at fixed $W$:
\[
 \max_y\min_x\max_W f
 =\max_y\max_W\min_x f
 =\max_W\max_y\min_x f
 =\max_W\min_x\max_y f.
\]
Each exchange holds by the minimax theorem because, with the remaining
argument fixed, $f$ is affine in both variables being optimized
\cite[Theorem~3.4]{sion1958general}.
Compactness gives attainment and hence the quantifier interpretation.
\end{proof}

Thus response satisfiability has two equivalent forms:
\begin{equation}
 \begin{gathered}
 \forall y\in\cY\quad\exists x\in\mathcal X\quad
 \forall W\in\mathcal W:\quad f(W,x,y)\le0
 \\
 \Longleftrightarrow
 \\
 \forall W\in\mathcal W\quad\exists x\in\mathcal X\quad
 \forall y\in\cY:\quad f(W,x,y)\le0.
 \end{gathered}
 \label{eq:app-response-condition}
\end{equation}
These give two kinds of response maps: outcomes to actions, and
tests to actions. They underlie the two classical algorithmic routes
to Blackwell approachability:
\begin{itemize}
 \item \emph{No regret over $\mathcal W$, then best response:}
       $W_t\mapsto x_t\in\arg\min_{x\in\mathcal X}
                          \max_{y\in\cY}f(W_t,x,y)$
       \mbox{(Appendix~\ref{app:test-regret})}.
 \item \emph{Calibration over $\cY$, then best response:}
       $y_t\mapsto x_t\in\arg\min_{x\in\mathcal X}
                         \max_{W\in\mathcal W}f(W,x,y_t)$
       \mbox{(Appendix~\ref{app:calibrated-responses})}.
\end{itemize}

\paragraph{Stationary distributions.}
\label{app:stationary-response}
As an illustration of the max--min--max identity, we give a short
proof that every finite Markov chain has a stationary distribution,
without spectral arguments. Let $\Phi$ be the set of all
column-stochastic matrices, and set
\[
 \mathcal X=\Delta_d,\qquad \cY=[-1,1]^d,\qquad
 \mathcal W=\{I-\phi:\phi\in\Phi\},
\]
with payoff $f(W,x,y)=y^{\mathsf T}Wx$.
For each $y$, choose $b(y)\in\arg\min_{x\in\Delta_d}y^{\mathsf T}x$. Then
\[
 y^{\mathsf T}(I-\phi)b(y)
 =y^{\mathsf T}b(y)-y^{\mathsf T}\phi b(y)\le0
\]
for every $\phi\in\Phi$, because $\phi b(y)\in\Delta_d$
and $b(y)\in\arg\min_{x\in\Delta_d}y^{\mathsf T}x$.
Thus $b(y)$ satisfies every test at outcome $y$.
Since $\mathcal W,\mathcal X,\cY$ are compact and convex and $f$ is
trilinear, Proposition~\ref{prop:app-maxminmax} gives
\[
 \forall\phi\in\Phi\quad\exists x\in\Delta_d\quad
 \forall y\in[-1,1]^d:\quad y^{\mathsf T}(I-\phi)x\le0.
\]
Since $[-1,1]^d$ contains a ball around zero, these inequalities
force $(I-\phi)x=0$. Hence $\phi x=x$: the distribution $x$ is stationary.

\subsection{Equivalence between Preconditions of Blackwell Approachability and \texorpdfstring{\mbox{$\Phi$-Regret}}{Phi-Regret} Minimization}
\label{app:response-cones}
\label{app:fixed-point-responses}

Response satisfiability \eqref{eq:app-response-condition} has a useful
structure when $\cY$ contains a ball around zero.

First observe that, for the hardest response targets, it is enough
to consider response maps that are \emph{scale-invariant}, meaning
$b(\lambda y)=b(y)$ for all $\lambda>0$.
Recall that a response map $b:\cY\to\mathcal X$ defines the target
\[
 \mathcal S=\conv\{yb(y)^{\mathsf T}:y\in\cY\}.
\]
A smaller target is at least as hard to approach.

Since $0\in\cY$, the zero payoff belongs to $\mathcal S$.
Now take an outcome $y$ on the boundary of $\cY$.
Its response payoff $yb(y)^{\mathsf T}$ also belongs to
$\mathcal S$. Convexity puts the whole segment between these two
points in the target:
\[
 \lambda\,yb(y)^{\mathsf T}\in\mathcal S\qquad(0\le\lambda\le1).
\]
Convexity of $\cY$ also gives $\lambda y\in\cY$ for every
$0\le\lambda\le1$. The definition of the response target therefore
requires the payoff $(\lambda y)b(\lambda y)^{\mathsf T}$ to belong
to $\mathcal S$ as well.
By setting $b(\lambda y)=b(y)$ for $0<\lambda\le1$, this payoff
becomes $\lambda yb(y)^{\mathsf T}$, a point that convexity has
already forced into $\mathcal S$. Making this choice for every
boundary outcome $y$ defines a scale-invariant response map.
It introduces no additional points to the target beyond those
that already had to be included by convexity.

A scale-invariant response map has a natural cone version of its
target. Every direction occurs in $\cY$, so we can keep the same
response beyond the boundary as well, obtaining
\[
 b:\R^d\to\mathcal X,\qquad b(\lambda y)=b(y)\quad(\lambda>0).
\]
Let $\mathcal S$ now denote the closed convex hull of the response
payoffs $yb(y)^{\mathsf T}$ over all $y\in\R^d$.
Each payoff comes with all of its positive multiples, because
\[
 (\lambda y)b(\lambda y)^{\mathsf T}
 =\lambda\,yb(y)^{\mathsf T}\qquad(\lambda>0).
\]
Hence $\mathcal S$ is a convex cone, whose supporting hyperplanes
all pass through the origin. The supporting tests used in
\eqref{eq:app-polar-tests} therefore have zero offsets, $a=0$, and
are linear functions represented by matrices in the convex cone
\[
 \mathcal U
 =\{U:y^{\mathsf T}Ub(y)\le0\ \text{for every }y\in\R^d\}.
\]
The payoff is $f(U,x,y)=y^{\mathsf T}Ux$, and the distance formula becomes
\begin{equation}
 \dist_{\mathcal K}(\kappa,\mathcal S)
 =\max_{U\in\mathcal U\cap\mathcal K^\circ}\langle U,\kappa\rangle.
 \label{eq:app-cone-distance}
\end{equation}
The restriction to $\mathcal K^\circ$ gives a compact test set and
bounded payoffs.

By construction, $b(y)$ satisfies every matrix test at outcome $y$.
The max--min--max identity exchanges these outcome responses with
responses to tests. Apply it first to the compact set
$\mathcal U\cap\mathcal K^\circ$. Every matrix in $\mathcal U$ can
be scaled down to this set, and positive scaling does not change
which actions satisfy a test. Consequently,
\[
 \forall U\in\mathcal U\quad\exists x\in\mathcal X\quad
 \forall y\in\cY:\quad y^{\mathsf T}Ux\le0.
\]
Now use the ball around zero in $\cY$. If $Ux\ne0$, a small positive
multiple of $Ux$ is a feasible outcome and has positive inner product
with $Ux$. Therefore the inequalities for all $y$ hold exactly when
$Ux=0$, giving
\begin{equation}
 \forall U\in\mathcal U\quad\exists x\in\mathcal X:\quad Ux=0.
 \label{eq:app-zero-response}
\end{equation}
We have arrived at a convex cone of matrices, each with a zero on
$\mathcal X$, starting from a scale-invariant response map.
Each matrix has its own zero; the matrices need not share one common zero.

If $0\in\mathcal X$, this condition is trivial: choose $x=0$.
We therefore assume $0\notin\mathcal X$, as on a simplex, or append
a constant action coordinate $(x,1)$. This also includes affine maps:
\[
 F(x)=Mx+c=[\,M\ \ c\,](x,1),\qquad f(F,x,y)=y^{\mathsf T}F(x).
\]
The vector $c$ is allowed: it contributes $y^{\mathsf T}c$, rather
than an independent scalar offset.

\paragraph{The converse.}
The argument also works in reverse. Suppose $\mathcal U$ is a closed
convex cone of matrices, each with a zero on $\mathcal X$.
First restrict to matrices of norm at most one, so that the test set
is compact. The same response duality gives, for every $y\in\cY$,
an action $b(y)$ satisfying all these tests. Rescaling gives the same
inequalities for every matrix in the cone.
Choose one response for each outcome direction and use it for all
positive multiples.
This gives a scale-invariant map on $\R^d$ with
$y^{\mathsf T}Ub(y)\le0$ for all $y$ and $U$.
Thus having a zero for each matrix is equivalent to having a common
scale-invariant outcome response. The choice of response need not be unique.

\paragraph{The fixed-point condition for transformation regret.}
When the maps take values in the action space ($d=m$ for $U$), set
\[
 \phi_U(x)=x-Ux,\qquad
 \phi_F(x)=x-F(x)\quad\text{for affine residuals }F.
\]
These transformations act in the original action coordinates.
Their fixed points are exactly the zeros above, and
\[
 y^{\mathsf T}Ux=\langle y,x-\phi_U(x)\rangle
\]
is the $\Phi$-regret payoff. The transformation family $\Phi$ is
convex when the family of maps is convex. For compact convex
families and the bounded outcome set $\cY$, the payoffs are bounded,
and the approachability response condition is exactly
\begin{equation}
 \forall\phi\in\Phi\quad\exists x\in\mathcal X:\quad \phi(x)=x.
 \label{eq:app-phi-fixed-points}
\end{equation}
This is the fixed-point precondition for $\Phi$-regret minimization.
It requires a fixed point for each transformation, without requiring
$\phi(\mathcal X)\subseteq\mathcal X$. The stronger endomorphism
condition is discussed in Appendix~\ref{app:test-regret}.

\subsection{Moving First Through Repeated Play}
\label{app:repeated-maxminmax}

The max--min--max identity lets us exchange the two maxima, but it
does not generally let us move the minimum first. We can nevertheless
move first in a sequential sense, by adapting our actions over repeated
rounds. This is the fundamental motivation for the framework of
online learning and Blackwell approachability.

\paragraph{Why the minimum cannot move first.}
The identity does not imply $\min_x\max_y\max_W f=0$.
For example, consider the matrix tests
\[
 \mathcal X=\Delta_2,\qquad \cY=[-1,1],\qquad
 \mathcal W=\{(-a,b):0\le a,b\le1\},\qquad f(W,x,y)=yWx.
\]
For $y\ge0$ the action $e_1$ satisfies every test, and for $y\le0$
the action $e_2$ does. Every test also has a zero in $\Delta_2$:
use $x=(b,a)/(a+b)$ when $a+b>0$. Both max--min--max values are zero,
whereas
\[
 \min_{x\in\Delta_2}\max_{y\in[-1,1]}\max_{W\in\mathcal W}yWx
 =\min_{x\in\Delta_2}\max\{x_1,x_2\}=\frac12.
\]
There are two tempting but invalid uses of minimax here. First, one
might try to exchange the outer maximum and minimum:
\[
 \max_y\min_x\Bigl(\max_W f(W,x,y)\Bigr)
 \overset{?}{=}\min_x\max_y\Bigl(\max_W f(W,x,y)\Bigr).
\]
The bracketed function is convex in $y$, as a maximum of affine
functions. The convex--concave minimax theorem requires concavity
in the maximized variable, so it does not justify this exchange.
In the example above, the bracketed function at $x=(1/2,1/2)$ is $|y|/2$.

Second, in the trilinear case, one might combine the two maximizing
choices into a tensor:
\[
 \cY\otimes\mathcal W
 =\{y\otimes W:y\in\cY,\ W\in\mathcal W\}.
\]
The tensor records all the products of $y$ and $W$ needed to evaluate
$f$, so $f(W,x,y)=g(x,y\otimes W)$ for a bilinear function $g$.
Starting from the max--max--min expression, the tempting exchange is
\[
 \max_y\max_W\min_x f(W,x,y)
 =\max_{z\in\cY\otimes\mathcal W}\min_x g(x,z)
 \overset{?}{=}\min_x\max_{z\in\cY\otimes\mathcal W}g(x,z).
\]
But the set $\cY\otimes\mathcal W$ of feasible tensors need not be
convex, so the minimax theorem does not justify this exchange either.

\paragraph{Moving first sequentially.}
The learner's alternative is to adapt its actions across rounds.
The sequential value is
\[
 V_T=\min_{x_0}\max_{y_0}\cdots
     \min_{x_{T-1}}\max_{y_{T-1}}
     \max_{W\in\mathcal W}\frac1T\sum_{t<T}f(W,x_t,y_t).
\]
Each action can depend on the preceding outcomes, while the test is
chosen in hindsight after the entire transcript. For the distance
tests, the terminal objective is exactly the approachability loss.
Blackwell's theorem gives $V_T\to0$ under response satisfiability,
even when the one-round value is positive. Repetition therefore
permits the learner to move first in an asymptotic sense that a
one-shot minimax exchange cannot justify.

\Needspace{6\baselineskip}
We now develop the two principal response-based routes foreshadowed
after \eqref{eq:app-response-condition}:
\begin{itemize}
 \item Choose a test using regret minimization over $\mathcal W$,
       then play a response that satisfies it against every outcome.
       Regret compares the chosen tests with the best test in hindsight,
       controlling the terminal maximum.
 \item Predict the outcome using calibration over $\cY$,
       then play a convex mixture of the corresponding outcome responses.
       Calibration controls the discrepancy between forecast payoffs
       in the target and realized payoffs.
\end{itemize}
The first route uses \eqref{eq:app-response-condition}; the second
uses the original response map $b$. We develop them in this order.

\subsection{Regret over Tests and Transformation Regret}
\label{app:test-regret}

Run an online linear optimization algorithm over the compact convex
test set $\mathcal W$, with the bounded payoffs from the preceding
formulations.
On round $t$, it selects $W_t$; we choose $x_t$ satisfying
\eqref{eq:app-response-condition}, observe $y_t$, and give the
subroutine the linear reward $W\mapsto f(W,x_t,y_t)$. Its regret
bound yields, when the tests represent the distance to the chosen target,
\begin{equation}
 \begin{aligned}
 T\,\APPLOSS(x_{0:T-1},y_{0:T-1})
 &=\max_{W\in\mathcal W}\sum_{t<T}f(W,x_t,y_t)\\
 &\le\sum_{t<T}f(W_t,x_t,y_t)+\operatorname{Reg}_{\mathcal W}(T)\\
 &\le\operatorname{Reg}_{\mathcal W}(T).
 \end{aligned}
 \label{eq:app-regret-reduction}
\end{equation}
The equality uses the distance representation; for a general test
family, the same argument bounds its maximum cumulative violation.
Thus sublinear regret gives approachability. This is the reduction
from online linear optimization to approachability of Abernethy,
Bartlett, and Hazan~\cite{abernethy2011blackwell}.

For the transformation family of Appendix~\ref{app:fixed-point-responses},
the objective is
\[
 \max_{\phi\in\Phi}\sum_{t<T}
       \langle y_t,x_t-\phi(x_t)\rangle,
\]
namely $\Phi$-regret for linear losses, including affine transformations.
The response step finds a fixed point $x_t=\phi_t(x_t)$, whose
existence is the condition \eqref{eq:app-phi-fixed-points}.
Its current-round regret is zero, so the same calculation bounds
$\Phi$-regret by external regret over $\Phi$.
This is the fixed-point reduction of
Gordon, Greenwald, and Marks~\cite[Theorem~1]{gordon2008noregret};
for transformations that may leave the action set, see
Dann, Mansour, Mohri, Schneider, and Sivan~\cite{dann2025rate}. See also Stoltz and
Lugosi~\cite[Section~3]{stoltz2005internal}.

\paragraph{Recovering Blum--Mansour.}
Section~\ref{sec:app-response} already identifies the experts
swap-regret loss with a $\mathcal K$-distance. Its tests are
\[
 W=\tfrac12(I-P),\qquad P\ge0,\qquad
 \mathbf1^{\mathsf T}P=\mathbf1^{\mathsf T}.
\]
Responding to a test means choosing a stationary distribution
$x=Px$, whose existence follows from the response-duality argument
in Appendix~\ref{app:stationary-response}.
Run one external-regret algorithm per column to construct
$P_t$, play $x_t=P_tx_t$, and give column learner $i$ the weighted
loss vector $x_{t,i}y_t$. Summing their external regrets bounds
the swap regret. This is exactly the Blum--Mansour
algorithm~\cite{blum2007external}.

\paragraph{Endomorphisms and TreeSwap.}
A stronger condition is that every comparator is an endomorphism:
$\phi(\mathcal X)\subseteq\mathcal X$. The full class of affine
endomorphisms of a compact convex set is convex, and each has a
fixed point by Brouwer's theorem. The convex hull of any chosen
subfamily still consists of endomorphisms. Thus proper transformations
automatically provide the convexity and fixed points above, while
also making every transformed action feasible. TreeSwap uses this
additional property to compare external-regret subroutines with
transformed actions~\cite{dagan2024external,peng2024fast}.
Having a feasible fixed point alone does not provide that comparison.

Some transformations that leave $\mathcal X$ nevertheless reduce to
endomorphisms by one common change of their linear or affine
residuals $F$. If an invertible matrix $A$ on the action space makes
every $\psi_F(x)=x-AF(x)$ an endomorphism, then with $z=A^{-\mathsf T}y$,
\begin{equation}
 z^{\mathsf T}(x-\psi_F(x))=y^{\mathsf T}F(x).
 \label{eq:app-proper-change}
\end{equation}
Proper transformation-regret algorithms therefore transfer exactly,
with outcome set $A^{-\mathsf T}\cY$; their rates must be evaluated
in that transformed geometry. This is a rate-preserving linear
equivalence of regret objectives, in the terminology of
Dann, Mansour, Mohri, Schneider, and Sivan~\cite{dann2025rate}.

For example, let $P$ range over column-stochastic matrices and take
$\phi_P=2I-P$ on the simplex. These maps preserve its affine
hyperplane and have stationary distributions as fixed points,
but can send points outside the simplex. Here $W_P=P-I$ and
$A=-I$ give $\psi_P=P$ and $z=-y$. If $\cY=-\cY$, the outcome
geometry is unchanged, so TreeSwap gives the same rate.
There is also an inverse-map interpretation:
\[
 (2I-P)^{-1}=\sum_{k=0}^{\infty}2^{-k-1}P^k
\]
is column-stochastic and has the same fixed points. This explains
why inverse steps remain feasible. The inverse depends on $P$,
however; the common sign change in \eqref{eq:app-proper-change}
is what preserves the regret objective. The next example shows
that a common invertible change making every map proper need not exist.

\subsection{Calibrated Responses and a Skew-Symmetric Family}
\label{app:calibrated-responses}

The second route uses response satisfiability directly in the outcome
variable. Given calibrated forecasts $\mu_t$ on $\cY$, play
\[
 x_t=\sum_p\mu_t(p)b(p).
\]
Each $p\otimes b(p)$ lies in the target, and calibration controls
the difference between these forecast payoffs and the realized
payoffs~\cite[Section~4.4]{perchet2014approachability}.
Section~\ref{sec:app-two-routes} proves
$\APPLOSS\le\CAL_{\cY,T}/T$, and
Section~\ref{sec:app-harmonic} replaces the calibrator by the explicit
harmonic response algorithm. This route needs an outcome response
$b(y)$, without an endomorphism representation of the tests.

\paragraph{Harmonic bounds for supporting tests.}
\label{app:harmonic-tests}

Under the hypotheses of Theorem~\ref{thm:app-direct}, let
$\mathcal T\subseteq\mathcal K^\circ$ be a nonempty compact family such
that $y^{\mathsf T}Wb(y)\le0$ for every $y\in\cY$ and
$W\in\mathcal T$. Then
\begin{equation}
 \max_{W\in\mathcal T}\frac1T\sum_{t<T}y_t^{\mathsf T}Wx_t\le O\!\left(\frac{\psi(\cY,\mathcal X)}{\log T}\right).
 \label{eq:app-direct-tests}
\end{equation}
Indeed, the certificate $s_T$ in the proof of the theorem satisfies
$\langle W,s_T\rangle\le0$. Hence
\[
 \langle W,\bar\kappa_T\rangle
 \le\langle W,\bar\kappa_T-s_T\rangle
 \le\|\bar\kappa_T-s_T\|_{\mathcal K}\le O\!\left(\frac{\psi(\cY,\mathcal X)}{\log T}\right).
\]
For affine tests $W=(U,a)\in\mathcal W$, the same certificate gives
\[
 \frac1T\sum_{t<T}f(W,x_t,y_t)
 =W(\bar\kappa_T)
 \le\langle U,\bar\kappa_T-s_T\rangle
 \le O\!\left(\frac{\psi(\cY,\mathcal X)}{\log T}\right),
\]
since $W(s_T)\le0$ and $U\in\mathcal K^\circ$.
One can also restrict the matrix discrepancy to the chosen test family.
Adjoin the zero test if needed and replace the tensor norm by
$\sigma_{\mathcal T}(C)=\sup_{W\in\mathcal T}\langle W,C\rangle$
and maximize over the actual feasible columns of $Y,X$.
The same proof gives
$O((\sup_{Y,X}\sigma_{\mathcal T}(YK_TX^{\mathsf T})/T+1)/h_T)$.
This bounds test violation; it bounds distance when the tests represent
that distance, as in \eqref{eq:app-support-distance}.

\paragraph{An affine skew-symmetric response family.}
\label{app:skew-response}

This example uses an affine lift to obtain a nontrivial response
problem with a dimension-independent harmonic bound. Each test has a
feasible zero, but no action satisfies all tests against every outcome.

\begin{proposition}[Affine skew-symmetric responses]
\label{prop:affine-skew}
Let $\mathcal X=\cY=B_2^d$, let $e=e_1$, and set
\[
 \mathcal W_{\mathrm{skew}}
 =\{(\alpha,S):\alpha\ge0,\ S^{\mathsf T}=-S,
                 \ \alpha^2+\|S\|_{2\to2}^2\le1\}.
\]
For $(\alpha,S)\in\mathcal W_{\mathrm{skew}}$, define
\[
 F_{\alpha,S}(x)=\alpha x+S(x-e),\qquad
 f((\alpha,S),x,y)=\tfrac12\langle y,F_{\alpha,S}(x)\rangle.
\]
The test set is compact and convex, the payoff is continuous and
separately affine, and $|f|\le1$.  Every $F_{\alpha,S}$ has a zero in
$B_2^d$.  Moreover, the fixed outcome response
\[
 b(y)=
 \begin{cases}
 e-\dfrac{\langle e,y\rangle}{\|y\|_2^2}y,&y\ne0,\\[4pt]
 0,&y=0,
 \end{cases}
\]
belongs to $B_2^d$ and satisfies
$f((\alpha,S),b(y),y)=0$ for every test and outcome.
For $d\ge2$, however,
\begin{equation}
 \min_{x\in B_2^d}\max_{(\alpha,S)\in\mathcal W_{\mathrm{skew}}}
             \max_{y\in B_2^d} f((\alpha,S),x,y)=\frac14.
 \label{eq:skew-common-action-gap}
\end{equation}
In particular there is no action that is a common zero of the family.
\end{proposition}
\begin{proof}
The function $(\alpha,S)\mapsto
(\alpha^2+\|S\|_{2\to2}^2)^{1/2}$ is a norm on the product space.
Intersecting its unit ball with the closed half-space $\alpha\ge0$
and the skew-symmetric subspace proves compactness and convexity.
Separate affinity is immediate from the displayed payoff.
Skew symmetry gives, for every $v$,
\[
 \|(\alpha I+S)v\|_2^2
   =\alpha^2\|v\|_2^2+\|Sv\|_2^2,
 \qquad
 \|\alpha I+S\|_{2\to2}^2
   =\alpha^2+\|S\|_{2\to2}^2\le1.
\]
Consequently $\|F_{\alpha,S}(x)\|_2\le
\|\alpha I+S\|_{2\to2}\|x\|_2+\|Se\|_2\le2$, proving
the normalization.

If $\alpha>0$, the matrix $\alpha I+S$ is invertible: a vector in
its kernel satisfies $\alpha\|v\|_2^2=0$ after taking the inner
product with $v$.  Thus
$x=(\alpha I+S)^{-1}Se$ is a zero.  Taking the inner product of
$\alpha x+S(x-e)=0$ with $x-e$ yields
$\|x\|_2^2=\langle x,e\rangle\le\|x\|_2$, hence $x\in B_2^d$.
If $\alpha=0$, use $x=e$.

For $y\ne0$, $b(y)$ is the orthogonal projection of $e$ onto
$y^\perp$, so $\|b(y)\|_2\le1$ and $\langle y,b(y)\rangle=0$.
Since $b(y)-e$ is a scalar multiple of $y$ and
$\langle y,Sy\rangle=0$, the simultaneous response identity follows.
It is automatic at $y=0$.

For the last assertion, choose $J$ with $Je_1=e_2$, $Je_2=-e_1$,
and $Je_i=0$ for $i>2$.  The tests $(1,0)$ and $(0,J)$ imply that
the maximum in \eqref{eq:skew-common-action-gap}, at any $x$, is at least
\[
 \frac12\max\{\|x\|_2,\|J(x-e)\|_2\}
 \ge\frac14\bigl(\|Jx\|_2+\|J(x-e)\|_2\bigr)
 \ge\frac14.
\]
At $x=e/2$, orthogonality of $e$ and $Se$ gives
\[
 \|F_{\alpha,S}(e/2)\|_2
  =\tfrac12(\alpha^2+\|Se\|_2^2)^{1/2}\le\tfrac12,
\]
which proves the matching upper bound.
\end{proof}

The zeros are precisely the fixed points of the affine transformations
$\phi_{\alpha,S}(x)=x-F_{\alpha,S}(x)$. For $d\ge2$, these maps need
not be endomorphisms of $B_2^d$: with $J$ as in the proof,
$\phi_{0,J}(-e)=-e+2e_2\notin B_2^d$.
There is a stronger obstruction. No common invertible matrix $A$
makes every map $x\mapsto x-AF_{\alpha,S}(x)$ an endomorphism.
Indeed, at the extreme point $x=-e$,
\[
 F_{0,J}(-e)=-2e_2,\qquad F_{0,-J}(-e)=2e_2.
\]
The two proposed images are $-e+2Ae_2$ and $-e-2Ae_2$.
They are distinct and have midpoint $-e$, so they cannot both
belong to the ball. This is the opposite-displacements obstruction
used in the study of improper $\Phi$-regret and rate-preserving
linear reductions~\cite[Section~4.4]{dann2025rate}.
It rules out the common invertible change in
\eqref{eq:app-proper-change} for this family; it does not assert
an impossibility for every conceivable reduction.

Calibrated responses still apply immediately. Forecast the outcomes
in the Euclidean ball and respond with the displayed projection
$b(p)$. Its simultaneous response identity holds for the whole
family, so harmonic replay gives dimension-independent test violation.
We spell out the affine coordinates to apply the tensor bound.

To apply the bilinear tensor bound, use the affine lift
\begin{equation}
 \widetilde{\mathcal X}
 =\{(x,1)/\sqrt2:x\in B_2^d\}\subset B_2^{d+1},\qquad
 \widetilde W_{\alpha,S}
 =\frac1{\sqrt2}[\,\alpha I+S\ \ -Se\,].
 \label{eq:skew-affine-lift}
\end{equation}
This feasible action set is a convex affine slice and excludes zero.
The lifted payoff agrees exactly
with the original normalization:
\[
 y^{\mathsf T}\widetilde W_{\alpha,S}\frac{(x,1)}{\sqrt2}
 =f((\alpha,S),x,y),\qquad
 \|\widetilde W_{\alpha,S}\|_{2\to2}
 \le\frac{(\|\alpha I+S\|_{2\to2}^2+\|Se\|_2^2)^{1/2}}{\sqrt2}
 \le1.
\]
The lifted response is $\widetilde b(y)=(b(y),1)/\sqrt2$.
It extends to all $y\in\R^d$, is scale-invariant, and satisfies
$y^{\mathsf T}\widetilde W_{\alpha,S}\widetilde b(y)=0$ for every test.
Thus the construction in Appendix~\ref{app:response-cones} applies,
retaining the affine column $-Se$ in \eqref{eq:skew-affine-lift}.
Because harmonic replay uses convex weights, replaying these lifted
responses is exactly the lift of replaying the original responses.
Since $\cY=B_2^d$, Lemma~\ref{prop:app-tensor-geometry} bounds the
joint coefficient by an absolute constant. The direct harmonic estimate
gives, when initialized at $p_{\mathrm{init}}=0$, for every adaptive
outcome transcript and $T\ge2$,
\begin{equation}
 \max_{(\alpha,S)\in\mathcal W_{\mathrm{skew}}}
 \frac1T\sum_{t<T}f((\alpha,S),x_t,y_t)
 \le O\!\left(\frac1{\log T}\right).
 \label{eq:skew-harmonic-bound}
\end{equation}
Indeed, each lifted graph point
$y\widetilde b(y)^{\mathsf T}$ pairs to zero with every
$\widetilde W_{\alpha,S}$, and these tests have dual norm at most one.
Thus the response-graph distance bound of
Theorem~\ref{thm:app-direct} implies
\eqref{eq:skew-harmonic-bound}.

\subsection{Response Partitions and Orthant Calibration}
\label{sec:app-orthants}

We next organize outcome responses by a finite collection of labels.
General response partitions use the bounded response target; the
orthant instance below also has the scale-invariant cone structure
of Appendix~\ref{app:response-cones}.
Let $\mathcal X=\Delta_N$, cover the outcome set by closed
convex cells $D_1,\ldots,D_N$, and respond to $y\in D_i$ with $e_i$,
breaking ties consistently. The convex hull of all permitted response
pairs is an approachable target. With a constant outcome coordinate
when needed, each column of the averaged tensor records the mass
assigned to a label and its weighted outcome sum. Membership in this
target says that every positive-mass label has its conditional mean
in the corresponding cell. A fixed tie-breaking response gives a
contained target and therefore suffices for this requirement.

One natural source of such partitions is a convex piecewise-affine
potential
\[
 G(y)=\max_{i\le N}(a_i^{\mathsf T}y+c_i),\qquad
 D_i=\{y\in\cY:a_i^{\mathsf T}y+c_i\ge a_j^{\mathsf T}y+c_j\ \forall j\}.
\]
These cells are polyhedra intersected with $\cY$, and are polytopes
when $\cY$ is a polytope. They also define a weakly proper scoring
rule for the mean: for an active index $i(p)$ at report $p$, set
\[
 S(p,y)=G(p)+a_{i(p)}^{\mathsf T}(y-p).
\]
If an outcome distribution has mean $\mu$, then
$\E S(p,Y)\le G(\mu)=\E S(\mu,Y)$, by the supporting-plane inequality.
Thus truthfully reporting the mean maximizes expected score.
This construction gives a broad class of response partitions;
an arbitrary convex-cell partition need not arise from such a potential.

\paragraph{The orthant instance.}
Taking $G(y)=\|y\|_1=\max_{o\in\{-1,+1\}^d}\langle o,y\rangle$
gives the sign response and the orthant cells. We now compare the
two algorithmic routes for this instance.

Let $\mathcal O=\{-1,+1\}^d$.  On each round, an orthant calibrator
announces $\pi_t\in\Delta(\mathcal O)$ and observes
$y_t\in[-1,1]^d$.  For each label $o$, define the orthant
and its cumulative outcome vector by
\[
 C_o=\{v\in\mathbb R^d:o_iv_i\ge0\ (1\le i\le d)\},\qquad
 V_o=\sum_t\pi_t(o)y_t.
\]
Its unsquared loss is
\begin{equation}
 \operatorname{OrthCal}_T
 =\sum_{o\in\mathcal O}\operatorname{dist}_\infty(V_o,C_o)
 =\sum_{o\in\mathcal O}\max_{1\le i\le d}(-o_iV_{o,i})_+.
 \label{eq:orthant-loss}
\end{equation}
This uses $d$ outcome coordinates; the number of possible labels is
$2^d$.  When a label receives mass $w_o>0$, its summand is
$w_o\operatorname{dist}_\infty(V_o/w_o,C_o)$, by homogeneity of
distance to a cone.  Labels receiving no mass contribute zero.

Here is its exact payoff representation.  Take
\[
 \mathcal X=\Delta(\mathcal O),\qquad
 Z=(\mathbb R^d)^{\mathcal O},\qquad
 \|z\|_Z=\sum_o\|z_o\|_\infty,\qquad
 \kappa(a,y)_o=a(o)y,
 \qquad \mathcal S=\prod_o C_o.
\]
The target is convex.  Minimizing independently in each
column gives
\[
 \operatorname{dist}_Z\left((1/T)\sum_t\kappa(\pi_t,y_t),\mathcal S\right)
 =\frac{\operatorname{OrthCal}_T}{T}.
\]
Fix $o(y)_i=+1$ when $y_i\ge0$ and $o(y)_i=-1$ otherwise, and choose
the sign response $b(y)=e_{o(y)}$. Then $y\in C_{o(y)}$, so
$\kappa(b(y),y)\in \mathcal S$. The norm on $Z$ is the tensor norm for
$\cY=[-1,1]^d$ and $\mathcal X=\Delta(\mathcal O)$.
Extend the sign response to all of $\R^d$. It is scale-invariant,
and the closed convex hull of its response tensors over $\R^d$ is
exactly $\mathcal S=\prod_o C_o$, as in Appendix~\ref{app:response-cones}.
Indeed, each generator has one column in its prescribed orthant;
conversely, taking limits of vectors with sign $o$ fills column cone
$C_o$, and summing over labels gives the product.
The bounded response hull $\conv\{y\otimes b(y):y\in\cY\}$ is
contained in this cone, so Lemma~\ref{lem:calibration-approachability}
bounds the distance to $\mathcal S$ by $\CAL_{\infty,T}/T$.

\paragraph{Supporting tests.}\label{app:orthant-tests}
The support-function description is explicit. Let
\[
 K_o=\{u\in\mathbb R^d:o_iu_i\le0\ (1\le i\le d),\ \|u\|_1\le1\}.
\]
H\"older's inequality, and a signed coordinate vector at a maximally
violated coordinate, give
$\operatorname{dist}_\infty(v,C_o)=\max_{u\in K_o}u^{\mathsf T}v$.
For matrices $W$ having columns $u_o\in K_o$,
\[
 \operatorname{OrthCal}_T
 =\max_{W\in\prod_o K_o}\sum_ty_t^{\mathsf T}W\pi_t,
 \qquad y^{\mathsf T}Wb(y)\le0.
\]
Thus the bounded matrix tests in Appendix~\ref{app:response-cones} are
$\mathcal U\cap\mathcal K^\circ=\prod_o K_o$. Their description is exponentially
large, but the sign response is immediate.

\paragraph{The test-regret algorithm.}
The reduction of Abernethy, Bartlett, and Hazan has a concrete form
here. For each orthant $o$, run an external-regret learner over the
$d+1$ vertices of
$K_o=\conv\{0,-o_1e_1,\ldots,-o_de_d\}$.
On round $t$, assemble their chosen columns into $W_t$ and solve
\[
 \pi_t\in\Delta(\mathcal O),\qquad W_t\pi_t=0.
\]
The sign response and max--min--max identity guarantee feasibility.
After observing $y_t$, give learner $o$ the reward
$u\mapsto\pi_t(o)\langle u,y_t\rangle$.
As in Blum--Mansour, separate weighted regret learners are coupled
by a response computation; here that computation is a linear
feasibility problem. Exponential weights with a common learning rate
$\eta>0$ gives
\[
 \operatorname{OrthCal}_T
 \le\frac{2^d\log(d+1)}{\eta}
       +O\!\left(\eta\sum_{t<T}\sum_o\pi_t(o)^2\right)
 \le\frac{2^d\log(d+1)}{\eta}+O(\eta T).
\]
Optimizing $\eta$ yields average error
$O(\sqrt{2^d\log(d+1)/T})$. This is a guarantee for this explicit
implementation of the reduction, rather than a lower bound for all
test-regret methods.

The orthant example itself also has a proper transformation
representation. Let $A$ have columns $-o$, and let $P$ be
column-stochastic with column $o$ supported on $o$ and its
single-coordinate sign flips. Then
$\mathcal W=\{\tfrac12 A(I-P):P\text{ as above}\}$:
the column vertices are exactly $0,-o_1e_1,\ldots,-o_de_d$.
Thus the loss is half the corresponding restricted proper
transformation regret for outcome vectors $A^{\mathsf T}y$.
This example illustrates the response construction and its rates;
the preceding skew family supplies the obstruction to making a
family proper by a common invertible residual transformation.

\paragraph{The harmonic-response algorithm.}
Instead of maintaining a learner for every orthant and solving the
coupled response problem, harmonic approachability computes the sign
of each new smoothed forecast $p_s$ and replays these past labels with
harmonic weights.

\begin{corollary}[Harmonic orthant calibration]
\label{cor:orthant}
Push the harmonic cube forecasts forward by $y\mapsto o(y)$.
After initialization, this gives
$\pi_t=\sum_{s<t}[h_t(2(t-s)-1)]^{-1}\delta_{o(p_s)}$ for $t\ge1$.
For every adaptive outcome transcript and $T\ge2$,
\begin{equation}
 \frac{\operatorname{OrthCal}_T}{T}
 \le\frac{\CAL_{\infty,T}}T
 \le\frac{C\log(2d)}{h_T}
 \label{eq:harmonic-orthant-bound}
\end{equation}
for an absolute constant $C$. Consequently, for $d\ge2$ and
$0<\varepsilon\le1$, a sufficient horizon is $T\ge d^{C'/\varepsilon}$
for an absolute constant $C'$.
\end{corollary}
\begin{proof}
The first inequality follows from the payoff representation and
Lemma~\ref{lem:calibration-approachability}. The second
inequality follows from Theorem~\ref{thm:matrix-calibration},
Lemma~\ref{lem:harmonic-kernel-comparison}, and the cube estimate in
Section~\ref{sec:geometry}. Finally, $2h_T\ge\log T$ gives the stated
horizon.
\end{proof}

Only realized sign labels need be stored: on round $t\ge1$ at most
$\min\{t,2^d\}$ orthants have positive announced mass, and at most
$\min\{T,2^d\}$ distinct labels occur through horizon $T$. Thus neither
the response nor the representation enumerates the $2^d$ orthants.
Appendix~\ref{app:sampling} gives the implementation details.
The sufficient horizon is polynomial
in $d$ at each fixed accuracy; its exponent depends on accuracy.

The action geometry explains why Lemma~\ref{cor:cube-cube} does
not remove $\log(2d)$ here. The symmetric convex hull of the label
simplex is $B_1^{2^d}$, and
$\|z\|_Z=\sum_o\|z_o\|_\infty$ is the projective norm on
$\ell_\infty^d\otimes\ell_1^{2^d}$. Replacing the label simplex by a
cube would change this tensor norm
and would no longer bound the same orthant-calibration loss.

The sufficient horizon for the explicit test-regret implementation is
$O(2^d\log(d+1)/\varepsilon^2)$, while the harmonic guarantee is
$d^{O(1/\varepsilon)}$ for $d\ge2$. At fixed accuracy, the latter is
polynomial in the outcome dimension despite the $2^d$ possible labels.
Its accuracy dependence is worse when the dimension is fixed; the
comparison concerns high-dimensional forecasting.

\subsection{General Games and Profile Swap Regret}
\label{app:profile-swap}

\paragraph{Swap regret as a response target.}
Consider $m$ learner actions, $d$ opponent actions, and a loss matrix
$A\in\R^{m\times d}$. Let $x\in\Delta_m$, $y\in\Delta_d$, so the
learner's loss is $x^{\mathsf T}Ay$. For
$C=(1/T)\sum_{t<T}y_tx_t^{\mathsf T}$, the average gain from a swap
$\phi:[m]\to[m]$ is
\[
 (1/T)\sum_{t<T}\sum_{i=1}^m
 x_{t,i}\bigl((Ay_t)_i-(Ay_t)_{\phi(i)}\bigr)
 =\sum_{i=1}^m\bigl((AC)_{ii}-(AC)_{\phi(i),i}\bigr).
\]
Allowing swaps that change only one action shows that the target for
nonpositive swap regret is
\[
 \mathcal S_{\mathrm{swap}}
 =\{C\in\mathcal P:(AC)_{ii}\le(AC)_{ji}\text{ for all }i,j\},
 \qquad
 \mathcal P=\conv\{yx^{\mathsf T}:y\in\Delta_d,x\in\Delta_m\}.
\]
Here $\mathcal P$ is the set of nonnegative $d\times m$ matrices whose
entries sum to one. Define
$\operatorname{BR}(y)=\arg\min_{x\in\Delta_m}x^{\mathsf T}Ay$.
Then
\begin{equation}
 \mathcal S_{\mathrm{swap}}
 =\conv\{yx^{\mathsf T}:y\in\Delta_d,\ x\in\operatorname{BR}(y)\}.
 \label{eq:app-swap-target}
\end{equation}
To see this, let $p_i$ be the mass of column $i$ of $C$.
When $p_i>0$, write that column as $p_i y^{(i)}$ with
$y^{(i)}\in\Delta_d$. The swap inequalities say exactly that
$e_i\in\operatorname{BR}(y^{(i)})$, and
$C=\sum_{i:p_i>0}p_i y^{(i)}e_i^{\mathsf T}$.
Conversely, a best response puts weight only on loss-minimizing
actions, so each generating tensor satisfies all the swap inequalities.
Thus the familiar swap-regret constraints describe the convex hull
of best-response pairs.

\paragraph{Profile swap regret.}
We include affine coordinates when needed, as in
Appendix~\ref{app:response-universality}, so the tensor profile also
records the average action and outcome.
For general polytopes $\mathcal X,\cY$ and a bilinear learner loss
$\ell(x,y)$, put
\[
 r(x,y)=\ell(x,y)-\min_{x'\in\mathcal X}\ell(x',y),\qquad
 \operatorname{BR}(y)=\arg\min_{x\in\mathcal X}\ell(x,y).
\]
Profile swap regret minimizes the average response loss over all
decompositions of the observed tensor profile
\cite{arunachaleswaran2025profile}:
\[
 \frac{\operatorname{ProfileSwapReg}_T(C)}{T}
 =\min_{\substack{C=\sum_k\lambda_k y_kx_k^{\mathsf T}\\
                  \lambda_k\ge0,\ \sum_k\lambda_k=1}}
        \sum_k\lambda_k r(x_k,y_k),\qquad
 x_k\in\mathcal X,\ y_k\in\cY.
\]
Because $r\ge0$, its zero set is exactly
\begin{equation}
 \mathcal S_{\mathrm{BR}}
 =\conv\{yx^{\mathsf T}:y\in\cY,\ x\in\operatorname{BR}(y)\}.
 \label{eq:app-profile-target}
\end{equation}
This is the smallest convex target containing \emph{all}
best-response pairs \cite[Lemma 9]{arunachaleswaran2025profile}.
Profile swap distance is precisely distance to this set (Euclidean
distance in that paper). Thus profile swap regret gives the minimal
response-target construction a direct game-theoretic interpretation.
For any fixed selection $b(y)\in\operatorname{BR}(y)$, our target
$\mathcal S_b$ is contained in $\mathcal S_{\mathrm{BR}}$;
the inclusion can be strict when best responses are not unique.
Approaching $\mathcal S_b$ therefore also controls profile swap distance.
Loss differences from linear deviations, including constant deviations
for external regret in these affine coordinates, are linear functionals
of the tensor profile and are nonpositive on every best-response tensor.
Thus vanishing distance to $\mathcal S_{\mathrm{BR}}$ implies vanishing
average linear swap regret and external regret.
Recent response-based algorithms use this connection to study swap
regret \cite{anagnostides2026swap}.

\subsection{Tensor Norms and Affine Coordinates}
\label{app:tensor-duality}
\label{app:affine-response-geometry}

All norms and polar tests are taken on the span of the feasible tensor
payoffs. The signed atomic hull $\mathcal K$ from
Section~\ref{sec:app-response} has the dual description
\begin{equation}
 \|W\|_{\mathcal K,*}
 =\sup_{y\in\cY,\,x\in\mathcal X}|y^{\mathsf T}Wx|,
 \qquad
 \|C\|_{\mathcal K}
 =\sup_{\|W\|_{\mathcal K,*}\le1}\langle W,C\rangle.
 \label{eq:app-projective-norm}
\end{equation}
Expanding products of convex combinations shows that $\mathcal K$ is
also the tensor unit ball generated by the symmetric convex hulls of
$\cY$ and $\mathcal X$. In particular,
\begin{equation}\label{eq:app-rank-one}
 \|uv^{\mathsf T}\|_{\mathcal K}
 =\|u\|_{\cY}\|v\|_{\mathcal X}.
\end{equation}
For the upper bound, normalize the factors and expand their atomic
representations. For the reverse bound, choose norming dual functionals
$a,c$ for $u,v$. The matrix $ac^{\mathsf T}$ has dual tensor norm at
most one and pairs with $uv^{\mathsf T}$ to the product of their norms.
Zero factors are immediate.

If a payoff is linear in $y$ and affine in $x$, only the action needs
lifting: $y^{\mathsf T}(Mx+c)=y^{\mathsf T}[\,M\ \ c\,](x,1)$.
A lift can be omitted when the existing coordinates already have a
linear functional identically one, as on a simplex. For symmetric
norm balls,
\[
 \conv\{\pm(x,1):x\in\mathcal X\}=\mathcal X\times[-1,1].
\]
The feasible set remains the slice with last coordinate one.
Thus cube lifts remain cubes. If the outcome set is a Euclidean ball,
Lemma~\ref{prop:app-tensor-geometry} applies to any lifted action set.
If both sets are Euclidean balls that require lifts, the inclusions
\[
 B_2^{r+1}\subset B_2^r\times[-1,1]\subset\sqrt2 B_2^{r+1}
\]
give the same dimension-independent estimate with an absolute factor:
a normalized test has Euclidean operator norm at most one, while both
lifted vectors have Euclidean norm at most $\sqrt2$.

These affine coordinates can be essential. For the payoff $y\otimes x$,
$0\in\cY$ implies $0\in\mathcal S$; if also $0\in\mathcal X$, always
playing zero solves the response-graph problem. The action lift keeps
the outcome-dependent column $y$ in $y(x,1)^{\mathsf T}$ and removes
this automatic solution. The skew-symmetric example above uses it.

\section{Harmonic Normalization and Pulse Estimates}
\label{app:harmonic-optimality}

We give the normalization calculations and analytic estimates used in
Sections~\ref{sec:harmonic} and~\ref{sec:harmonic-optimality}.
Throughout, $\mathcal A,\mathcal B\subseteq\R^d$ are nonempty compact convex sets,
$\mathcal B=-\mathcal B$, and
$r=\max_{a\in\mathcal A,\,b\in\mathcal B}\langle a,b\rangle$.
Subsection~\ref{app:harmonic-pulses} also assumes $0\in\mathcal A$,
as in Section~\ref{sec:harmonic-optimality}.

\subsection{Mass Identities and Normalization Corrections}
\label{app:harmonic-normalization}

Set $h_0=0$, the empty-sum convention, for the identities below.
The positive lower row sums of $K_T$ are $h_t$, and its
negative upper column sums have magnitudes $h_{s+1}$. Interchanging
finite sums gives
\begin{equation}\label{eq:harmonic-sum-identities}
 \sum_{j=0}^{T-1}h_j=(T-\tfrac12)h_T-T/2,\qquad
 \sum_{j=1}^{T}h_j=(T+\tfrac12)h_T-T/2.
\end{equation}
For example, the first sum is
$\sum_{k=0}^{T-1}(T-1-k)/(2k+1)$; the second differs by $h_T$.
Adding the identities proves
$\|K_T\|_1=T(2h_T-1)$, and hence $K_T\in\mathcal H_T$.
Normalizing each positive lower row and each negative upper column
separately gives $\widehat K_T$ in
\eqref{eq:harmonic-algorithm-kernel}. There are $T-1$ such rows
and $T$ such columns, so its total absolute mass is $2T-1$.
The same mass identity holds for every member of
$\widehat{\mathcal H}_T$ and proves the scaled inclusion in
\eqref{eq:optimality-feasible-inclusion}.
We also use
\[
 2h_T-1=1+\sum_{j=1}^{T-1}\frac2{2j+1}
 \ge\sum_{j=1}^T\frac1j\ge\log T.
\]

For matrices of the same size, the bound $|\langle a,b\rangle|\le r$
and central symmetry of $\mathcal B$ give
\begin{equation}\label{eq:optimality-gamma-lipschitz}
 |\Gamma_{\mathcal A,\mathcal B}^{(T)}(A)-\Gamma_{\mathcal A,\mathcal B}^{(T)}(B)|
 \le\Gamma_{\mathcal A,\mathcal B}^{(T)}(A-B)\le r\sum_{t,s<T}|A_{ts}-B_{ts}|.
\end{equation}
This turns a bound on correction mass into a bound on the objective.

\begin{proposition}[Normalization correction]
\label{lem:optimality-normalization}\label{prop:harmonic-matrix}
For every $T\ge1$,
\begin{equation}\label{eq:optimality-two-sided-normalization}
 \left|h_T\Gamma_{\mathcal A,\mathcal B}^{(T)}(\widehat K_T)
       -\Gamma_{\mathcal A,\mathcal B}^{(T)}(K_T)\right|\le r(T-h_T).
\end{equation}
\end{proposition}
\begin{proof}
Every entry of $h_T\widehat K_T-K_T$ has the same sign as the
corresponding entry of $K_T$, because $h_t,h_{s+1}\le h_T$.
The total masses therefore give directly
\[
 \sum_{t,s<T}|h_T(\widehat K_T)_{ts}-(K_T)_{ts}|
 =h_T(2T-1)-T(2h_T-1)=T-h_T.
\]
Apply \eqref{eq:optimality-gamma-lipschitz}.
\end{proof}

\begin{proof}[Proof of Lemma~\ref{lem:harmonic-kernel-comparison}]
For $T=1$, both matrices equal $(-1)$ and $h_1=1$, so the claim is
immediate. Assume $T\ge2$.
Write $h=h_T$, $a=\Gamma_{\mathcal A,\mathcal B}^{(T)}(K_T)$, and
$b=\Gamma_{\mathcal A,\mathcal B}^{(T)}(\widehat K_T)$. The normalization correction in
Proposition~\ref{prop:harmonic-matrix} gives
\[
 |hb-a|\le r(T-h).
\]
To absorb the additive term, choose $v\in\mathcal A$ and $u\in\mathcal B$
with $\langle v,u\rangle=r$. Use
\[
 a_t=v,\qquad b_s=\frac{T-1-2s}{T-1}\,u.
\]
These right tests lie in the segment joining $u$ and $-u$.
Summing their objectives using \eqref{eq:harmonic-sum-identities} gives
\[
 a\ge r\frac{T^2-h}{2(T-1)}\ge\frac{r T}{2},\qquad
 hb\ge\frac{r h}{T-1}\sum_{t=1}^{T-1}\frac{t}{h_t}
       \ge\frac{r T}{2},
\]
where we used $h_t\le h\le T$.
It follows that $hb\le a+r(T-h)\le3a$ and
$a\le hb+r(T-h)\le3hb$. Hence
\[
 \frac13\,\Gamma_{\mathcal A,\mathcal B}^{(T)}(K_T)
 \le h_T\Gamma_{\mathcal A,\mathcal B}^{(T)}(\widehat K_T)
 \le3\,\Gamma_{\mathcal A,\mathcal B}^{(T)}(K_T),
\]
as claimed.
\end{proof}

\subsection{Recovering the Hilbert Objective from Pulses}
\label{app:harmonic-pulses}

\begin{proof}[Proof of Lemma~\ref{lem:optimality-pulses}]
Use the maximizing vectors and pulses from the main text, with
$w=1/2$ and $\Lambda=4T$. Set $J=[-w/2,\Lambda-w/2]$. The purpose of this calculation
is to recover the original maximizing sum
\[
 \sum_{t,s<T}\frac{\langle a_t,b_s\rangle}{2(t-s)-1}
 =\Gamma_{\mathcal A,\mathcal B}^{(T)}(K_T)
\]
from the integral $I$ that determines the randomized witnesses'
common signed correlation. We will justify the two approximations
\[
 I\ \approx\ \frac1{2\Lambda}\int_J\int_J
       \frac{\langle f(x),g(y)\rangle}{x-y}\,dx\,dy
 \ \approx\ \frac{w^2}{2\Lambda}
       \sum_{t,s<T}\frac{\langle a_t,b_s\rangle}{2(t-s)-1},
\]
with absolute errors at most $r/64$ and $3r/64$, respectively.
The second approximation is easy to interpret: on a row/column
pulse rectangle the dot product is the constant
$\langle a_t,b_s\rangle$, the rectangle has area $w^2$, and
$x-y$ is close to the center difference $2(t-s)-1$.
The first approximation accounts for the periodic copies of these
rectangles. This is why a periodic reciprocal kernel enters.

\paragraph{Account for periodic copies.}
Taking the odd part gives
\[
 \int_0^\infty\frac{\phi(v)}{v}\,dv
 =\frac12\lim_{R\to\infty}\int_{-R}^{R}\frac{\chi(v)}{v}\,dv.
\]
Substituting the definition of $\chi$ and setting $x=y+v$ turns the
weight $1/v$ into $1/(x-y)$. Because the functions are periodic,
all translated copies contribute as well. Their summed kernel is

\begin{equation}\label{eq:optimality-cotangent-integral}
 I=\frac1{2\Lambda}\int_J\int_J
       \langle f(x),g(y)\rangle k_\Lambda(x-y)\,dx\,dy,
 \qquad k_\Lambda(z)=\frac\pi \Lambda\cot\frac{\pi z}{\Lambda}.
\end{equation}
Here the factor $1/\Lambda$ comes from shift averaging, and the factor
$1/2$ from taking the odd part. To justify the identity, observe
that on the two supports $1/2\le|x-y|<\Lambda/2$.
With $R_m=(m+1/2)\Lambda$, change variables $x=y+v$ modulo $\Lambda$ to obtain
\[
 \int_0^{R_m}\frac{\phi(v)}v\,dv
 =\frac12\int_{-R_m}^{R_m}\frac{\chi(v)}v\,dv
 =\frac1{2\Lambda}\int_J\int_J\langle f(x),g(y)\rangle
          \sum_{j=-m}^{m}\frac1{x-y+j\Lambda}\,dx\,dy.
\]
There is no singularity on the supports. The cotangent
partial-fraction expansion
\[
 \frac1z+\sum_{j=1}^\infty
       \left(\frac1{z+j\Lambda}+\frac1{z-j\Lambda}\right)
       =\frac\pi \Lambda\cot\frac{\pi z}{\Lambda}
\]
converges uniformly on these differences: the paired terms equal
$2z/(z^2-j^2\Lambda^2)=O(1/(\Lambda j^2))$. Passing to the limit proves
\eqref{eq:optimality-cotangent-integral}.

For $0<|z|\le \Lambda/2$,
\begin{equation}\label{eq:optimality-cotangent-error}
 \left|k_\Lambda(z)-1/z\right|\le2/\Lambda.
\end{equation}
Indeed, $1/u-\cot u$ increases from zero to $2/\pi$ on
$(0,\pi/2]$. The support lengths of $f$ and $g$ are both $wT$,
and every pairing has magnitude at most $r$.
Thus replacing $k_\Lambda$ by $1/z$ in
\eqref{eq:optimality-cotangent-integral} costs at most
$r w^2T^2/\Lambda^2=r/64$.

\paragraph{Recover the original Hilbert denominators.}
Next replace each pulse pair by its centers. Put $c=2(t-s)-1$ and
write $z=c+u-v$ with $|u|,|v|\le w/2$. In the identity
\[
 \frac1{c+q}=\frac1c-\frac q{c^2}+\frac{q^2}{c^2(c+q)},
 \qquad q=u-v,
\]
the linear term integrates to zero. Since $|q|\le w=1/2$ and
$|c|\ge1$, the remainder has magnitude at most $2w^2/|c|^3$.
The integral over a pulse pair therefore differs from $w^2/c$ by
at most $2w^4/|c|^3$. Using
\[
 \sum_{t,s<T}|2(t-s)-1|^{-3}
 \le T\sum_{j\in\mathbb Z}|2j-1|^{-3}<3T,
\]
the total error, including the prefactor $1/(2\Lambda)$ and the pairing
bound $r$, is at most $2r w^4(3T)/(2\Lambda)=3r/64$. The center terms equal
\[
 \frac{w^2}{2\Lambda}\sum_{t,s<T}
      \frac{\langle a_t,b_s\rangle}{2(t-s)-1}
 =\frac{\Gamma_{\mathcal A,\mathcal B}^{(T)}(K_T)}{32T}.
\]
The two errors sum to $r/16$, proving the lemma.
\end{proof}

\subsection{A Visual Example of the Pulse Construction}
\label{app:pulse-visualization}

For illustration, consider a one-dimensional example with $T=2$,
$\mathcal A=\mathcal B=[-1,1]$, and $a_0=a_1=b_0=b_1=1$.
Then $w=1/2$ and $\Lambda=8$: the pulses of $f$ are centered at $0,2$,
and those of $g$ at $1,3$. The plots below show these functions together
with their correlation $\chi$ and its odd part $\phi$ from
\eqref{eq:optimality-correlation}. The correlation $\chi(v)$ is the total
overlap length of the shifted pulses $f(x+v)$ and the pulses $g(x)$,
divided by $8$. Two pulse pairs have center difference $-1$, giving
the middle triangle twice the height of the other two. Taking
$(\chi(v)-\chi(-v))/2$ gives the odd function $\phi$.

\begin{center}
 \includegraphics[width=\textwidth]{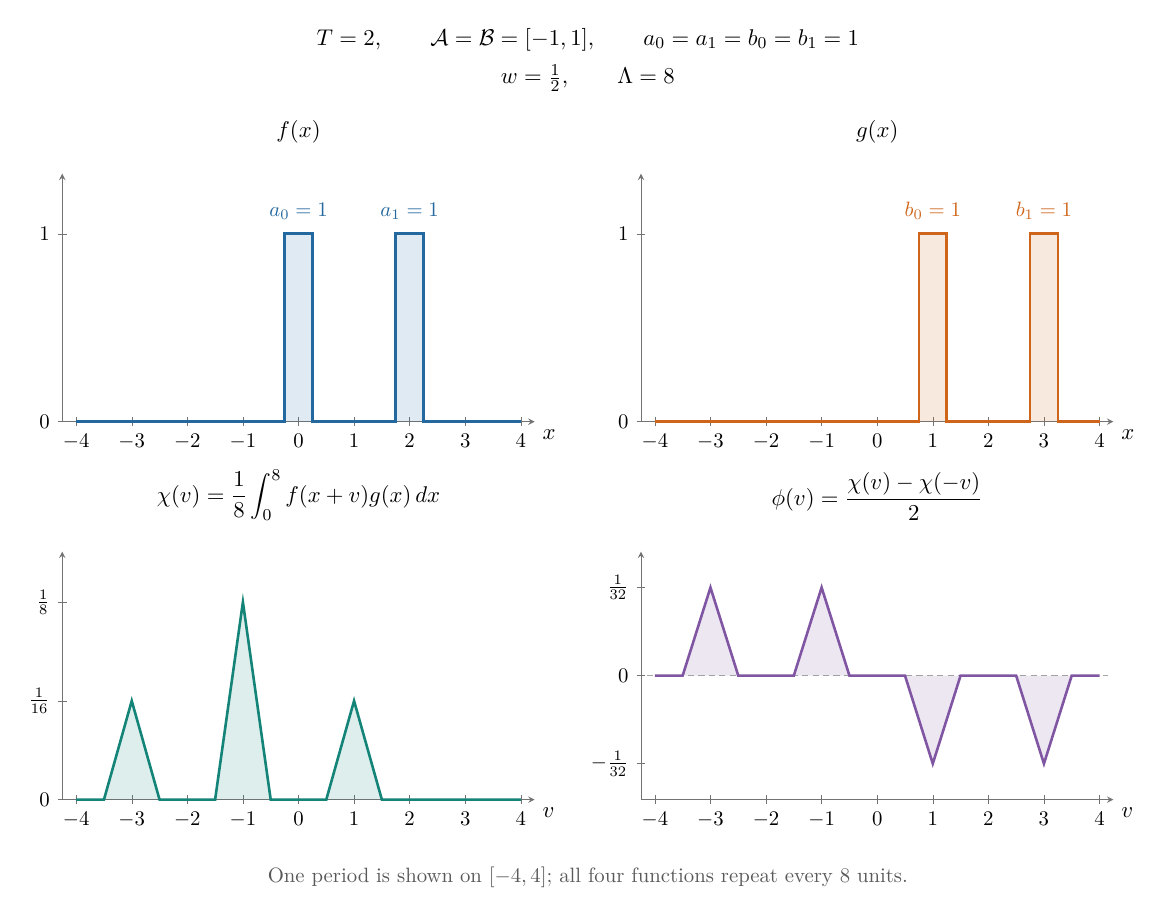}
\end{center}

\section{Further Geometric Bounds}
\label{app:harmonic-refinements}

We prove the additional geometric bounds listed in Section~\ref{sec:harmonic}:
general $\ell_p$ norms, symmetric polytopes, and products of Euclidean
balls. We then give a polynomial lower bound for $\gamma$ for some norms.
Throughout, $C$ denotes an absolute constant that may change between
occurrences. We use the column-norm representation of the matrix objective
from Section~\ref{sec:geometry}.

We will also use invariance under polarity:
\begin{equation}\label{eq:gamma-polarity}
 \gamma(\cL,\cL)=\gamma(\cL^\circ,\cL^\circ).
\end{equation}
Indeed, transposing the bilinear objective exchanges its two sets,
and reversing both time indices sends $K_T$ to $K_T^{\mathsf T}$.
Thus $\Gamma_{\cL,\cL^\circ}^{(T)}(K_T)
=\Gamma_{\cL^\circ,\cL}^{(T)}(K_T)$ at every horizon.

Each upper bound gives a calibration guarantee through
Theorem~\ref{thm:general-geometric-rate}. More generally, if
$\cY\subseteq c+r\cL$, then $c_{\cY}\in\cY$ implies
$\widetilde\cY\subseteq2r\cL$, so
$\gamma(\cY,\cL)\le2r\gamma(\cL,\cL)$.

\subsection{General \texorpdfstring{$\ell_p$}{lp} Norms}

\begin{theorem}\label{thm:lp}
For every $d\ge1$ and $1\le p\le\infty$,
\[
 \gamma(B_p^d,B_p^d)\le C\min\{p^*,\log(2d)\},
\]
where $p^*=\max\{p,p/(p-1)\}$ for $1<p<\infty$, and
$p^*=\infty$ at the endpoints. Consequently, for
$\cY\subseteq c+rB_p^d$ and $\cL=B_p^d$, Algorithm~\ref{alg:harmonic}
achieves $\varepsilon$-calibration whenever
\[
 T\ge\exp\!\left(\frac{Cr\min\{p^*,\log(2d)\}}{\varepsilon}\right).
\]
\end{theorem}
\begin{proof}
By polarity \eqref{eq:gamma-polarity}, it suffices to consider $p\ge2$.
Lemma~\ref{lem:harmonic-finite-moments} and the interpolation in
Footnote~\ref{fn:hilbert-interpolation} give, for every real $q\ge2$,
\begin{equation}\label{eq:lp-scalar-moments}
 \|K_T^{\mathsf T}x\|_q\le Cq\|x\|_q.
\end{equation}
For $x_0,\ldots,x_{T-1}\in B_q^d$, put
$z_s=\sum_{t<T}(K_T)_{ts}x_t$. Apply this estimate in each coordinate:
\[
 \sum_{s<T}\|z_s\|_q^q
 \le(Cq)^q\sum_{t<T}\|x_t\|_q^q\le(Cq)^q T.
\]
H\"older over $s$ and the column-norm representation give
$\gamma(B_q^d,B_q^d)\le Cq$.
For $2\le q\le p$, the norm comparisons
\[
 \|v\|_p\le\|v\|_q\le d^{1/q-1/p}\|v\|_p
\]
then give $\gamma(B_p^d,B_p^d)\le Cq\,d^{1/q-1/p}$.
Choose $q=\min\{p,\max\{2,\log d\}\}$, using the second term
when $p=\infty$. Either $q=p$, or the dimension factor is at most $e$;
in both cases the bound is $C\min\{p,\log(2d)\}$.
Polarity gives the claimed $p^*$ dependence for $p<2$, and
Theorem~\ref{thm:general-geometric-rate} gives the calibration guarantee.
\end{proof}

\subsection{Maxima of Euclidean Norms}

The next bound includes both symmetric polytopes and products of
Euclidean balls.
\begin{lemma}\label{lem:hilbert-max-geometric}
Suppose $A_j:\R^d\to\R^{d_j}$, $1\le j\le m$, are linear maps and
$\|x\|_{\cL}=\max_j\|A_jx\|_2$ is a norm. Then
\[
 \gamma(\cL,\cL)\le C\log(2m).
\]
\end{lemma}
\begin{proof}
First extend \eqref{eq:lp-scalar-moments} to Euclidean-valued inputs.
For a standard Gaussian vector $g$ and $q\ge2$,
$\E|\langle g,v\rangle|^q=m_q\|v\|_2^q$ for a positive scalar $m_q$.
Apply the scalar estimate to $(\langle g,v_t\rangle)_{t<T}$,
take expectations, and cancel $m_q$ to obtain
\[
 \sum_{s<T}\left\|\sum_{t<T}(K_T)_{ts}v_t\right\|_2^q
 \le(Cq)^q\sum_{t<T}\|v_t\|_2^q.
\]
For $x_t\in\cL$, put $z_s=\sum_{t<T}(K_T)_{ts}x_t$.
Applying this estimate to each sequence $A_jx_t$ gives
\[
 \begin{aligned}
 \frac1T\sum_{s<T}\|z_s\|_{\cL}
 &\le\left(\frac1T\sum_{s<T}\sum_{j=1}^m\|A_jz_s\|_2^q\right)^{1/q}\\
 &\le Cq\left(\frac1T\sum_{t<T}\sum_{j=1}^m\|A_jx_t\|_2^q\right)^{1/q}
 \le Cq\,m^{1/q}.
 \end{aligned}
\]
Choose $q=\max\{2,\log m\}$ and take the supremum over inputs and horizons.
\end{proof}

\begin{corollary}[Symmetric polytopes]\label{thm:polyhedral-gamma}
If $\cL=\{x\in\R^d:|\langle u_j,x\rangle|\le1,\ 1\le j\le m\}$
is a norm unit ball, then $\gamma(\cL,\cL)\le C\log(2m)$.
For $\cY=\cL$, Algorithm~\ref{alg:harmonic} achieves
$\varepsilon$-calibration whenever $T\ge(2m)^{C/\varepsilon}$.
\end{corollary}
\begin{proof}
Apply Lemma~\ref{lem:hilbert-max-geometric} to
$A_jx=\langle u_j,x\rangle\in\R$, then use
Theorem~\ref{thm:general-geometric-rate}.
\end{proof}

\begin{corollary}[Products of Euclidean balls]\label{cor:structured}
If $\cL=\prod_{j=1}^kB_2^{d_j}$, with norm
$\|v\|_{\cL}=\max_j\|v_j\|_2$, then
$\gamma(\cL,\cL)\le C\log(2k)$, independently of the dimensions $d_j$.
For $\cY=\cL$, Algorithm~\ref{alg:harmonic} achieves
$\varepsilon$-calibration whenever $T\ge(2k)^{C/\varepsilon}$.
\end{corollary}
\begin{proof}
Use the coordinate projections $A_jv=v_j$ in
Lemma~\ref{lem:hilbert-max-geometric}, then apply
Theorem~\ref{thm:general-geometric-rate}. The forecast buckets are
complete joint forecasts.
\end{proof}

\subsection{A Polynomial Lower Bound for the Geometric Coefficient}
\label{app:geometric-lower-bound}

Bourgain constructed sums of $n$ sines of size $O(n^{2/3})$, although
the corresponding cosine sum has value $n$ at zero. We use translates
of these sine sums as witnesses for the matrix objective.

\begin{theorem}\label{thm:bourgain-lower-bound}
There is an absolute constant $c>0$ such that, for every $d\ge2$,
some norm ball $\cL\subseteq\R^d$ satisfies
\[
 \gamma(\cL,\cL)\ge c\,d^{1/3}.
\]
\end{theorem}
\begin{proof}
Bourgain's sine-sum theorem~\cite[Expos\'e~3, opening statement]{bourgain1983sinus} gives, for every
$n\ge1$, distinct positive integers $\lambda_1,\ldots,\lambda_n$ such
that
\[
 S(u)=\sum_{j=1}^n\sin(\lambda_j u),
 \qquad A:=\sup_{u\in\R}|S(u)|\le C_0 n^{2/3},
\]
where $C_0$ is an absolute constant. Consider the real vector space
\[
 E=\operatorname{span}\{\sin(\lambda_j\,\cdot),
                       \cos(\lambda_j\,\cdot):1\le j\le n\},
 \qquad \|f\|=\sup_{u\in\R}|f(u)|.
\]
Identify $E$ with $\R^{2n}$ using the displayed basis, and let $\cL$
be its norm unit ball. Translation preserves this space and its norm,
and every evaluation $f\mapsto f(u)$ belongs to $\cL^\circ$.

Choose $0<\eta<2\pi/\max_j\lambda_j$ and write
$\omega_j=\lambda_j\eta\in(0,2\pi)$. For $0\le t,s<T$, take
\[
 a_t(u)=\frac{S(u+t\eta)}A,
 \qquad b_s(f)=f\bigl(-(s+\tfrac12)\eta\bigr).
\]
Then $a_t\in\cL$ and $b_s\in\cL^\circ$, and
$\langle a_t,b_s\rangle=A^{-1}\sum_{j=1}^n
\sin((t-s-\tfrac12)\omega_j)$. Counting the $T-|k|$ pairs with
$t-s=k$ therefore gives
\begin{equation}\label{eq:bourgain-witnesses}
 \frac{\Gamma_{\cL,\cL^\circ}^{(T)}(K_T)}T
 \ge\frac1A\sum_{j=1}^n
 \sum_{k=-(T-1)}^{T-1}
 \left(1-\frac{|k|}T\right)
 \frac{\sin((k-\tfrac12)\omega_j)}{2k-1}.
\end{equation}
For every $0<\omega<2\pi$, the odd sine series gives
\[
 \lim_{N\to\infty}\sum_{k=-N}^{N}
 \frac{\sin((k-\tfrac12)\omega)}{2k-1}
 =2\sum_{m=0}^\infty
 \frac{\sin((2m+1)\omega/2)}{2m+1}
 =\frac\pi2.
\]
Both tails converge by Dirichlet's test; the last equality is the
Fourier series of the square wave on $(0,\pi)$. The weights
$1-|k|/T$ in \eqref{eq:bourgain-witnesses} take the arithmetic mean
of the symmetric partial sums with $N=0,\ldots,T-1$, so the weighted
sum has the same limit. Since $\cL$ is centrally symmetric, its center
of mass is zero. Letting $T\to\infty$ in
\eqref{eq:bourgain-witnesses} yields
\[
 \gamma(\cL,\cL)\ge\frac{\pi n}{2A}
 \ge\frac{\pi}{2C_0}n^{1/3}.
\]
This proves the claim in dimension $2n$. For odd $d$, take
$n=\lfloor d/2\rfloor$ and replace this ball by
$\cL\times[-1,1]$. Extending the vectors and functionals above by a
zero coordinate preserves their objective values. Since
$n\ge d/3$ for $d\ge2$, adjusting the absolute constant proves the
claim in every dimension.
\end{proof}

This rules out a uniform polylogarithmic bound on $\gamma(\cL,\cL)$.
It is a limitation of the guarantee in
Theorem~\ref{thm:general-geometric-rate}; it does not by itself give a
lower bound on calibration error.

\section{Calibration Rates: OLO and Harmonic Discrepancy}
\label{app:comparisons}

Appendix~\ref{app:harmonic-refinements} bounded $\gamma$ for familiar
norms and exhibited norms for which it grows polynomially with dimension.
We now ask how this geometric coefficient compares with the one governing
earlier calibration guarantees. Answering this question leads from the
fixed-norm examples to more elaborate constructions separating harmonic
discrepancy from online linear optimization.

Fishelson, Golowich, Mohri, and Schneider~\cite{fishelson2025high} extend the high-dimensional
calibration result of Peng~\cite{peng2025high} from the simplex to
arbitrary convex forecast sets $\cY$ and error norms $\|\cdot\|_{\cL}$.
Their algorithm, TreeCal, is analyzed through a reduction to swap regret
with Bregman-divergence losses. Its calibration rate is governed by the
geometry of the forecast set and error norm. This appendix explains the
parameter in that rate and compares it with our harmonic discrepancy.

We focus on forecasting over a norm unit ball with error measured in
the same norm, so $\cY=\cL$. In this setting, Corollary C.6 of
Fishelson, Golowich, Mohri, and Schneider~\cite{fishelson2025high} gives the sufficient horizon
\begin{equation}\label{eq:treecal-comparison}
 \exp\!\left(O\!\left(\frac{\rho(\cL,\cL^\circ)}{\varepsilon^2}
                     \log(2+\varepsilon^{-2})\right)\right),
 \qquad 0<\varepsilon\le1.
\end{equation}
Here $\cL^\circ$ is the dual norm unit ball, and
$\rho(\cL,\cL^\circ)$ measures the difficulty of online linear
optimization with actions in $\cL$ and losses in $\cL^\circ$.
The theorem number and normalization refer to arXiv:2505.21460v2.

\subsection{The OLO Parameter and Optimal Regularizers}

In online linear optimization (OLO), a learner chooses $x_t\in A$
before observing a loss vector $y_t\in B$, incurs loss
$\langle x_t,y_t\rangle$, and competes with the best fixed action
in hindsight. Its cumulative external regret is
\[
 \sum_{t=1}^T\langle x_t,y_t\rangle
 -\min_{x\in A}\sum_{t=1}^T\langle x,y_t\rangle.
\]
Let $R_T^*(A,B)$ be the minimax expected value of this regret, and define
\[
 \rho(A,B)=\limsup_{T\to\infty}\frac{R_T^*(A,B)^2}{T}.
\]
Thus $\rho$ is the square of the optimal leading coefficient in a
$\sqrt T$ regret bound: a small value means that OLO on these two sets
can be learned quickly. The order of the arguments matters, since
the learner's action set and the adversary's loss set play different roles.

This parameter is closely tied to the classical problem of choosing a
regularizer for online learning. If $R$ is $1$-strongly convex with
respect to $\|\cdot\|_{\cL}$ and has oscillation
$\sup_{x\in\cL}R(x)-\inf_{x\in\cL}R(x)\le r$, then
follow-the-regularized-leader (FTRL), with a suitable learning rate,
has regret $O(\sqrt{rT})$ against losses in $\cL^\circ$.
The smaller the range of such a regularizer, the better the guarantee.
Srebro, Sridharan, and Tewari~\cite{srebro2011universality} showed that
an appropriate regularizer makes online mirror descent nearly minimax
optimal. Gatmiry, Schneider, and Jegelka~\cite{gatmiry2024regularizers}
strengthened this to a constant-factor guarantee for FTRL on symmetric
convex action and loss sets. Consequently,
\[
 \rho(\cL,\cL^\circ)\asymp
 \inf_{\substack{R:\cL\to\R\\
       R\text{ is }1\text{-strongly convex in }\|\cdot\|_{\cL}}}
 \left(\sup_{x\in\cL}R(x)-\inf_{x\in\cL}R(x)\right).
\]
The calibration analysis uses the maximum Bregman divergence of the
regularizer. Fishelson, Golowich, Mohri, and Schneider~\cite[Lemma~2.4]{fishelson2025high}
show that optimizing this quantity gives the same parameter up to
constants for centrally symmetric forecast sets. This is why the
optimal OLO coefficient appears in \eqref{eq:treecal-comparison}.

Our calibration guarantee uses $\gamma$ in place of $\rho$.
For $\cY=\cL$, Theorem~\ref{thm:general-geometric-rate} gives the horizon
\[
 \exp\!\left(O\!\left(\frac{\gamma(\cL,\cL)}{\varepsilon}\right)\right),
 \qquad
 \gamma(\cL,\cL)=\sup_{T\ge1}
           \frac{\Gamma_{\cL,\cL^\circ}^{(T)}(K_T)}T.
\]
The two bounds therefore differ both in their dependence on accuracy
and in their geometric coefficient. We concentrate here on the second
difference. Both coefficients concern the same pair $\cL,\cL^\circ$:
one comes from an online regret problem, and the other from the harmonic
matrix discrepancy. How closely are these two measures related?

\subsection{Polarity and the Two OLO Orientations}

The cube already shows that $\gamma$ can be much smaller than the OLO
coefficient in \eqref{eq:treecal-comparison}. Indeed,
\[
 \rho(B_\infty^d,B_1^d)=\Theta(d),\qquad
 \gamma(B_\infty^d,B_\infty^d)=O(\log(2d)).
\]
The OLO upper bound follows from Euclidean regularization, and signed
coordinate losses give a matching lower bound. The harmonic bound is
Lemma~\ref{lem:three-bodies}. Thus, for fixed accuracy, the harmonic
guarantee on the cube is polynomial in dimension, while
\eqref{eq:treecal-comparison} is exponential in dimension.

There is, however, another OLO problem associated with the same polar
pair: the learner can act in $B_1^d$ while losses range over $B_\infty^d$.
This is the signed-experts problem, for which
\[
 \rho(B_1^d,B_\infty^d)=\Theta(\log(2d)).
\]
Exponential weights over the $2d$ signed coordinate vectors gives the
upper bound, and independent random-sign losses give the matching
asymptotic lower bound. Reversing the OLO orientation therefore reduces
its coefficient to the same logarithmic scale as our harmonic upper bound.

For harmonic discrepancy, exchanging the two sets has no effect.
As shown in \eqref{eq:gamma-polarity}, the harmonic coefficient is
invariant under polarity: $\gamma(\cL,\cL)=\gamma(\cL^\circ,\cL^\circ)$.

This suggests comparing $\gamma$ with the better of the two OLO
orientations. Define
\[
 \rho_{\min}(\cL)=\min\{\rho(\cL,\cL^\circ),
                              \rho(\cL^\circ,\cL)\}.
\]
Like $\gamma$, this quantity is unchanged when $\cL$ is replaced by its
polar. The minimum compares the geometry of the two OLO problems;
it does not identify calibration on $\cL$ with calibration on $\cL^\circ$.

The principal norm-ball examples above leave this comparison open.
For the Euclidean ball, both $\rho_{\min}$ and $\gamma$ are bounded by
absolute constants. For the cube and its polar, $\rho_{\min}$ is
$\Theta(\log(2d))$, matching the scale of our harmonic upper bounds.
It is therefore natural to ask whether one of these parameters controls
the other for general norm balls.

The next two examples show that neither does. In the first,
$\gamma$ stays bounded while $\rho_{\min}$ grows without bound; in the
second, $\rho_{\min}$ stays bounded while $\gamma$ grows without bound.
Both constructions use action and loss sets that are polar to one
another. Thus the failure of comparison already occurs for the norm-ball
geometries underlying the calibration guarantees above.

\subsection{Bounded Harmonic Discrepancy and Large OLO in Both Directions}

The first construction puts an $\ell_4$ ball and an $\ell_{4/3}$ ball
in separate blocks. Taking the polar exchanges the blocks, so choosing
the other OLO orientation cannot remove the difficult block. Each block
nevertheless has bounded harmonic discrepancy.

\begin{proposition}\label{prop:gamma-small-olo-large}
For the unit ball
\[
 \cL_d=\{(x,y)\in\R^d\times\R^d:
                       \|x\|_4^2+\|y\|_{4/3}^2\le1\},
\]
we have
\[
 \gamma(\cL_d,\cL_d)=\Theta(1),\qquad
 \rho(\cL_d,\cL_d^\circ)
 =\rho(\cL_d^\circ,\cL_d)=\Theta(\sqrt d).
\]
\end{proposition}
\begin{proof}
The polar is
$\cL_d^\circ=\{(u,v):\|u\|_{4/3}^2+\|v\|_4^2\le1\}$,
which is obtained by exchanging the two blocks. Splitting the pairing
into its two blocks gives
\[
 \Gamma_{\cL_d,\cL_d^\circ}^{(T)}(K_T)
 \le\Gamma_{B_4^d,B_{4/3}^d}^{(T)}(K_T)
    +\Gamma_{B_{4/3}^d,B_4^d}^{(T)}(K_T)=O(T)
\]
by Theorem~\ref{thm:lp} and polarity. Taking $T=1$ gives
the lower bound $\gamma\ge1$.

For OLO, let the losses be $(\epsilon_t e_{i_t},0)$, where the signs
are independent and uniform and the deterministic indices visit all
coordinates equally often, up to one visit. These losses are in
$\cL_d^\circ$ and give every learner expected loss zero. If $S_i$ is
the signed sum on coordinate $i$, the expected regret is
\[
 \mathbb E\|(S_i)_{i=1}^d\|_{4/3}
 \ge d^{-1/4}\sum_i\mathbb E|S_i|
 \ge c d^{1/4}\sqrt T\qquad(T\ge2d).
\]
For the upper bound use the quadratic regularizer with matrix
$D=\operatorname{diag}(I_d,\sqrt d\,I_d)$.
On $\cL_d$ we have $\|(x,y)\|_D^2\le\sqrt d$, while on
$\cL_d^\circ$ we have $\|(u,v)\|_{D^{-1}}^2\le1$.
The quadratic regret bound is therefore $O(d^{1/4}\sqrt T)$.
Exchanging the blocks proves the same bounds in the other orientation.
\end{proof}

Thus the harmonic guarantee is dimension independent on these balls,
while the TreeCal bound has a $\sqrt d$ geometric coefficient even
after choosing the better OLO orientation. These are comparisons of
the stated guarantees, not lower bounds on an algorithm's calibration error.

\subsection{Bounded OLO and Unbounded Harmonic Discrepancy}

For the converse separation, we need a family whose OLO problems admit
uniformly good regularizers while their harmonic discrepancy grows.
Increasing the dimension of a single $\ell_p$ ball, with fixed
$1<p<\infty$, cannot achieve this: Theorem~\ref{thm:lp} bounds its
harmonic discrepancy independently of dimension. Instead, we increase
the number of levels at which two different norms are combined. The
choice below preserves the curvature needed for OLO, while allowing
harmonic effects to accumulate across the levels.

For a normed space $X$, write $\ell_p^2(X)$ for pairs $(u,v)\in X^2$
with norm $(\|u\|_X^p+\|v\|_X^p)^{1/p}$. Starting with $Z_0=\R$, set
\[
 Z_{n+1}=\ell_2^2(\ell_{4/3}^2(Z_n)),
 \qquad \cL_n=\text{the unit ball of }Z_n.
\]
One level takes four vectors from the previous space, combines them
in two pairs using $\ell_{4/3}$, and then combines the two resulting
norms using $\ell_2$. Thus $\dim Z_n=4^n$. If $N_n$ is its norm,
the recursion is
\[
 N_{n+1}(x)=\left[
   (N_n(x_1)^{4/3}+N_n(x_2)^{4/3})^{3/2}
  +(N_n(x_3)^{4/3}+N_n(x_4)^{4/3})^{3/2}
                       \right]^{1/2},\qquad N_0(x)=|x|.
\]
The dual spaces $E_n=Z_n^*$ have the same structure, with $4/3$
replaced by its conjugate exponent $4$:
\[
 E_0=\R,\qquad E_{n+1}=\ell_2^2(\ell_4^2(E_n)).
\]
It is useful to keep both descriptions in mind: $\cL_n$ is the OLO
action set, its polar is the unit ball of $E_n$, and the harmonic
coefficient is the same on either side by \eqref{eq:gamma-polarity}.

\begin{proposition}\label{prop:olo-small-gamma-large}
For this family,
\[
 \rho(\cL_n,\cL_n^\circ)\le3,\qquad
 \rho_{\min}(\cL_n)=\Theta(1),\qquad
 \gamma(\cL_n,\cL_n)\longrightarrow\infty.
\]
\end{proposition}

Before proving the proposition, we explain the two properties of these
norms and how they enter the argument.

For the OLO bound, the key fact is that $\|\cdot\|_{E_n}^2/2$ is
$3$-smooth for every $n$: its curvature, measured in the norm of $E_n$,
has an upper bound independent of the nesting depth. Both the
$\ell_2$ and $\ell_4$ operations preserve this bound; the induction
below shows why the constant does not multiply at each level.
Convex duality turns this smoothness into strong convexity of
\[
 R_n(x)=\tfrac12N_n(x)^2.
\]
Specifically, $R_n$ is $1/3$-strongly convex in $N_n$, and its range
on $\cL_n$ is $[0,1/2]$. The regularizer therefore has both the
curvature and the bounded range needed for a regret bound uniform in
$n$. FTRL gives $R_T^*(\cL_n,\cL_n^\circ)\le\sqrt{3T}$, hence
$\rho(\cL_n,\cL_n^\circ)\le3$. One good orientation is enough to
bound $\rho_{\min}$.

For the harmonic bound, the relevant property is how the periodic
Hilbert transform $\mathcal H$ acts on functions on the circle
$\mathbb T$ taking values in $E_n$.
This is a continuous counterpart of the harmonic matrix $K_T$ in our
discrepancy. Qiu~\cite[Theorem~3.10]{qiu2012umd} shows that iterating
two distinct exponents makes the martingale-transform constants grow:
changing the signs of martingale differences can amplify their
$L_2$ norm by an arbitrarily large factor as the depth increases.
Bourgain's comparison~\cite[Section~1, Lemmas~1--2]{bourgain1983remarks}
then implies that the operator norms of $\mathcal H$ on
$L_2(\mathbb T;E_n)$ are unbounded. Thus the alternating exponents
$2$ and $4$ give growing Hilbert-transform norms while preserving
the uniform smoothness used for OLO. By approximation and duality,
this growth supplies smooth periodic functions $f,g$ taking values in
$E_n,E_n^*$, respectively, with unit $L_2$ norms and arbitrarily large pairing
$\int\langle\mathcal Hf,g\rangle$.

There is still a step between such functions and a large value of
$\gamma$. An $L_2$ bound controls the average squared norm of a function,
whereas a discrepancy witness must lie in the prescribed unit ball
at every time. We obtain this pointwise control by taking many equally
spaced translates of $f$ and $g$, placing them in separate coordinates
of an $\ell_2$ sum, and normalizing by the square root of the number of
copies. At each time, the squared norm of this vector is a Riemann sum
for the original squared $L_2$ norm. It is therefore uniformly bounded when
there are enough copies, while the integrated correlation is preserved.
This also explains the role of the outer $\ell_2$ in our construction:
these sums of copies embed isometrically into later spaces $E_{n+k}$
by setting one child of each added $\ell_4$ pair to zero.

Finally, the lemma below converts the resulting bounded periodic
functions into witnesses for the discrete harmonic discrepancy. It
uses a common random phase and a logarithmically distributed spacing,
as in Section~\ref{sec:harmonic-optimality}, to turn their
Hilbert-transform correlation into a large pairing with $K_T$.
Polarity then transfers the lower bound from the unit ball of
$E_{n+k}$ to $\cL_{n+k}$. We now make these steps precise, beginning
with the conversion from periodic correlations to $\gamma$.
\begin{lemma}[From periodic correlations to harmonic discrepancy]
\label{lem:periodic-correlation-gamma}
Let $a,b$ be real trigonometric polynomials taking values in a norm
unit ball $L$ and its polar, respectively. Define, using normalized
measure on $\mathbb T=\R/(2\pi\mathbb Z)$,
\[
 c(v)=\int_{\mathbb T}\langle a(t+v),b(t)\rangle\,dm(t),\qquad
 \phi(v)=\frac{c(v)-c(-v)}2.
\]
Then
\[
 \gamma(L,L)\ge |I|,\qquad
 I=\int_0^\infty\frac{\phi(v)}v\,dv.
\]
\end{lemma}

The function $c(v)$ measures the average pairing of $a$ and $b$ when
one is shifted by $v$ relative to the other. Its odd part $\phi(v)$
keeps the component that changes sign when the direction of the shift
is reversed. This is the component detected by the Hilbert transform:
the identity
$\int_0^\infty\sin(mv)\,dv/v=(\pi/2)\operatorname{sgn}(m)$
gives, with the usual normalization of $\mathcal H$,
\begin{equation}\label{eq:periodic-correlation-hilbert}
 |I|=\frac\pi2\left|\int_{\mathbb T}
                       \langle\mathcal Ha(t),b(t)\rangle\,dm(t)\right|.
\end{equation}
Thus the lemma converts a large Hilbert-transform pairing into a large
harmonic discrepancy.

The construction uses two averages. A shared random phase turns
pointwise pairings into the correlation $c$ at the relative displacement
of the two indices. A random reflection, together with a sign change,
extracts its odd part $\phi$. We then average the spacing between
indices on a logarithmic scale. The measure $du/u$ is unchanged by
rescaling $u$, so this second average nearly removes the magnitude
of the displacement and retains its sign. After choosing the sign of
$b$ so that $I\ge0$, the resulting expected pairings have the signs
of $K_T$ and size approximately $I/\log T$.
Since the total absolute mass of $K_T$ is asymptotic to $T\log T$,
summing these contributions and dividing by $T$ recovers $|I|$.
The proof below chooses finite cutoffs so that this approximation holds
simultaneously for every matrix entry.

\begin{proof}
First, $I$ converges: smoothness and oddness give $\phi(v)=O(v)$
near zero, while periodicity and mean zero give a bounded primitive,
so integration by parts controls the tail at infinity. Replacing $b$
by $-b$ if necessary, assume $I\ge0$. The case $I=0$ is immediate.

\emph{A shared random phase produces the correlation.}
Fix a horizon $T$ and a spacing $u>0$. Draw one uniform phase
$\theta\in\mathbb T$, shared by all indices, and choose each of the
following two families with probability $1/2$:
\[
 \begin{array}{lll}
  a_i=a(\theta+iu),&b_j=b(\theta+(j+\tfrac12)u),&\text{or}\\
  a_i=a(\theta-iu),&b_j=-b(\theta-(j+\tfrac12)u),&0\le i,j<T.
 \end{array}
\]
Every realization is feasible, since $a$ takes values in $L$ and
$L^\circ$ is symmetric. Set $k=i-j-1/2$. In the first family,
translation invariance of the uniform phase gives
\[
 \mathbb E_\theta
   \langle a(\theta+iu),b(\theta+(j+\tfrac12)u)\rangle
 =\int_{\mathbb T}\langle a(t+ku),b(t)\rangle\,dm(t)
 =c(ku).
\]
The second family similarly gives $-c(-ku)$. Averaging the two
therefore yields
\[
 \mathbb E[\langle a_i,b_j\rangle\mid u]
 =\frac{c(ku)-c(-ku)}2=\phi(ku).
\]
The half-step shift ensures that $k$ is never zero and has exactly
the sign of $(K_T)_{ij}=1/(2k)$. It also gives the uniform bounds
$1/2\le|k|<T$.

\emph{Logarithmic spacing removes the size of the displacement.}
The reason for using the measure $du/u$ is the change of variables
$v=|k|u$: oddness of $\phi$ gives
\[
 \int_0^\infty\frac{\phi(ku)}u\,du
 =\operatorname{sgn}(k)\int_0^\infty\frac{\phi(v)}v\,dv
 =\operatorname{sgn}(k)I.
\]
Thus the integral depends on the sign of $k$, but not its magnitude.
To turn this measure into a probability distribution, we restrict it
to a finite interval. Fix $0<\delta<I/2$ and choose
$0<a_0\le b_0$ so that
\[
 \left|\int_\alpha^\beta\frac{\phi(v)}v\,dv-I\right|\le2\delta
 \qquad\text{whenever }0<\alpha\le a_0\text{ and }\beta\ge b_0.
\]
Now draw $u$ independently of the phase and reflection, with density
\[
 p_T(u)=\frac1{u\Lambda_T}\quad(a_0/T\le u\le2b_0),
 \qquad \Lambda_T=\log(2b_0T/a_0).
\]
This assigns equal probability to intervals of equal length on the
logarithmic scale. The endpoints are chosen to work for every pair
of indices: after the substitution $v=|k|u$, the lower endpoint
$|k|a_0/T$ is at most $a_0$, and the upper endpoint $2|k|b_0$ is at
least $b_0$. Consequently,
\[
 \begin{aligned}
 \operatorname{sgn}(k)\,\mathbb E\langle a_i,b_j\rangle
 &=\frac1{\Lambda_T}\int_{a_0/T}^{2b_0}
                         \frac{\phi(|k|u)}u\,du\\
 &=\frac1{\Lambda_T}\int_{|k|a_0/T}^{2|k|b_0}
                         \frac{\phi(v)}v\,dv\\
 &\ge\frac{I-2\delta}{\Lambda_T}.
 \end{aligned}
\]
We have obtained the same positive lower bound for every entry after
aligning its sign with the harmonic kernel.

\emph{The harmonic matrix compensates for the logarithmic normalization.}
All indices use the same sampled spacing, phase, and reflection,
so every realization gives a single feasible discrepancy witness.
The supremum defining $\Gamma$ is therefore at least the expected
objective. Since $\sum_{i,j<T}|(K_T)_{ij}|=T(2h_T-1)$,
\[
 \begin{aligned}
 \gamma(L,L)
 &\ge\frac{\Gamma_{L,L^\circ}^{(T)}(K_T)}T\\
 &\ge\frac1T\,\mathbb E\!\left[
               \sum_{i,j<T}(K_T)_{ij}\langle a_i,b_j\rangle\right]\\
 &\ge\frac{I-2\delta}{T\Lambda_T}
               \sum_{i,j<T}|(K_T)_{ij}|\\
 &=\frac{2h_T-1}{\Lambda_T}(I-2\delta).
 \end{aligned}
\]
Here $2h_T-1=\log T+O(1)$ and
$\Lambda_T=\log T+\log(2b_0/a_0)$, so their ratio tends to $1$.
Letting $T\to\infty$ and then $\delta\downarrow0$ gives
$\gamma(L,L)\ge I$, as claimed.
\end{proof}

\begin{proof}[Proof of Proposition~\ref{prop:olo-small-gamma-large}]
\emph{The OLO bound.}
For any norm $N$ formed by nesting scalar absolute values with
$\ell_2$ and $\ell_4$ sums, $N^2/2$ is $3$-smooth in $N$.
Here is an induction proving the stronger directional inequality
\[
 D^2(N^2/2)(x)[v,v]\le3N(v)^2-2(DN(x)[v])^2.
\]
The scalar case is equality away from zero. For an outer $\ell_p$ sum,
$p\in\{2,4\}$, put $a_i=N_i(x_i)$, $b_i=DN_i(x_i)[v_i]$,
$s=(\sum_i a_i^p)^{1/p}$ and $b=Ds(x)[v]$. Define
\[
 A=s^{2-p}\sum_i a_i^{p-2}b_i^2,\qquad
 B=s^{2-p}\sum_i a_i^{p-2}N_i(v_i)^2.
\]
Cauchy--Schwarz gives $A\ge b^2$, and H\"older gives $B\le N(v)^2$.
The induction hypothesis gives
$a_i a_i''\le3N_i(v_i)^2-3b_i^2$, so differentiation yields
\[
 D^2(s^2/2)[v,v]
 \le3B+(p-4)A-(p-2)b^2
 \le3N(v)^2-2b^2.
\]
The integrated smoothness inequality extends across zero blocks by
continuity. Apply it to $E_n$. By convex duality, $N_n^2/2$ is
$1/3$-strongly convex on $Z_n$, with range $[0,1/2]$ on $\cL_n$.
Follow-the-regularized-leader with unit dual-norm losses gives
\[
 R_T^*(\cL_n,\cL_n^\circ)
 \le\inf_{\eta>0}\left(\frac1{2\eta}+\frac{3\eta T}2\right)
 =\sqrt{3T}.
\]
A scalar random-walk lower bound applies to either orientation of any
nontrivial norm ball, so $\rho_{\min}(\cL_n)=\Theta(1)$.

\emph{The harmonic lower bound.}
Qiu's theorem on iterated $\ell_p^2(\ell_q^2)$ spaces
\cite[Theorem~3.10]{qiu2012umd} implies that the martingale-transform
constants of $E_n$ are unbounded. Bourgain's comparison
\cite[Section~1, Lemmas~1--2]{bourgain1983remarks} therefore makes the
norms of the periodic Hilbert transform $\mathcal H$ on
$L_2(\mathbb T;E_n)$ unbounded as well.
By density and finite-dimensional duality, for any $M>0$ there are
$n$ and real trigonometric polynomials $f,g$, with values in $E_n,E_n^*$,
such that
\[
 \|f\|_{L_2}=\|g\|_{L_2}=1,\qquad
 \left|\int_{\mathbb T}\langle\mathcal Hf,g\rangle\,dm\right|>M.
\]
To obtain pointwise bounded functions, take $q=2^k$ equally spaced
phases $\theta_j=2\pi j/q$ and the space $F=\ell_2^q(E_n)$, and set
\[
 a(t)=q^{-1/2}(f(t+\theta_j))_{j<q},\qquad
 b(t)=q^{-1/2}(g(t+\theta_j))_{j<q}.
\]
Uniform convergence of the translated Riemann sums gives
$\sup_t\|a(t)\|_F,\sup_t\|b(t)\|_{F^*}\le2$ for large enough $k$.
Their correlation is exactly that of $f,g$, by translation invariance.
For its odd part $\phi$, the Hilbert-transform identity
\eqref{eq:periodic-correlation-hilbert} gives
\[
 \left|\int_0^\infty\frac{\phi(v)}v\,dv\right|
 =\frac\pi2\left|\int_{\mathbb T}
                         \langle\mathcal Hf,g\rangle\,dm\right|.
\]
Apply Lemma~\ref{lem:periodic-correlation-gamma} to $a/2,b/2$ to get
$\gamma(B_F,B_F)>\pi M/8$.

Finally, $E_{n+k}$ contains $F=\ell_2^{2^k}(E_n)$ isometrically:
at each added level, set the second child of each $\ell_4$ pair to zero.
Harmonic discrepancy is monotone under isometric subspaces, since
dual witnesses extend without increasing their norms by Hahn--Banach.
Thus $\gamma(B_{E_{n+k}},B_{E_{n+k}})>\pi M/8$.
These coefficients are nondecreasing with $n$, and polarity
\eqref{eq:gamma-polarity} identifies them with $\gamma(\cL_n,\cL_n)$.
Since $M$ was arbitrary, the latter tend to infinity.
\end{proof}

The nested-norm example gives a dimension-independent TreeCal guarantee
while the coefficient in our harmonic guarantee is unbounded.
Together with Proposition~\ref{prop:gamma-small-olo-large}, it rules out
a comparison in either direction between $\gamma$ and $\rho_{\min}$.
The construction does not give a quantitative growth rate for $\gamma$
in the dimension: the number of copying levels depends on the Hilbert
transform witness. The two calibration horizons also have different
accuracy dependence, as displayed in \eqref{eq:treecal-comparison}.

\section{How This Paper Came Together: An AI Acknowledgement}
\label{app:ai-acknowledgements}

We started this project because we wanted to better understand recent
results on high-dimensional swap regret minimization.
The TreeSwap algorithm of Dagan, Daskalakis, Fishelson, and Golowich~\cite{dagan2024external} and the
independent work of Peng and Rubinstein~\cite{peng2024fast} reduced the
number of rounds needed for fixed accuracy from polynomial to
polylogarithmic in the number of actions. Swap regret is a
particular instance of a more general online learning framework called
Blackwell approachability~\cite{blackwell1956analog}. This raised the
question that motivated our work: which other approachability problems
admit efficient algorithms in the fixed-accuracy regime?

During the first year, roughly summer 2024 to summer 2025, we spent a good
deal of time trying to understand what made the TreeSwap proof work. One obstacle
was the special structure of swap regret. Its transformations are
endomorphisms: they map the learner's action set back into itself.
Endomorphisms have the nice property that averaging repeated applications
converges to a fixed point (think Markov chains). The TreeSwap proof relied
on this property, but the more general fixed-point conditions arising in approachability do not
automatically provide this structure. We focused on an instance of
approachability built on skew-symmetric transformations
(Appendix~\ref{app:skew-response}).

This eventually led us towards high-dimensional calibration.
Peng~\cite{peng2025high} obtained a polynomial dependence on dimension for
multi-class calibration at fixed accuracy, and our TreeCal
algorithm~\cite{fishelson2025high} extended this approach to general convex
forecast sets and norms. AI played a modest role in this part of the story. The models
we had access to at this time were useful for basic literature searches, screening
ideas, and checking calculations, but we did not find them capable of
providing much help with the main research obstacles.

This approach handled our skew-symmetric example, but we were still
unsatisfied. In particular, it came nowhere close to recovering TreeSwap's
rates. Many natural approachability problems, including regret
minimization, have outcomes in a hypercube. For these outcome sets,
TreeCal's calibration guarantee still required exponentially many rounds
in the dimension, so the reduction through calibration gave us no
useful high-dimensional rate.
We quickly settled on a pivotal approachability instance that we called \emph{orthant
calibration} (Appendix~\ref{sec:app-orthants}). It had outcomes in the hypercube
but lacked the endomorphism structure used by TreeSwap. Orthant calibration
became our focus from summer 2025 to summer 2026.

There was nevertheless a reason to be optimistic about the tree-based
approach. At the level of its averaging structure, TreeSwap looks very
much like responding to TreeCal's forecasts. Yet its swap-regret
guarantee is much better than the generic calibration guarantee would
suggest. Perhaps responding to TreeCal could also work well for
orthant calibration, even if an analysis through calibration could
not show it.
We tried to make this idea work for a long time. Eventually, we constructed
by hand a dyadic adversary, with coin flips governing outcomes at different
dyadic time scales, and proved that it defeats the TreeCal-style
response strategy we were studying. This was a useful failure: it showed
that this particular route needed something more.

This adversary also
had a weakness though. Its independent sign changes had a multiscale structure
that we could exploit to obtain a good approachability guarantee against
it. The algorithm that did so reversed the order of two operations. Instead
of responding to average outcomes of sub-trees, it averaged
responses to the outcomes within a sub-tree. Schematically, the two
ideas were
\[
 \text{tree-weighted average, then respond}
 \qquad\text{and}\qquad
 \text{respond, then tree-weighted average}.
\]
The first was motivated by TreeSwap and TreeCal but failed on our adversary. The second defeated
this particular adversary but was otherwise useless as a general algorithm.
The question became
whether we could combine what was useful about both.

After spending considerable time on this interpolation, we arrived at a
simple proposal: put an averaging operation on each side of the response
map. First form tree-weighted averages of past outcomes, then
respond to those averages, then take a tree-weighted average of the
resulting responses. This was the main idea we brought to Gemini 3 Deep Think. We had
an algorithmic direction and a reason to pursue it, but we were still
missing an analysis that would make the
whole thing work.

Given this idea nugget, Gemini was able to simplify our main algorithm and
put the core analysis into place. Rather than keeping tree weights, it
suggested using smooth harmonic weights for both averaging operations.
The connection is natural in hindsight: tree averaging puts
roughly equal total weight on successive dyadic scales. Harmonic weights
do the same thing, since the sums
\[
 1,\qquad \frac12+\frac13,\qquad
 \frac14+\cdots+\frac17,\qquad
 \frac18+\cdots+\frac1{15},\qquad\ldots
\]
are all of comparable size.

Gemini also helped us see that
we could analyze calibration directly, and supplied a route through
properties of the discrete Hilbert transform matrix in place of the tree
telescoping arguments we had been trying to extend. It worked out the
higher-moment estimates and much of the analysis needed to prove the optimality
of the harmonic weights. AI was then enormously helpful in
preparing the manuscript. Gemini, ChatGPT, and Claude all contributed to
this broader process.

Looking back, the human contribution that mattered most was the sequence
of questions that led to the core idea nugget. We had to identify a
promising extension of TreeSwap, find an adversary that broke it, find a
different algorithm that handled that adversary, and come up with the idea
of interpolating between the two by putting tree-weighted averaging on
both sides of the response map. Once we supplied that idea, Gemini was
able to work out all the details. The short description of the final
algorithm hides the work needed to find this idea nugget.

The experiments linked in the introduction demonstrate the importance of
prompting with our idea nugget. In those attempts, strong models given
substantial background did not independently resolve orthant calibration
within the available budgets. Providing the idea nugget made it possible
to resolve the problem with little additional steering. This does not
establish that the models could not have found the idea nugget with a
longer search. Our guess is that the human steering saved something like
a factor of $10$--$100$ in model tokens. For us, the value of the
collaboration was that human intuition and trial and error supplied a
useful direction, and AI then made it possible to pursue that direction
much further.

We take responsibility for the statements and proofs presented here.

\end{document}